\documentclass[a4paper,fleqn,review]{cas-sc}

\usepackage[numbers,longnamesfirst]{natbib}

\def\tsc#1{\csdef{#1}{\textsc{\lowercase{#1}}\xspace}}
\tsc{WGM}
\tsc{QE}

\usepackage{subfigure}
\usepackage{enumerate}
\usepackage{bm}
\usepackage{setspace}
\usepackage{multirow}
\usepackage{booktabs}
\usepackage{amssymb,extarrows,amsmath}
\usepackage{algorithm} 
\usepackage{algpseudocode}
\usepackage[algo2e,linesnumbered,lined,ruled]{algorithm2e} %
\usepackage{relsize}
\usepackage{extarrows}
\usepackage[section]{placeins} 
\usepackage{xcolor,float}

\newtheorem{theorem}{Theorem}[section] 
\newtheorem{definition}[theorem]{Definition} 
\newtheorem{lemma}[theorem]{Lemma} 

\newtheorem{remark}[theorem]{Remark}

\newtheorem{assumption}[theorem]{Assumption}
\usepackage{cleveref}
\newenvironment{proof}{{\noindent\it Proof.} }{\hfill $\square$\par}
\newcommand{\tabincell}[2]{\begin{tabular}{@{}#1@{}}#2\end{tabular}}%
\DeclareMathOperator*{\argmin}{arg\,min}
\makeatother
\graphicspath{{figs/}}
\allowdisplaybreaks[4]

\begin{document}
\let\WriteBookmarks\relax
\def\floatpagepagefraction{1}
\def\textpagefraction{.001}
\shorttitle{A Bi-variant Diffeomorphic Image Registration with Relaxed Constraints}    
\shortauthors{Yanyan Li et al.}  
\title [mode = title]{A Bi-variant Variational Model for Diffeomorphic Image Registration with Relaxed Jacobian Determinant Constraints}  


%

\author[1,5]{Yanyan Li}




\credit{Conceptualization, Software, Validation, Writing-Original Draft, Writing-Review and Editing}

\affiliation[1]{organization={School of Mathematics and Computational Science},
            addressline={Xiangtan University}, 
            city={Xiangtan},
            postcode={411105}, 
            state={Hunan},
            country={China}}

\author[2]{Ke Chen}
\credit{Conceptualization, Validation, Methodology, Writing-Review and Editing}


\affiliation[2]{organization={Department of Mathematics and Statistics},
            addressline={University of Strathclyde}, 
            city={Glasgow},
            postcode={G1 1XH}, 
            country={United Kingdom}}

\author[3]{Chong Chen}
\credit{Conceptualization, Validation, Methodology, Writing-Review and Editing}

\affiliation[3]{organization={LSEC, ICMSEC, Academy of Mathematics and Systems Science},
            addressline={Chinese Academy of Sciences}, 
            city={Beijing},
            postcode={100190}, 
            state={Beijing},
            country={China}}

\author[1,4]{Jianping Zhang}[orcid=0000-0001-7141-9464]
\cormark[1]
\ead{jpzhang@xtu.edu.cn}
\credit{Conceptualization, Supervision, Methodology, Writing-Review and Editing, Funding acquisition}
\affiliation[4]{organization={National Center for Applied Mathematics in Hunan},
            addressline={Xiangtan University}, 
            city={Xiangtan},
            postcode={411105}, 
            state={Hunan},
            country={China}}

\affiliation[5]{organization={Hunan Key Laboratory for Computation and Simulation in Science and Engieering},
            addressline={Xiangtan University}, 
            city={Xiangtan},
            postcode={411105}, 
            state={Hunan},
            country={China}}
\cortext[1]{Corresponding author}



\begin{abstract}
Diffeomorphic registration is a widely used technique for finding a smooth and invertible transformation between two coordinate systems, which are measured using template and reference images. The point-wise volume-preserving constraint $\det(\nabla\bm{\varphi}(\bm{x})) =1$ is effective in some cases, but may be too restrictive in others, especially when local deformations are relatively large. This can result in poor matching when enforcing large local deformations. In this paper, we propose a new bi-variant diffeomorphic image registration model that introduces a soft constraint on the Jacobian equation $\det(\nabla\bm{\varphi}(\bm{x})) = f(\bm{x}) > 0$. This allows local deformations to shrink and grow within a flexible range $0<\kappa_{m}<\det(\nabla\bm{\varphi}(\bm{x}))<\kappa_{M}$. The Jacobian determinant of transformation is explicitly controlled by optimizing the relaxation function $f(\bm{x})$. To prevent deformation folding and improve the smoothness of the transformation, a positive constraint is imposed on the optimization of the relaxation function $f(\bm{x})$, and a regularizer is used to ensure the smoothness of $f(\bm{x})$. Furthermore, the positivity constraint ensures that $f(\bm{x})$ is as close to one as possible, which helps to achieve a volume-preserving transformation on average. We also analyze the existence of the minimizer for the variational model and propose a penalty-splitting algorithm with a multilevel strategy to solve this model. Numerical experiments demonstrate the convergence of the proposed algorithm and show that the positivity constraint can effectively control the range of relative volume without compromising the accuracy of the registration. Moreover, the proposed model generates diffeomorphic maps for large local deformations and outperforms several existing registration models in terms of performance.
\end{abstract}


\begin{keywords}
Diffeomorphic image registration \sep Jacobian soft constraints \sep Bi-variant optimization \sep Multilevel strategy \sep Alternative iterative minimization
\end{keywords}

\maketitle

\section{Introduction}\label{sec:Section1}
Image registration is the process of finding the optimal transformation between two or more images to establish a geometric correspondence. During the last decades, many registration models have been developed to obtain reasonable transformations, such as the total variation (TV) model \cite{Frohn,Lars2007}, the modified total variation (MTV) model \cite{MTV2010}, the total fractional variation (TFV) model \cite{JPZhang2015,HHan2020a}, the diffusion model \cite{fischer2002fast,kostler2008}, the curvature model \cite{fischer2003curvature,FISCHER2004107,henn2006full}, the elastic model \cite{Fischler}, the viscous fluid model \cite{article1996,Emiliano2003} and the optical flow model \cite{horndetermining1981}. Although these models can generate smooth transformations for small deformations, they may not all be effective for large deformations. However, some models can be guaranteed to generate diffeomorphic mappings. 
 
The diffeomorphic demons algorithm works in the space of diffeomorphisms to enforce transformation invertibility. Vercauteren et al. \cite{Berlin2006,Vercauteren2008} demonstrated that the demons algorithm \cite{thirion1998image} can be extended to represent the entire spatial transformation in the log-domain. Yeo et al. \cite{Yeo2010} proposed the spherical demon method based on the mean curvature and average convexity to transform the surface of the image. In addition, large deformation diffeomorphic metric mapping (LDDMM) \cite{paul1998,Beg_2005} is a popular registration framework, which can handle large deformations and generate diffeomorphic transformations. Based on the LDDMM framework and shape analysis, Charon et al. proposed a generalization of registration algorithms to non-oriented shapes~\cite{Charon_2013} and an extension of diffeomorphic registration to allow a morphological analysis of data structures with inherent density variations and imbalances~\cite{Hsieh_2021}. C. Chen \cite{Chen_2021} employed optimal diffeomorphic transportation that combines the Wasserstein distance and the flow of diffeomorphisms, to build a joint image reconstruction and motion estimation model, which is suitable for spatio-temporal imaging involving mass-preserving large diffeomorphic deformations. Bauer et al.~\cite{Bauer_2015} proposed diffeomorphic density matching by optimal information transport and applied it to medical image registration, building on connections between the Fisher-Rao information metric on the space of probability densities and right-invariant metrics on the infinite-dimensional diffeomorphism manifold. Feydy et al.~\cite{Feydy_2017} introduced a non-local geometric similarity measure based on unbalanced optimal transport methods and fast entropic solvers to achieve robust and simple diffeomorphic registration. We also refer the readers to \cite{ashburner2011,Beg_2005,CChen2019,CChen2018,HHan2021a,Joshi2000,MangA,Risser2011} for more details. 

Several registration models have been developed to avoid mesh folding by restricting the Jacobian determinant quantity $\det(\nabla\bm{\varphi})$ of the transformation $\bm{\varphi}$. Haber and Modersitzki \cite{HaberNumerical} proposed an elastic registration model subject to point-wise volume-preserving by restricting $ \det(\nabla\bm{\varphi}) =1$, which can ensure that the mapping is diffeomorphic. Although such incompressibility has important applications in some fields, it is not necessary or reasonable in others. To generate a more reasonable and practical mapping, Haber and Modersitzki then proposed to relax the Jacobian determinant constraint into a certain interval $0<\kappa_{m}<\det(\nabla\bm{\varphi})<\kappa_{M}$ \cite{Haber2007}. Lam and Lui \cite{KCLam2014,lamquasiconformal2015} introduced a novel model that uses a Beltrami coefficient term to obtain diffeomorphic image registration using quasi-conformal maps. One of the most important features of this method is that the deformed Jacobian determinant can be represented by the Beltrami coefficient. As pointed out, if the infinite norm of the Beltrami coefficient is less than $1$, then the Jacobian determinant is greater than $0$, so that such a model can deal with large deformation.

It is possible to reformulate diffeomorphic image registration into a variational problem with an additional penalty term related to the deformed Jacobian determinant. Burger et al.~\cite{Hyperelastic2013} monitored $\det(\nabla\bm{\varphi}) $ using an unbiased function in the hyper-elastic model to guarantee diffeomorphic deformation. R{\"u}haak et al.\,\cite{Ruhaak2017} proposed a volume penalty function to control the change of $\det(\nabla\bm{\varphi}) $, which prevented the formation of foldings by keeping the $\det(\nabla\bm{\varphi}) $ positive. Yanovsky et al. \cite{Yanovskyarticle,Yanovsky2007Log} used the symmetric Kullback--Leibler (sKL) distance to quantify the deformation for obtaining an unbiased diffeomorphic mapping. Other works have also involved controlling $\det(\nabla\bm{\varphi})$ by penalty functions \cite{Alessa2021,Daoping2018,Zhang2022}.

Generally, there are three main strategies to obtain diffeomorphic maps. The first approach is to search for deformations on the diffeomorphism manifold, which is computationally efficient but does not guarantee the geometric properties of the deformation field and may lead to nearly non-diffeomorphic deformations. The second method is to impose additional constraints to prevent non-diffeomorphic mappings, such as $\det(\nabla\bm{\varphi})>0$. This approach can explicitly control the deformation and obtain diffeomorphic mappings without manual intervention. However, it may cause either over-preservation of volume or large volume changes, resulting in mismatches or inaccurate registration. The third approach is to indirectly control $\det(\nabla\bm{\varphi})$ through the Beltrami coefficient (or Beltrami-like in 3D). This method is capable of dealing with large deformations well and producing smooth deformations. However, it works with the Beltrami coefficient and the deformation simultaneously, making the energy functional too complex and resulting in high computational cost, and the convergence of the method cannot be guaranteed.

Recently, deep learning-based approaches have attracted considerable attention in the field of image registration. Y. Guo et al. \cite{GUO2023} developed a framework that uses quasi-conformal geometry and convolutional neural networks to automatically detect and register the cortical surfaces of the brain. D. Wei et al. \cite{Wei_2022} proposed a recurrent deep neural network for the registration of infant brain MR images. T. Mok et al. \cite{Mok_2020} proposed a novel and efficient unsupervised symmetric image registration method that maximizes the similarity between images in the space of diffeomorphic maps and estimates forward and inverse transformations simultaneously.

Deep learning-based methods exploit the inductive capabilities of neural networks on datasets, thereby shifting the iterative optimization of traditional methods to the deep learning training process. Consequently, deep learning-based image registration methods have outperformed traditional methods in terms of registration efficiency. However, these methods lack rigorous mathematical theoretical support, which limits their interpretability. In addition, they have limited capacity to register large deformations and do not guarantee the generation of diffeomorphic transformations \cite{Dalca_2018}. Conversely, traditional image registration methods offer rigorous mathematical theoretical support and employ precise mathematical expressions to guide the registration process. They have also demonstrated success in registering large deformations and preserving the diffeomorphism. Despite numerous research achievements in traditional registration methods, their registration accuracy and efficiency (especially for 3D problems) still fail to meet the needs of users. The development of a novel diffeomorphic registration model and the design of a highly accurate and fast registration algorithm still remain difficult and challenging.

The purpose of this work is to reformulate a new bi-variant diffeomorphic registration approach that utilizes a variational minimization model and to demonstrate the existence of a minimizer for the variational model. This approach produces a smooth and reversible transformation and can be effectively applied to the registration of large local deformations. Our contributions are summarized as follows.
\begin{itemize} 
\item[1.] We propose a novel bi-variant diffeomorphic image registration model. It incorporates a penalty term $\int_{\Omega}\phi(f(\bm{x}))d\bm{x}$ and a soft constraint related to the Jacobian equation $\det(\nabla{\bm{\varphi}}(\bm{x})) = f(\bm{x})>0$. The penalty term can optimize the unknown $ f(\bm{x})$ in a positive relaxation function space. The constraint guarantees the diffeomorphism of the deformation by the positive relaxation function \( f(\bm{x}) \) and allows local deformations to expand and contract in a flexible range without compromising the accuracy of the registration for large local deformations.

\item[2.] To automatically optimize $f(\bm{x})$, we propose incorporating a positive penalty term $\int_{\Omega}\phi(f(\bm{x}))d\bm{x}$ into the registration model. The control function $\phi(\cdot)$ is utilized to enforce $f(\bm{x})>0$ as closely as possible to one, thereby facilitating the achievement of a volume-preserving transformation on average. Additionally, we employ a regularization term $\int_{\Omega} \|\nabla f(\bm{x})\|^2 d\bm{x}$ to enhance the smoothness of $f(\bm{x})$, which in turn indirectly improves the smoothness of the transformation ${\bm{\varphi}}(\bm{x})$. 

\item [3.] We conduct a theoretical analysis to demonstrate the existence of an optimal solution for the bi-variant diffeomorphic image registration model that we have proposed.

\item[4.] We propose a penalty-splitting algorithm that utilizes a multilevel strategy to solve the bi-variant variational model. The main aim of this approach is to prevent the algorithm from getting trapped in a local minimum, reduce the computational cost, and improve the model's ability to capture large local deformations. The results of numerical experiments demonstrate the convergence of the proposed algorithm. By incorporating the penalty term $\int_{\Omega}\phi(f(\bm{x}))d\bm{x}$, we can effectively reduce or eliminate the grid folding ratio and maintain the volume of deformation on average. Moreover, our proposed method outperforms state-of-the-art diffeomorphic registration models in accurately handling large deformations.
\end{itemize}

The rest of the paper is organized as follows. In \Cref{sec:Section2}, we review several popular models related to diffeomorphic image registration. In \Cref{sec:Section3}, we give the mathematical analysis and the solving algorithm of the newly proposed registration model. The numerical implementation of the proposed model is presented in \Cref{sec:Section4}. We present 2D and 3D examples to evaluate the performance of our approach in \Cref{sec:Section5}. Finally, we summarize this work in \Cref{sec:Section6}.

\section{Reviews}\label{sec:Section2} 
Given a pair of images $T(\cdot),R(\cdot): \Omega\subset\mathbb{R}^d\rightarrow\mathbb{R}$ ($ d=2 $ or $ 3 $), the aim of image registration is to seek a suitable transformation $\bar{\bm{\varphi}}(\cdot): \mathbb{R}^d\rightarrow\mathbb{R}^d$ that satisfies $T(\bar{\bm{\varphi}}(\bm{x}))\approx  R(\bm{x})$. The updated transformation $\bar{\bm{\varphi}}(\bm{x})$ will be described by adding the displacement field $\bm{u}(\bm{x})$ to the current transformation $\bm{\varphi}(\bm{x})$, that is, $\bar{\bm{\varphi}}(\bm{x})=\bm{\varphi}(\bm{x})+\bm{u}(\bm{x})$. 
To make the deformed image $T\big(\bar{\bm{\varphi}}(\bm{x})\big)$ closer to the reference image $R(\bm{x})$, the similarity measure $\mathcal{D}(\bm{u})$ of the image registration can be minimized as
\begin{equation}\label{eq:eqSSD1}
\min_{\bm{u}}\left\{\mathcal{D}(\bm{u}):=\frac{1}{2} \int_{\Omega}\left[T\big(\bm{\varphi}(\bm{x})+\bm{u}(\bm{x})\big)-R(\bm{x})\right]^{2} d \bm{x}\right\}.
\end{equation}

A simple optimization of \eqref{eq:eqSSD1} leads to an ill-posed problem with unstable and non-smooth solutions \cite{HaberNumerical}. Hence a regularizer $\mathcal{F}(\cdot)$ is added to construct the well-posed problem
\begin{equation}\label{eq:eqmodel1}
\min_{\bm{u}} \big\{\mathcal{J}(\bm{u}):=\mathcal{D}(\bm{u})+\tau \mathcal{F}(\bm{u})\big\},
\end{equation}
where $\tau>0$ is a constant to balance $\mathcal{D}(\bm{u})$ and $\mathcal{F}(\bm{u})$. We refer the readers to \cite{collignon1995,hermosillo2002variational,MooreQuality} for many classical similarity measures and to \cite{Chumchob2011,fischer2002fast,fischer2003curvature,Fischler,JPZhang2015,ZHANG2016,zhang_VectorialMinimized_2022} for the popular regularizers in image registration.

The diffeomorphism of the mapping $\bar{\bm{\varphi}}=\bm{\varphi}+\bm{u}$ is an important requirement in image registration, and its necessary condition is $\mathcal{K}(\bm{u}(\bm{x}))=\det(\nabla\bar{\bm{\varphi}}) > 0$. To ensure this, a natural approach is to impose a positivity constraint on $\mathcal{K}(\bm{u}(\bm{x}))$ when minimizing functional $\mathcal{J}(\bm{u})$, thus preventing non-diffeomorphic mapping $\bar{\bm{\varphi}}$. Therefore, the constrained model can be written as 
\begin{equation}\label{eq:eqCop}
\begin{split}
&\min_{\bm{u}} \big\{\mathcal{J}(\bm{u})=\mathcal{D}(\bm{u})+\tau \mathcal{F}(\bm{u})\big\},\\
&\text{s.t.}\quad \mathcal{K}(\bm{u}(\bm{x}))=\det\big(\nabla (\bm{\varphi}+\bm{u})\big) >0,\; \forall \bm{x}\in\Omega.
\end{split}
\end{equation}

Before introducing our model, we will review several recent works that have implemented the Jacobian determinant positivity constraint $\det\big(\nabla (\bm{\varphi}+\bm{u})\big) >0$ in some indirect ways.
\begin{itemize}
\item Haber and Modersitzki \cite{HaberNumerical} proposed a point-wise volume-preserving image registration model, which is followed by
\begin{equation}\label{eq:eqvolume}
\begin{aligned}
&\min_{\bm{u}}\left\{\mathcal{J}(\bm{u})=\mathcal{D}(\bm{u})+\frac{\tau}{2} \int_{\Omega} \big[\mu\|\nabla \bm{u}\|^2 +(\lambda+\mu)(\nabla \cdot \bm{u})^{2}\big] d\bm{x}\right\}, \\
&~\text{s.t.} \quad \mathcal{K}(\bm{u}(\bm{x})) =1,\; \forall \bm{x}\in\Omega,
\end{aligned}
\end{equation}
where $\lambda$ and $\mu$ are the so-called Lam\'{e} constants. However, point-wise volume preservation is not desirable when the anatomical structure is compressible in medical imaging, so a soft inequality constraint that allows local regions of the image to shrink or grow within a specified range is more practical.
\item Modersitzki et al. \cite{Haber2007} then proposed the relaxed constraint
to improve the registration model \eqref{eq:eqvolume}, namely,
\begin{equation}\label{eq:eqvolume2}
\begin{aligned}
&\min_{\bm{u}}\left\{\mathcal{J}(\bm{u})=\mathcal{D}(\bm{u})+\frac{\tau}{2} \int_{\Omega} \big[\mu\|\nabla \bm{u}\|^2 +(\lambda+\mu)(\nabla \cdot \bm{u})^{2}\big] d\bm{x}\right\}, \\ 
&~\text{s.t.} \quad \kappa_{m} \leq \mathcal{K}(\bm{u}(\bm{x})) \leq \kappa_{M},\; \forall \bm{x}\in\Omega,
\end{aligned}
\end{equation}
where the positive constants $ \kappa_{m}$ and $\kappa_{M} $ are provided by the user as prior information in the specific application. If $\kappa_{m}=\kappa_{M}=1$, the model \eqref{eq:eqvolume2} degenerates to point-wise volume-preserving registration model \eqref{eq:eqvolume}.

\item Yanovsky et al. \cite{Yanovsky2007Log} applied the symmetric Kullback-Leibler (sKL) distance to propose a log-unbiased fluid image registration model as follow
\begin{equation}\label{eqLog-unbiased}
\begin{aligned}
\min_{\bm{u}}\Big\{\mathcal{J}(\bm{u})=\mathcal{D}(\bm{u})+ \tau \int_{\Omega}\big(\mathcal{K}(\bm{u}(\bm{x}))-1\big) \log\big(\mathcal{K}(\bm{u}(\bm{x}))\big) d \bm{x}\Big\},
\end{aligned}
\end{equation}
where $\tau>0$ is the regularization parameter. The corresponding Euler-Lagrange equation can be written as
\begin{equation*}
\bm{g}(\bm{x}, \bm{u}):=[T(\bm{\varphi}+\bm{u})-R(\bm{x})] \nabla T
-\tau\left[\begin{array}{c}
-\frac{\partial}{\partial x_{1}}\left(\frac{\partial (\bm{\varphi}_{2}+\bm{u}_{2})}{\partial x_{2}} L^{\prime}\right)+\frac{\partial}{\partial x_{2}}\left(\frac{\partial (\bm{\varphi}_{2}+\bm{u}_{2})}{\partial x_{1}} L^{\prime}\right) \\
\frac{\partial}{\partial x_{1}}\left(\frac{\partial (\bm{\varphi}_{1}+\bm{u}_{1})}{\partial x_{2}} L^{\prime}\right)-\frac{\partial}{\partial x_{2}}\left(\frac{\partial (\bm{\varphi}_{1}+\bm{u}_{1})}{\partial x_{1}} L^{\prime}\right)
\end{array}\right]=\bm{0},
\end{equation*}
where $L^{\prime}=1+\log \big(\mathcal{K}(\bm{u}(\bm{x}))\big)-1 /\mathcal{K}(\bm{u}(\bm{x}))$.
The authors take the idea of \cite{Emiliano2003} and use it to calculate the instantaneous velocity $\bm{v}$ with the Gaussian kernel $G_{\sigma}*\bm{g}$ of the functional $\bm{g}$ (where $\sigma$ is the variance). Subsequently, the displacement field $\bm{u}$ is determined by solving the material derivative of $\bm{u}$ as follows:
\begin{equation*}
\begin{aligned}
\left\{ {\begin{array}{*{20}{c}}
{\frac{\partial \bm{u}}{\partial t}+(\nabla\bm{u})\bm{v}=\bm{v},}\\
{\bm{u}(\bm{x},0)=\bm{0}.}\\
\end{array}} \right.
\end{aligned}
\end{equation*}
This log-unbiased registration can help to obtain an unbiased diffeomorphic transformation, also see \cite{Yanovskyarticle} for more details.

\item Burger et al.~\cite{Hyperelastic2013} employed a penalty function $\phi(z)=\left((z-1)^{2} / z\right)^{2}$ to control $ \mathcal{K}(\bm{u}(\bm{x}))$ in order to ensure that the cost of shrinking and expanding is the same, due to $ \phi(1/{z})=\phi(z) $. Consequently, the penalty term of hyper-elastic registration can be expressed as 
\begin{equation}\label{Hyper}
\mathcal{F}(\bm{u})=\int_{\Omega}\left(\frac{(\mathcal{K}(\bm{u}(\bm{x}))-1)^2}{\mathcal{K}(\bm{u}(\bm{x}))}\right)^2d\bm{x},
\end{equation}
where one can explicitly show that the deformation is physically meaningful due to
$ \mathcal{F}(\bm{u}) \rightarrow\infty $ for $\mathcal{K}(\bm{u}(\bm{x})) \rightarrow 0$. 

\item R\"{u}haak et al.~\cite{Ruhaak2017} also directly measured the change of $ \mathcal{K}(\bm{u}(\bm{x}))$ by adding the volume change control term $ \int_{\Omega}\phi\big(\mathcal{K}(\bm{u}(\bm{x}))\big) d\bm{x} $ in the energy functional, and $\phi$ was defined as
\begin{equation}\label{eq:Vcc}
{\phi}(z)=\left\{\begin{array}{ll}
\frac{(z-1)^{2}}{z} &\text {if}~z>0, \\
+\infty &\text {otherwise}.
\end{array}\right.
\end{equation}

\item Zhang and Chen~\cite{Daoping2018} proposed a regularizer $\mathcal{F}(\bm{u})$ based on the Beltrami coefficient to seek a diffeomorphic mapping $\bar{\bm{\varphi}}=\bm{\varphi}+\bm{u}$, which was formulated as 
\begin{equation}\label{eq:ZC1}
\mathcal{F}(\bm{u})=\int_{\Omega}\phi(|\mu|^2)d\bm{x},~\text{with}~ |\mu|^2=\frac{\|\nabla\left(\bm{\varphi}+\bm{u}\right)\|^2-2\mathcal{K}(\bm{u}(\bm{x}))}{\|\nabla\left(\bm{\varphi}+\bm{u}\right)\|^2+2\mathcal{K}(\bm{u}(\bm{x}))},
\end{equation}
where $ \phi(z)=\frac{z}{(z-1)^2} $ or $ \frac{z^2}{(z-1)^2} $.  
\end{itemize}

\section{A novel image registration model and existence of minimizers}\label{sec:Section3}
In this part, we first introduce the proposed diffeomorphic registration model, which can potentially be applied to large local deformations. We then give the mathematical analysis including the existence of the optimal solution of our bi-variant variational model.
\subsection{Proposed model}
To prevent the folding of the transformation $\bar{\bm{\varphi}}$, it is necessary to ensure that $\det\big(\nabla \bar{\bm{\varphi}}(\bm{x})\big)$ is greater than zero. To achieve this, we propose a diffeomorphic model that incorporates the Jacobian equation $ \det\big(\nabla \bar{\bm{\varphi}}(\bm{x}) \big) = f(\bm{x})$ as a constraint, where the relaxation function $ f(\bm{x})>0 $ is unknown. Our new model is designed to provide flexibility to $f(\bm{x})$ while still guaranteeing its positivity. This simplifies our model compared to similar existing models in the literature while maintaining its effectiveness. Thus, the proposed diffeomorphic registration model can be expressed as
\begin{equation}\label{eq:eqComodel}
\begin{aligned}
\min_{\bm{u}\in\mathcal{V}, 0<f\in L^2(\Omega)}&\left\{\mathcal{D}(\bm{u})+\frac{\tau_1}{2}\int_{\Omega}\|\nabla \bm{u}\|^2d\bm{x}\right\}, \\ 
\text{s.t.}\quad&\mathcal{K}(\bm{u}(\bm{x}))=\det\big(\nabla(\bm{\varphi}+\bm{u})\big) = f(\bm{x}), \quad \forall \bm{x}\in \Omega,
\end{aligned}
\end{equation}
where $\mathcal{V}$ is defined by 
\[\mathcal{V}:=\big\{\bm{v}~|~\bm{v}\in[\mathcal{H}^1(\Omega)]^d\text{ and } \bm{v}|_{\partial\Omega}=\bm{0} \big\}. \]
The regularization parameter $\tau_1$ plays a crucial role in achieving a trade-off between a reliable similarity metric and a smooth solution \cite{Haber2007,Chumchob2011,Zhang2022}. When $\tau_1$ is too small, it can result in deformation twists or folding, indicating that the transformation is not bijective. On the other hand, a higher value of $\tau_1$ can prevent grid folding, but it may also cause a decrease in the distance between the images \cite{Fair,Haber2007}.

Obviously, it is not easy to maintain the positivity of $f(\bm{x})$ in \eqref{eq:eqComodel} during direct optimization. To address this problem, we propose to add a penalty term $\int_{\Omega} \phi \big(f(\bm{x}) \big)\mathrm{d}\bm{x}$ to \eqref{eq:eqComodel}, allowing $f(\bm{x})$ to be automatically optimized. The registration model can then be reformulated as
\begin{equation}\label{eq:model1}
\begin{aligned}
\min_{\bm{u}\in\mathcal{V}, 
f\in L^2(\Omega)}&  \left\{\mathcal{D}(\bm{u})
+\frac{\tau_1}{2}\int_{\Omega}\|\nabla \bm{u}\|^2d\bm{x}+\tau_2 \int_{\Omega} \phi\big(f(\bm{x})\big) \mathrm{d} \bm{x}\right\},\\
\text{s.t.} \quad &~ \det\big(\nabla(\bm{\varphi}+\bm{u})\big)=f(\bm{x}),  \quad  \forall \bm{x}\in \Omega, 
\end{aligned}
\end{equation}
where $\phi(\cdot)$ is a known control function, and the value of $f(\bm{x})$ is determined by the penalty parameter $\tau_2>0$, which affects the range of modifications to the Jacobian determinant $ \det\big(\nabla(\bm{\varphi}+\bm{u})\big) $.

Here, we introduce two choices of function $\phi(\cdot)$. One choice is to consider the form in~\cite{Ruhaak2017} as
\begin{equation}\label{eq:eqpenaltyf2}
\phi_1(f)=\left\{\begin{array}{ll}
\frac{(f-1)^{2}}{f}, & \text{if}~f>0, \\
+\infty,& \text{otherwise}.
\end{array}\right.
\end{equation}
The penalty term $\int_{\Omega}\phi\big(f(\bm{x})\big)d\bm{x}$ in \eqref{eq:model1} is designed to prevent the formation of folding by keeping $ f(\bm{x})>0 $. The control term will be infinite when $ f(\bm{x}) \leq 0 $, and $ f(\bm{x}) $ is explicitly monitored since it is minimized when $ f(\bm{x}) = 1 $. An alternative approach is inspired by the work of log-unbiased image registration \cite{Yanovsky2007Log}, which associates deformations with their corresponding global density maps. We propose to measure the magnitude of $f(\bm{x})$ by using the sKL distance between $f(\bm{x})$ and $1$, namely,
\begin{equation}\label{eq:eqpenaltyf3}
  \phi_2(f)=\left\{\begin{array}{ll}
    (f-1)\log(f), & \text{if}~f>0, \\
    +\infty, & \text{otherwise}.
  \end{array}\right.
\end{equation}
It is evident that the control function $\phi_2(f)$ is always non-negative, with a minimum of zero when $f(\bm{x})=1$. Additionally, $f(\bm{x}) \leq 0$ can be explicitly restricted. This ensures that $f(\bm{x})>0$ for any $\bm{x}\in \Omega$ and $f(\bm{x})$ is as close to 1 as possible. Consequently, the constraint condition $ \det\big(\nabla\bar{\bm{\varphi}}(\bm{x}) \big)=f(\bm{x})$ indirectly restricts the range of the Jacobian determinant $\det\big(\nabla \bar{\bm{\varphi}}(\bm{x})\big)$, thus making the deformation tend to be volume-preserving.

Finally, we introduce a diffusion regularizer for the positivity constraint $f(\bm{x})$ to improve the smoothness of the solutions of \eqref{eq:model1}. Let $\mathcal{M}(\Omega)$ be the solution constraints, which is defined as 
\[\mathcal{M}(\Omega)=\big\{\bm{\omega}:=(\bm{u}, f)\in \mathcal{V}(\Omega)\times\mathcal{W}(\Omega)\big\}\subset [\mathcal{H}^1(\Omega)]^d\times\mathcal{H}^1(\Omega)=[\mathcal{H}^1(\Omega)]^{d+1},\]
where $\mathcal{W}(\Omega):=\big\{f(\bm{x})~| ~f\in\mathcal{H}^1(\Omega)\text{ and } {f} _{\partial\Omega}=1 \big\}$. The function space $[\mathcal{H}^1(\Omega)]^{d+1}$ is a reflexive Banach space, and its subset $\mathcal{M}(\Omega)\subset [\mathcal{H}^1(\Omega)]^{d+1}$ is convex. \textbf{}Consequently, a novel bi-variant registration model for computing large diffeomorphic transformation is proposed as
\begin{equation}\label{eq:model2}
\min_{(\bm{u}, f)\in \mathcal{M}(\Omega)}\left\{\mathcal{L}(\bm{u},f):=\mathcal{D}(\bm{u})+\mathcal{S}(\bm{u}, f)\right\},\quad\text{s.t.} \; \det\big(\nabla(\bm{\varphi}+\bm{u})\big)=f(\bm{x}), \; \forall \bm{x}\in \Omega, 
\end{equation}
where $\mathcal{S}(\bm{u}, f)=\frac{\tau_1}{2}\int_{\Omega}\|\nabla \bm{u}\|^2d\bm{x}
+\tau_2 \int_{\Omega} \phi\big(f\big) d\bm{x} +\frac{\tau_3}{2}\int_{\Omega}\|\nabla f\|^2d \bm{x}$ is a combined regularizer, the parameter $\tau_3$ is employed to control the smoothness of the relaxation function $f(\bm{x})$. A smaller $\tau_3$ value will make the function $f(\bm{x})$ less smooth, which in turn will make the deformation $\bm{\varphi}(\bm{x})$ less smooth. On the other hand, if $\tau_3$ is increased, the deformation $\bm{\varphi}(\bm{x})$ may become over-smooth, resulting in a decrease in the similarity between images. 

\subsection{Existence of minimizers}
The penalty method is a popular technique for transforming an optimization problem with equality constraints \eqref{eq:model2} into an unconstrained optimization problem \cite{NumericalOptimization2006}. This is achieved by introducing a penalty term into the objective functional that penalizes any violation of the constraints. To enforce the constraints of \eqref{eq:model2} in a least-squares (LS) manner, a quadratic penalty function is utilized, resulting in a modified bi-variant optimization problem which is given as follows.
\begin{equation}\label{eq:eqALM1}
\min_{\bm{\omega}\in\mathcal{M}(\Omega)}\Big\{\mathcal{L}_\lambda(\bm{\omega}):=\mathcal{L}_\lambda(\bm{u},f)=\mathcal{D}(\bm{u})+\mathcal{S}(\bm{u}, f)+\mathcal{C}(\bm{u},f)\Big\},
\end{equation}
where $\mathcal{C}(\bm{\omega}):=\mathcal{C}(\bm{u},f)=\frac{\lambda}{2}\int_{\Omega}\big( \det\big(\nabla(\bm{\varphi}+\bm{u})\big)-f(\bm{x}) \big)^2d\bm{x}$, and the penalty parameter $\lambda >0$ controls the strength of the penalty for the constraint term. By increasing the value of $\lambda$, the penalty for violating the restrictions is increased, thus encouraging optimization algorithms to look for a solution to \eqref{eq:eqALM1} that satisfies the conditions of \eqref{eq:model2}. However, if $\lambda$ is too large, it can cause optimization problems to become ill-conditioned \cite{NumericalOptimization2006}.

\begin{lemma}\label{lemma-1}
Let $\Omega$ be an open bounded set of $\mathbb{R}^d$, and $(\bm{u}, f)\in\mathcal{M}(\Omega)$, then
\begin{equation}\label{Poincare_inequa}
\|\bm{u}\|_{L^2(\Omega)}\leq c_1\|\nabla\bm{u}\|_{L^2(\Omega)},\;\|f-1\|_{L^2(\Omega)}\leq c_2\|\nabla f\|_{L^2(\Omega)}
\end{equation}
for some constants $c_1$ and $c_2$ depending only on $d$ and $\Omega$.
\end{lemma}
\begin{proof}
Refer to \cite[Poincar\'{e} inequality]{GAubert2006}.
\end{proof}

In order to apply the compactness results of \cite{EZeidler1985} to the minimization problem (\ref{eq:eqALM1}), we can use Lemma \ref{lemma-1} to obtain an inequality property as follows:
\begin{equation}\label{M-norm}
\|\bm{\omega}\|_{\mathcal{M}(\Omega)}:=\|\bm{u}\|_{L^2(\Omega)}+\|\nabla\bm{u}\|_{L^2(\Omega)}+\|f\|_{L^2(\Omega)}+\|\nabla f\|_{L^2(\Omega)}\leq c_3\mathcal{L}_\lambda(\bm{u},f)+c_4,
\end{equation}
for some constants $c_3$ and $c_4$ depending only on $c_1$, $c_2$, $\tau_1$ and $\tau_3$. Hence the coercive property of $\mathcal{L}_\lambda$ holds because
\begin{equation}\label{eq-coea}
\mathcal{L}_\lambda(\bm{\omega})\rightarrow+\infty \quad \text { as }\quad
\|\bm{\omega}\|_{\mathcal{M}(\Omega)} \rightarrow \infty, \quad \bm{\omega} \in \mathcal{M}(\Omega).
\end{equation}

The boundedness of $\mathcal{M}(\Omega)$ plays an important role in \cite[Proposition 38.12]{EZeidler1985}. Similarly, we also introduce this frequently used trick which reduces the minimum problem (\ref{eq:eqALM1}) on the unbounded subset $\mathcal{M}(\Omega)$ of the Banach space $[\mathcal{H}^1(\Omega)]^{d+1}$ to an equivalent minimum problem
\begin{equation}\label{eq-modify-min}
\min_{\bm{\omega}\in \mathcal{M}_{\bm{\omega}_0,r}(\Omega)} \mathcal{L}_\lambda(\bm{\omega}),
\end{equation}
where $\mathcal{M}_{\bm{\omega}_0,r}(\Omega):=\mathcal{M}(\Omega)\cap \bar{U}\left(\bm{\omega}_0, r\right)$ is bounded, and $\bar{U}\left(\bm{\omega}_0, r\right) \stackrel{\text { def }}{=}\left\{\bm{\omega} \in [\mathcal{H}^1(\Omega)]^{d+1}|\left\|\bm{\omega}-\bm{\omega}_0\right\|_{\mathcal{M}(\Omega)} \leq r\right\}$.

\begin{lemma}\label{lemma-2} For the functional $\mathcal{L}_\lambda(\bm{\omega}): \mathcal{M}(\Omega) \subseteq [\mathcal{H}^1(\Omega)]^{d+1} \rightarrow[-\infty, \infty]$, the minimum problem (\ref{eq:eqALM1}), i.e. $\min\limits_{\bm{\omega}\in \mathcal{M}(\Omega)} \mathcal{L}_\lambda(\bm{\omega})$, is equivalent to (\ref{eq-modify-min}) where $\bm{\omega}_0\in \mathcal{M}(\Omega)$ when (\ref{eq-coea}) holds 
and $r$ is chosen to be sufficiently large.
\end{lemma}
\begin{proof}
Let $\mathcal{L}_\lambda(\bm{\omega}) \not\equiv +\infty$. By (\ref{eq-coea}) there exists a $r>0$ such that $\mathcal{L}_\lambda(\bm{\omega})>\mathcal{L}_\lambda(\bm{\omega}_0)$ holds for all $\bm{\omega}$ with $\left\|\bm{\omega}-\bm{\omega}_0\right\|_{\mathcal{M}(\Omega)}>r$. Refer the readers to \cite[Corollary 38.14]{EZeidler1985} for more details.
\end{proof}

To proceed, we make the following assumptions. 
\begin{assumption}\label{assumption1}
For any $\boldsymbol{x} \in \Omega$, $\bm{\varphi}\in \mathcal{C}^{\ell+1}(\Omega)$ ($\ell\geq 0$) and $\bm{\omega}\in\mathcal{M}(\Omega)$ (especially $\bm{\omega}\in\mathcal{M}_{\bm{\omega}_0,r}(\Omega)$), there exists a constant $\mathcal{A}_0>0$, and the following assumptions hold:
\begin{itemize}[leftmargin=1.2cm]
\item[1)]  assume that two images $T(\boldsymbol{x})$ and $R(\boldsymbol{x})$ satisfy
\begin{equation}\label{assump1}
\max \left\{\|T\|_{L^{\infty}(\Omega)},\|R\|_{L^{\infty}(\Omega)},\|\nabla T\|_{L^{\infty}(\Omega)},\left\|\nabla^2 T\right\|_{L^{\infty}(\Omega)}\right\}<\mathcal{A}_0<+\infty,
\end{equation}
hence $\|T-R\|_{L^{\infty}(\Omega)}<2 \mathcal{A}_0$;
\item[2)] assume that the positivity constraint function $\phi\big(f\big)$ is continuous when $f>0$.
\end{itemize}
\end{assumption}
Next, let us analyze the properties of the energy functional $\mathcal{L}_\lambda(\bm{\omega})$.
\begin{lemma}[Lower semi-continuity of $\mathcal{S}(\bm{\omega})$ and $\mathcal{C}(\bm{\omega})$]\label{lemma-lower-semi-2}
The bi-variant regularization functional $\mathcal{S}(\bm{\omega}):=\mathcal{S}(\bm{u}, f)$ in (\ref{eq-modify-min}) and the penalty functional $\mathcal{C}(\bm{\omega}):=\mathcal{C}(\bm{u},f)$ satisfy the lower semicontinuity, that is, let $\bm{\omega}_j \in \mathcal{M}(\Omega)$ and $\bm{\omega}_j \underset{\mathrm{L}^1(\Omega)}{\stackrel{*}{\longrightarrow}} \bm{\omega}$ ($j\rightarrow+\infty$); then
\begin{equation}\label{assump2}
\mathcal{S}(\bm{\omega}) \leq \varliminf_{j \rightarrow+\infty} \mathcal{S}(\bm{\omega}_j),\quad \text{ and }\quad
\mathcal{C}(\bm{\omega}) \leq \varliminf_{j \rightarrow+\infty} \mathcal{C}(\bm{\omega}_j).
\end{equation}
\end{lemma}
\begin{proof}
Define $\mathcal{S}_1(\bm{\omega}):=\mathcal{S}_1(\bm{u},f)=\frac{\tau_1}{2}\int_{\Omega}\|\nabla \bm{u}\|^2d\bm{x}+\frac{\tau_3}{2}\int_{\Omega}\|\nabla f\|^2d \bm{x}$. Since $\mathcal{S}_1$ is convex, $\text{epi}(\mathcal{S}_1)=\{(\bm{\omega},t)\in \text{dom}(\mathcal{S}_1)\times \mathbb{R}:\mathcal{S}_1(\bm{\omega})\leq t\}$ is convex. Therefore, we can conclude that $\text{epi}(\mathcal{S}_1)$ is weakly sequentially closed, implying that $\mathcal{S}_1$ is weakly sequentially lower semi-continuous (LSC). In addition, due to the continuity of $\phi\big(f\big)$ in Assumption 2), $\mathcal{S}(\bm{\omega})$ is also weakly sequentially LSC.

For any $\boldsymbol{x} \in \Omega$, $\bm{\varphi}\in \mathcal{C}^{\ell+1}(\Omega)$ ($\ell\geq 0$) and $\bm{\omega}\in\mathcal{M}(\Omega)$ (especially, $\bm{\omega}\in\mathcal{M}_{\bm{\omega}_0,r}(\Omega)$), there exists a constant $\mathcal{A}_1>0$, one has $\|(\nabla\bar{\bm{\varphi}}_\theta)^*\|_{L^\infty(\Omega)}<\mathcal{A}_1<+\infty$ and $c_5=\|\det(\nabla\bar{\bm{\varphi}}_\theta)-f\|_{L^\infty(\Omega)}$ is boundedness, and hence the following inequality holds.
\begin{equation}\label{eq-c-lower-semi}
|\mathcal{C}(\bm{\omega})-\mathcal{C}(\bm{\omega}_0)|\leq\lambda c_5\left(\int_\Omega|\langle(\nabla\bar{\bm{\varphi}}_\theta)^*,\nabla (\bm{u}-\bm{u}_0)\rangle| d\bm{x}+\int_\Omega|f-f_0|d\bm{x}\right)
\leq c_6\|\bm{\omega}-\bm{\omega}_0\|_{\mathcal{M}(\Omega)},
\end{equation}
where $\bar{\bm{\varphi}}_\theta=\bm{\varphi}+\theta \boldsymbol{u}+(1-\theta) \bm{u}_0$ for any $\theta\in(0,1)$, $c_6=\lambda c_5\cdot\max(\mathcal{A}_1,1)$ and $(\nabla\bar{\bm{\varphi}}_\theta)^*$
is the adjoint matrix of $\nabla\bar{\bm{\varphi}}_\theta\in\mathbb{R}^{d\times d}$, which can also be seen from Lemma \ref{lemma3-8}. Inequality (\ref{eq-c-lower-semi}) implies that the penalty functional $\mathcal{C}(\bm{\omega}):=\mathcal{C}(\bm{u},f)$ satisfies the lower semi-continuity.  
\end{proof}
\begin{lemma}[Lower semi-continuity of $\mathcal{L}_\lambda$]\label{lemma-lower-semi-1}
Assume that $T(\boldsymbol{x})$ is differentiable with respect to $\boldsymbol{x}$. Then the functional $\mathcal{D}(\bm{u})$ is lower semi-continuous, and consequently $\mathcal{L}_\lambda(\bm{\omega})$ from (\ref{eq-modify-min}) is also LSC; that is, for each $\epsilon>0$ and $\bm{\omega}=(\bm{u},f) \in \mathcal{M}(\Omega)$ with $\mathcal{D}(\boldsymbol{u})<\infty$, there exists a $\delta_\epsilon>0$ such that for all $\bar{\bm{\omega}}=(\bar{\bm{u}},\bar{f}) \in \mathcal{M}(\Omega)$ satisfying $\|\boldsymbol{u}-\bar{\bm{u}}\|_{L^2(\Omega)}<\delta_\epsilon$, the inequality $\mathcal{D}(\boldsymbol{u})<\mathcal{D}(\bar{\bm{u}})+\epsilon$ holds.
\end{lemma}
\begin{proof}
Since the function $T(\boldsymbol{x})$ is differentiable, there exists a real number $\theta \in(0,1)$ such that
$$
T(\bm{\varphi}+\bm{u})=T(\bm{\varphi}+\bar{\bm{u}})+\nabla T(\bm{s}) \cdot \boldsymbol{h},
$$
where $\boldsymbol{s}=\bm{\varphi}+\theta \boldsymbol{u}+(1-\theta) \bar{\bm{u}}$ and $\boldsymbol{h}=\boldsymbol{u}-\bar{\bm{u}}$. Hence we have
$$
\mathcal{D}(\boldsymbol{u})=\mathcal{D}(\bar{\bm{u}})+\int_{\Omega}\left((T(\bm{\varphi}+\bar{\bm{u}})-R(\boldsymbol{x}))(\nabla T(\boldsymbol{s}) \cdot \boldsymbol{h})+\frac{1}{2}\boldsymbol{h}^T\left(\nabla T(\boldsymbol{s}) \nabla T(\boldsymbol{s})^T\right) \boldsymbol{h}\right) d \bm{x}.
$$
From (\ref{assump1}), the above equation leads to
$$
|\mathcal{D}(\boldsymbol{u})-\mathcal{D}(\bar{\bm{u}})| \leq c_1\|\boldsymbol{u}-\bar{\bm{u}}\|_{L^2(\Omega)}+c_2\|\boldsymbol{u}-\bar{\bm{u}}\|_{L^2(\Omega)}^2, c_1 \geq 0, c_2 \geq 0 .
$$
In this case, we have, if taking $\delta(\epsilon)=\left(-c_1+\sqrt{c_1^2+4 \epsilon c_2}\right) /\left(2 c_2\right), \mathcal{D}(\boldsymbol{u})<\mathcal{D}(\bar{\bm{u}})+\epsilon$. Consequently, combined with Lemma \ref{lemma-lower-semi-2}, that is, the lower semicontinuity of $\mathcal{S}(\bm{\omega})$ and $\mathcal{C}(\bm{\omega})$, the functional $\mathcal{L}_\lambda(\bm{\omega})$ from (\ref{eq-modify-min}) is LSC. This proves the lemma.
\end{proof}

The following Theorem illustrates that a more frequently used condition than Proposition 38.12. in \cite{EZeidler1985} may be necessary
to guarantee the existence of minimizers of $\mathcal{L}_\lambda(\bm{\omega})$ over $\mathcal{M}(\Omega)$ in (\ref{eq:eqALM1}).

\begin{theorem}[Existence] A functional $\mathcal{L}_\lambda(\bm{\omega}): \mathcal{M}_{\bm{\omega}_0,r}(\Omega) \subset [\mathcal{H}^1(\Omega)]^{d+1} \rightarrow[-\infty, \infty]$ on the convex, closed, and non-empty subset $\mathcal{M}_{\bm{\omega}_0,r}(\Omega)$ of the real reflexive Banach space $[\mathcal{H}^1(\Omega)]^{d+1}$ satisfies (\ref{eq-coea}), and is weak sequentially lower semi-continuous. Thus, the problem (\ref{eq-modify-min}) possesses a minimum, in other words, the minimization problem (\ref{eq:eqALM1}) admits a minimum.
\end{theorem}
\begin{proof}
Based on \cite[Proposition 38.12(c)]{EZeidler1985}, the minimization problem (\ref{eq-modify-min}) has a minimum due to the lower semicontinuity of $\mathcal{L}_\lambda$, the closed boundedness and convexity of $\mathcal{M}_{\bm{\omega}_0,r}(\Omega)$, and the reflexivity of the Banach space $[\mathcal{H}^1(\Omega)]^{d+1}$. From Lemma \ref{lemma-2}, we further show that the minimization problem (\ref{eq:eqALM1}) admits a minimum.
\end{proof}

\section{Numerical analysis and implementation of the proposed algorithm}\label{sec:Section4}
In this section, we present a numerical approach to solve the diffeomorphic model \eqref{eq:eqALM1} by exploiting the optimization-discretization method. We start by using the alternating iteration strategy to solve \eqref{eq:eqALM1} and deduce the Euler-Lagrange equations of two subproblems. Subsequently, we provide a brief overview of the discretization of the Euler-Lagrange equations and the deformation correction strategy. Finally, we design a diffeomorphic image registration algorithm with a multilevel strategy.
\subsection{Optimization}
To deal with the stability of the solution scheme, we add a "proximal-point" term to the variable $\bm{u}$ in \eqref{eq:eqALM1}, leading to the iterative minimization problem defined by
\begin{equation}\label{eq:eqPPP}
\begin{aligned}
\min_{(\bm{u}, f)\in\mathcal{M}(\Omega)}\Big\{\mathcal{L}^k_\lambda(\bm{u},f;\bm{u}^k)
=\mathcal{D}(\bm{u})&+\frac{\tau_1}{2}\int_{\Omega}\|\nabla \bm{u}\|^2d\bm{x}+\tau_2 \int_{\Omega} \phi\big(f(\bm{x})\big) \mathrm{d} \bm{x}+\frac{\tau_3}{2}\int_{\Omega}\|\nabla f(\bm{x})\|^2d\bm{x}\\
&+\frac{\lambda^k}{2}\int_{\Omega}\big(\det\big(\nabla(\bm{\varphi}+\bm{u})\big)-f(\bm{x})\big)^2d\bm{x}+\frac{1}{2\gamma}\int_{\Omega}\|\bm{u}-\bm{u}^k\|^2d\bm{x}\Big\}.
\end{aligned}
\end{equation} 
where the proximal-point term with parameter $\gamma>0$ is often used to regularize the minimization problem. This regularization can help improve the condition of the problem. However, choosing a larger value for $\gamma$ can lead to an unstable solution. On the other hand, selecting a smaller value for $\gamma$ can enhance the stability of the algorithm, but it may also cause the errors $\bm{u}^{k+1}-\bm{u}^k$ to approach zero quickly, which can affect the precision of the solution \cite{Dong_2013,Parikh_2014}.

Next, we split \eqref{eq:eqPPP} into two subproblems and then employ an alternating iterative strategy to solve them.
\begin{itemize}
  \item The $\bm{u}$-subproblem
\begin{equation}\label{eq:u-subproblem}
      \bm{u} =\argmin_{\bm{u}\in\mathcal{V}}\left\{\mathcal{D}(\bm{u})+\frac{\tau_1}{2}\int_{\Omega}\|\nabla \bm{u}\|^2d\bm{x}+\frac{\lambda^k}{2}\int_{\Omega}\big(\det\big(\nabla(\bm{\varphi}+\bm{u})\big)-f(\bm{x})\big)^2d\bm{x}
      +\frac{1}{2\gamma}\int_{\Omega}\|\bm{u}-\bm{u}^k\|^2d\bm{x} \right\}.
\end{equation}

  \item The $f$-subproblem
  \begin{equation}\label{eq:f-subproblem}
f =\argmin_{f\in \mathcal{W}} \left\{ \tau_2 \int_{\Omega} \phi\big(f(\bm{x})\big) d\bm{x}+\frac{\tau_3}{2}\int_{\Omega}\|\nabla f(\bm{x})\|^2d\bm{x}+\frac{\lambda^k}{2}\int_{\Omega}\big(\det\big(\nabla(\bm{\varphi}+\bm{u})\big)-f(\bm{x})\big)^2d\bm{x} \right\}.
  \end{equation}
\end{itemize}
Hence, the $k$-th update formulas of solving problem \eqref{eq:eqPPP} are as follows
\begin{equation}\label{eq:P}
\left\{\begin{aligned}
\bm{u}^{k+1} =\argmin_{\bm{u}\in\mathcal{V}} &\left\{\frac{1}{2}\int_{\Omega}\big(T(\bm{\varphi}^{k}+\bm{u})-R\big)^2d\bm{x}+\frac{\tau_1}{2}\int_{\Omega}\|\nabla \bm{u}\|^2d\bm{x} \right. \\
&\hspace{0.5cm}\left.+\frac{\lambda^k}{2}\int_{\Omega}\big(\det\big(\nabla(\bm{\varphi}^{k}+\bm{u})\big)-f^k\big)^2d\bm{x} 
+\frac{1}{2\gamma}\int_{\Omega}\|\bm{u}-\bm{u}^k\|^2d\bm{x} \right\}, \\
f^{k+1} =\argmin_{f\in \mathcal{W}} &\left\{\tau_2 \int_{\Omega} \phi(f) d\bm{x}+\frac{\tau_3}{2}\int_{\Omega}\|\nabla f\|^2d\bm{x} +\frac{\lambda^k}{2}\int_{\Omega}\big(\det\big(\nabla(\bm{\varphi}^{k}+\bm{u}^{k+1})\big)-f\big)^2d\bm{x} \right\},\\
\lambda^{k+1} = \rho \lambda^k,\;\;\;\; &\bm{\varphi}^{k+1}=\bm{\varphi}^{k}+\bm{u}^{k+1},
\end{aligned}\right.
\end{equation}
where $\rho>1$ is the growth factor of penalty parameter $\lambda$. 

Next, we calculate the G\^{a}teaux derivative of the Jacobian determinant constraint in equation \eqref{eq:P}. For illustrative purposes, we will consider the 2D case, but the methodology can be straightforwardly extended to the 3D scenario. Before we begin, we present three lemmas as follows.
\begin{lemma}\label{lemma_01}
Let $\epsilon$ be a sufficiently small constant and vector $\bm{v}\in\mathcal{V}$ be a suitable perturbation of $\bm{u}\in\mathcal{V}$, assuming that $\det(\nabla\bm{\varphi})\neq 0$, then one has
\begin{equation*}
\lim\limits_{\epsilon\rightarrow 0} \frac{\det(\nabla\bm{\varphi}+\epsilon\nabla\bm{v})-\det(\nabla\bm{\varphi})}{\epsilon}= \det(\nabla\bm{\varphi})\;{\rm{trace}}\big(\nabla\bm{v}(\nabla\bm{\varphi})^{-1}\big)=\det(\nabla\bm{\varphi})(\nabla \bm{\varphi})^{-\top} \cdot \nabla \bm{v},
\end{equation*}
where $A\cdot B$ denotes the matrix inner product $\sum_{i,j=1}^2 A_{ij}B_{ij}$ between $A\in \mathbb{R}^{2\times 2}$ and $B\in \mathbb{R}^{2\times 2}$. 
\end{lemma}
\begin{proof}
Firstly, from the properties of the determinant, we easily show
\begin{equation*}
\det(\nabla\bm{\varphi}+\epsilon\nabla\bm{v})-\det(\nabla\bm{\varphi})=\det \big( \nabla\bm{\varphi}+\epsilon\nabla\bm{v}(\nabla\bm{\varphi})^{-1}\nabla\bm{\varphi}\big)-\det(\nabla\bm{\varphi})=\Big(\det\big(\bm{I}_d+\epsilon\nabla\bm{v}(\nabla\bm{\varphi})^{-1}\big)-1\Big)\det(\nabla\bm{\varphi}).
\end{equation*}
Secondly, assume that $\nabla \bm{v}(\nabla \bm{\varphi})^{-1}=:\left( {\begin{array}{*{20}{c}}
  {a_{11}}&{a_{12}}\\
  {a_{21}}&{a_{22}}
  \end{array}} \right)$, we can deduce
\begin{equation}
\det\big(\bm{I}_d+\epsilon \nabla \bm{v}(\nabla \bm{\varphi})^{-1} \big)-1=
\left|\begin{array}{*{20}{c}}
  {1+\epsilon a_{11}}&{\epsilon a_{12}}\\
  {\epsilon a_{21}}&{1+\epsilon a_{22}}\\
  \end{array} \right|-1=\epsilon a_{11}+\epsilon a_{22}+\epsilon^2 a_{11}a_{22}-\epsilon^2 a_{12}a_{21},
\end{equation}
then it is easy to check that
$$ \lim _{\epsilon \rightarrow 0}\frac{{\det}\big(\bm{I}_{d}+\epsilon \nabla \bm{v}\left(\nabla \bm{\varphi}\right)^{-1}\big)-1}{\epsilon} = a_{11}+ a_{22} = \text{trace}( \nabla \bm{v}(\nabla \bm{\varphi})^{-1}),$$
hence we obtain
\begin{equation}\label{eq:appdiex-equal-2}
\lim\limits_{\epsilon\rightarrow 0} \frac{\det(\nabla\bm{\varphi}+\epsilon\nabla\bm{v})-\det(\nabla\bm{\varphi})}{\epsilon}= \det(\nabla\bm{\varphi})\text{trace}\big(\nabla\bm{v}(\nabla\bm{\varphi})^{-1}\big).
{}\end{equation}
Finally, assuming that
$\nabla\bm{v}:=\left( {\begin{array}{*{20}{c}}
  {t_{11}}&{t_{12}}\\
  {t_{21}}&{t_{22}}
  \end{array}} \right)$ and
$(\nabla \bm{\varphi})^{-1}:=\left( {\begin{array}{*{20}{c}}
  {b_{11}}&{b_{12}}\\
  {b_{21}}&{b_{22}}
  \end{array}} \right)$, 
 we have
\[ \text{trace}\big(\nabla \bm{v}(\nabla \bm{\varphi})^{-1}\big)=t_{11}b_{11}+t_{12}b_{21}+t_{21}b_{12}+t_{22}b_{22}\]
and
\[(\nabla \bm{\varphi})^{-\top} \cdot \nabla \bm{v}=\left( {\begin{array}{*{20}{c}}
  {b_{11}}&{b_{21}}\\
  {b_{12}}&{b_{22}}
  \end{array}} \right) \cdot
\left( {\begin{array}{*{20}{c}}
  {t_{11}}&{t_{12}}\\
  {t_{21}}&{t_{22}}
  \end{array}} \right)
=t_{11}b_{11}+t_{12}b_{21}+t_{21}b_{12}+t_{22}b_{22}.\] 
Therefore, we can deduce that 
\begin{equation}\label{eq:appdiex-equal-2b}
\text{trace}\big(\nabla \bm{v}(\nabla \bm{\varphi})^{-1}\big)=(\nabla \bm{\varphi})^{-\top} \cdot \nabla \bm{v},
\end{equation} 
which shows the assertion.
\end{proof}

\begin{lemma}\label{lemma3-8}
Let $ \bm{\nu}\in\mathcal{V} $ be an arbitrary perturbation of $ \bm{u} $. Thus, the G\^{a}teaux derivative of $\mathcal{C}(\bm{u}):=\int_{\Omega}(\det(\nabla(\bm{\varphi}+\bm{u}))-f)^2 d\bm{x}$ is given by
\begin{equation*}
\begin{aligned}
\frac{\partial\mathcal{C}(\bm{u})}{\partial\bm{u}}\bm{\nu}&=\int_\Omega\left\langle -2\nabla\cdot\mathcal{U}(\bm{u})(\bm{x}),\bm{\nu}\right\rangle d\bm{x},
\end{aligned}
\end{equation*}
where $\mathcal{U}(\bm{u})(\bm{x}):=\det(\nabla(\bm{\varphi}+\bm{u}))(\nabla(\bm{\varphi}+\bm{u}))^{-\top}(\det(\nabla(\bm{\varphi}+\bm{u}))-f)$.
\end{lemma}
\begin{proof}
For any $\bm{u}, \bm{\nu}\in\mathcal{V}$, let us define outer normal vector $\bm{n}$ on boundary $\partial\Omega$ and any small enough constant $\epsilon$, from \eqref{eq:appdiex-equal-2} and \eqref{eq:appdiex-equal-2b} in \Cref{lemma_01} one has
\[\int_{\partial\Omega}\left\langle \frac{\partial\mathcal{U}(\bm{u})(\bm{x})}{\partial\bm{n}},\bm{\nu}\right\rangle d\bm{s}=0,\]
and
\begin{align*}\label{eq:eqvarphi}
\frac{\partial\mathcal{C}(\bm{u})}{\partial\bm{u}}\bm{\nu}=&\lim\limits_{\epsilon\rightarrow 0} 
\frac{\mathcal{C}[\bm{u}+\epsilon \bm{\nu}]- \mathcal{C}(\bm{u})}{\epsilon}\\
=&\lim\limits_{\epsilon\rightarrow 0}\int_\Omega\frac{[\det(\nabla(\bm{\varphi}+\bm{u})+\epsilon\nabla\bm{\nu})-\det(\nabla(\bm{\varphi}+\bm{u}))]}{\epsilon}\\
&\hspace{4cm}\cdot[\det(\nabla(\bm{\varphi}+\bm{u})+\epsilon\nabla\bm{\nu})+\det(\nabla(\bm{\varphi}+\bm{u}))-2f]d\bm{x}\\
\xlongequal[]{\eqref{eq:appdiex-equal-2}}
&\int_\Omega\det(\nabla(\bm{\varphi}+\bm{u}))\cdot trace(\nabla\bm{\nu}(\nabla(\bm{\varphi}+\bm{u}))^{-1})[2\det(\nabla(\bm{\varphi}+\bm{u}))-2f]d\bm{x}\\
\xlongequal[]{\eqref{eq:appdiex-equal-2b}}
&\int_\Omega2\left\langle\det(\nabla(\bm{\varphi}+\bm{u}))(\nabla(\bm{\varphi}+\bm{u}))^{-\top}(\det(\nabla(\bm{\varphi}+\bm{u}))-f),\nabla\bm{\nu}\right\rangle d\bm{x}\\
=&\int_\Omega2\left\langle \mathcal{U}(\bm{u})(\bm{x}),\nabla\bm{\nu}\right\rangle d\bm{x}-\int_{\partial\Omega}2\left\langle \frac{\partial\mathcal{U}(\bm{u})(\bm{x})}{\partial\bm{n}},\bm{\nu}\right\rangle d\bm{s}\\
=&\int_\Omega\left\langle -2\nabla\cdot \mathcal{U}(\bm{u})(\bm{x}),\bm{\nu}\right\rangle d\bm{x},
\end{align*}
which proves the assertion.
\end{proof}

\begin{lemma}\label{lemma_02}
Let $\bm{u}\in\mathcal{V}$ be a suitable perturbation of $\bm{\varphi}$, then one has
\begin{equation*}
\begin{aligned}
\nabla\cdot\big[\det(\nabla(\bm{\varphi}+\bm{u}))(\nabla(\bm{\varphi}+\bm{u}))^{-\top}&\big(\det(\nabla(\bm{\varphi}+\bm{u}))-f\big)\big]\\ =&\det(\nabla(\bm{\varphi}+\bm{u}))(\nabla(\bm{\varphi}+\bm{u}))^{-\top}  \nabla\big(\det(\nabla(\bm{\varphi}+\bm{u}))-f \big).
\end{aligned}
\end{equation*}
\end{lemma}
\begin{proof}
Since we have the inclusion from the property of the determinant
\begin{equation}\label{eq:det}
\det(\nabla(\bm{\varphi}+\bm{u}))(\nabla(\bm{\varphi}+\bm{u}))^{-\top} =
(\nabla(\bm{\varphi}+\bm{u}))^{*\top} :=  
\left( {\begin{array}{*{20}{c}}
  {{(\varphi^2+u^2)_{y}}}&{ - {(\varphi^2+u^2) _{x}}}\\
  { - {(\varphi^1+u^1)_{y}}}&{{(\varphi^1+u^1)_{x}}}
  \end{array}} \right),
\end{equation}
so it is easy to obtain
\begin{equation}\label{eq:free-div}
\nabla\cdot \big[\det(\nabla{(\bm{\varphi}+\bm{u})})(\nabla(\bm{\varphi}+\bm{u}))^{-\top} \big] =\bm{0}.
\end{equation}
In particular, the formula \eqref{eq:free-div} also holds for the 3D case. 
Finally, we simplify 
\begin{equation*}
\begin{split}
\nabla\cdot\big[\det(\nabla(\bm{\varphi}+\bm{u}))&(\nabla(\bm{\varphi}+\bm{u}))^{-\top}\big(\det(\nabla(\bm{\varphi}+\bm{u}))-f\big)\big]\\ =&\big(\det(\nabla(\bm{\varphi}+\bm{u}))-f\big)\nabla\cdot\big[\det(\nabla(\bm{\varphi}+\bm{u}))(\nabla(\bm{\varphi}+\bm{u}))^{-\top}\big]  \\
&+\det(\nabla(\bm{\varphi}+\bm{u}))(\nabla(\bm{\varphi}+\bm{u}))^{-\top} \nabla \big(\det(\nabla(\bm{\varphi}+\bm{u}))-f \big) \\
=&\det(\nabla(\bm{\varphi}+\bm{u}))(\nabla(\bm{\varphi}+\bm{u}))^{-\top}  \nabla\big(\det(\nabla(\bm{\varphi}+\bm{u}))-f \big),
\end{split}
\end{equation*}
which proves the assertion.
\end{proof}

To obtain the optimal solution of the $\bm{u}$-subproblem, we consider its Euler-Lagrange equation, which is equivalent to solving
\begin{equation}\label{eq:eqEUsub1}
\begin{aligned}
\big(T(\bm{\varphi}^{k}&+\bm{u})-R \big)\nabla T-\tau_1\Delta\bm{u}+\frac{1}{\gamma}(\bm{u}-\bm{u}^k)\\
&-\lambda^k\det\big(\nabla(\bm{\varphi}^{k}+\bm{u})\big)\big(\nabla(\bm{\varphi}^{k}+\bm{u})\big)^{-\top}\nabla\big[\det\big(\nabla(\bm{\varphi}^{k}+\bm{u})\big)-f^k\big]=\bm{0}.\\
\end{aligned}
\end{equation} 
The above nonlinear systems \eqref{eq:eqEUsub1} with the boundary condition $\bm{u}=\bm{0}$ can be rewritten as 
\begin{equation}\label{eq:UEuler}
\left\{ \begin{array}{rll}
-\tau_1\Delta\bm{u}+\frac{1}{\gamma}\bm{u}&=\bm{r}(\bm{u}), &\text{in}~\Omega, \\
   \bm{u} &=\bm{0},         &\text{on}~\partial\Omega ,
\end{array}\right.
\end{equation}
where $ \bm{r}(\bm{u})=-\big(T(\bm{\varphi}^{k}+\bm{u})-R\big)\nabla T+\lambda^k\det\big(\nabla(\bm{\varphi}^{k}+\bm{u})\big)\big(\nabla(\bm{\varphi}^{k}+\bm{u})\big)^{-\top}\nabla\big[\det\big(\nabla(\bm{\varphi}^{k}+\bm{u})\big)-f^k\big]+\frac{1}{\gamma}\bm{u}^k $.
The nonlinear partial differential equations \eqref{eq:UEuler} are difficult to be directly solved, here a simple linearization technique is used, leading to the systems as follows
\begin{equation}\label{eq:UEuler-k}
  \left\{ \begin{array}{rll}
    -\tau_1\Delta\bm{u}+\frac{1}{\gamma}\bm{u}&=\bm{r}(\bm{u}^{k}), &\text{in}~\Omega, \\
     \bm{u}&=\bm{0},       &\text{on}~\partial\Omega.
  \end{array}\right.
\end{equation}

Once the solution $ \bm{u}^{k+1}$ of \eqref{eq:UEuler-k} is obtained, we minimize the $f$-subproblem by fixing $ \bm{u}^{k+1} $. In other words, we
look for $ f^{k+1}(\bm{x}) $ by solving the Euler-Lagrange equation
\begin{equation}\label{eq:eqGf}
\tau_2 d\phi(f)-\tau_3\Delta f-\lambda\big(\det\big(\nabla(\bm{\varphi}^k+\bm{u}^{k+1})\big)-f\big)=0,
\end{equation}
where $ d\phi(f):=d\phi_1(f) = 1 - 1/f^2 $ and $ d\phi(f):=d\phi_2(f)= \log(f)+1-1/f $ for the functions $\phi_1(f)$ and $\phi_2(f)$, respectively.
Combined with the boundary condition $f(\bm{x})=1$ on $\partial\Omega$, the optimal solution of the $f$-subproblem can be seen as the solution of the system defined by
\begin{equation}\label{eq:fEuler}
\left\{\begin{array}{rll}
-\tau_3\Delta f+{\lambda}f&={\lambda}\det\big(\nabla(\bm{\varphi}^k+\bm{u}^{k+1}))\big)-\tau_2 d\phi(f^k),  &\text{in}~\Omega, \\
 f &= {1}, &\text{on}~\partial\Omega.\\
\end{array}\right.
\end{equation}

\begin{remark}\label{remk1}
As described above, our model can theoretically guarantee that $f(\bm{x})$ is always greater than zero. However, due to the selection of multiple modeling parameters and the lack of proper handling of the nonlinear part of \eqref{eq:fEuler}, $f(\bm{x})$ may not always be greater than zero everywhere. To prevent grid folding, we use the deformation correction technique in our numerical implementation (see \Cref{DJd} for more details).
\end{remark}

\begin{remark}\label{remk2}
We observe that it is not necessary to add a "proximal point" term in the $f(\bm{x})$ subproblem in \eqref{eq:f-subproblem}, as the Euler-Lagrange equation~\eqref{eq:fEuler} already contains the zero-order derivative term of $f(\bm{x})$, and $\lambda$ gradually increases.
\end{remark}

\subsection{Discretization}
In this subsection, we will provide a brief overview of the discretization of systems \eqref{eq:UEuler-k} and \eqref{eq:fEuler}, as well as the deformation correction strategy. Without loss of generality, we will focus on the 2D case, although this technique can be extended to 3D. A uniform Cartesian grid of $m \times n$ is created in the image domain $\Omega:=[0, 1]^2$. Each pixel is a box with lengths $h_x = 1/m$ and $h_y = 1/n$, which are associated with cell-centered grid points. The discrete image domain is denoted by $\Omega_{h}=\{(x_i,y_j) \;|\;x_i=(i-1/2)h_x,y_j=(j-1/2)h_y,i=1,\ldots, m,\;j=1,\ldots, n\}$, and $(i, j)$ is the coordinate index $(x_i,y_j)$.

Typically, the finite difference scheme is used to approximate the first- and second-order derivatives in image processing. We denote the discrete schemes applied to  $v\in \mathbb{R}$ and $\bm{u}=(u^1,u^2)^{\top}\in \mathbb{R}^2$ at the grid point $(i, j)$ by
\begin{equation*}\label{eq:discretization-operater}
  \begin{array}{cl} 
    \begin{aligned}
      (\nabla v)_{i,j}&=(\delta_x v_{i,j},\delta_y v_{i,j})^{\top},\quad~ (\nabla \cdot \bm{u})_{i,j}=\delta_x u^1_{i,j}+\delta_y u^2_{i,j}, \\
      (\Delta u^\ell)_{i,j}&=\delta_{xx} u^\ell_{i,j}+\delta_{yy} u^\ell_{i,j},\quad (\Delta \bm{u})_{i,j}=\left((\Delta u^1)_{i,j},(\Delta u^2)_{i,j}\right)^{\top},
    \end{aligned}
  \end{array}
\end{equation*}
where $ \ell=1 \text{ or } 2 $, and
\begin{equation*}\label{eq:discretization-dev}
  \begin{array}{cl} 
    \begin{aligned}
      \delta_x v_{i,j}&=(v_{i+1,j}-v_{i-1,j})/(2h_x),\qquad\quad~\, \delta_y v_{i,j}=(v_{i,j+1}-v_{i,j-1})/(2h_y), \\
      \delta_{xx} u^\ell_{i,j} &= (u^\ell_{i-1,j}-2u^\ell_{i,j}+u^\ell_{i+1,j})/h_x^2,\quad \delta_{yy} u^\ell_{i,j} = (u^\ell_{i,j-1}-2u^\ell_{i,j}+u^\ell_{i,j+1})/h_y^2.
    \end{aligned}
  \end{array}
\end{equation*}

\subsubsection{The discretizations of two subproblems}
To proceed, let us define $\mathcal{C}=\det\big(\nabla(\bm{\varphi}+\bm{u})\big)-f$. Based on the above definitions,  the finite difference approximation $ \mathcal{C}^k_{i,j} $ of $\mathcal{C}$ at the grid point $(i, j)$ can be denoted as
\begin{equation*}\label{eq:discretization-u-subproblem}
  \begin{aligned}
    \mathcal{C}_{i,j}^k
    :=&\big[ \operatorname{det}\big( \nabla(\bm{\varphi}^k+\bm{u}^{k})\big) - f^k\big] \Big|_{i,j} 
    = \left| {\begin{array}{*{20}{c}}
        \delta_{x}(\varphi^{1,k}_{i,j} + u^{1,k}_{i,j})&\delta_{y}(\varphi^{1,k}_{i,j} + u^{1,k}_{i,j})\\
        \delta_{x}(\varphi^{2,k}_{i,j} + u^{2,k}_{i,j})&\delta_{y}(\varphi^{2,k}_{i,j} + u^{2,k}_{i,j})
    \end{array}} \right| - f^k_{i,j} \\
   =&(\delta_{x}\varphi^{1,k}_{i,j}+\delta_{x}u^{1,k}_{i,j})(\delta_{y}\varphi^{2,k}_{i,j}+\delta_{y}u^{2,k}_{i,j})-(\delta_{y}\varphi^{1,k}_{i,j}+\delta_{y}u^{1,k}_{i,j})(\delta_{x}\varphi^{2,k}_{i,j}+\delta_{x}u^{2,k}_{i,j})-f^k_{i,j}.
  \end{aligned}
\end{equation*}
See also \Cref{DJd} for more details. From \Cref{eq:det}, it is easy to obtain
\begin{equation*}
\begin{split}
  \det\big(\nabla(\bm{\varphi}&+\bm{u})\big)\big(\nabla(\bm{\varphi}+\bm{u})\big)^{-\top}\nabla \mathcal{C}\\
  &=  
  \left( {\begin{array}{*{20}{c}}
      {\delta_{y}(\varphi^2+u^2)}&{ - \delta_{x}(\varphi^2+u^2)}\\
      {- \delta_{y}(\varphi^1+u^1)}&{\delta_{x}(\varphi^1+u^1)}
  \end{array}} \right) (\delta_{x}\mathcal{C}, \delta_{y}\mathcal{C})^{\top}\\
  &= \Big(\delta_{x}\mathcal{C}\cdot\delta_{y}(\varphi^2+u^2)-\delta_{y}\mathcal{C}\cdot\delta_{x}(\varphi^2+u^2), -\delta_{x}\mathcal{C}\cdot\delta_{y}(\varphi^1+u^1)+\delta_{y}\mathcal{C}\cdot\delta_{x}(\varphi^1+u^1) \Big)^{\top}, 
\end{split} 
\end{equation*}
hence the finite difference approximation of the $\bm{u}$-subproblem \eqref{eq:UEuler-k} is given by
\begin{equation}\label{eq:discretization-u}
  \left\{\begin{array}{ll}  
    \begin{aligned}
      -\tau_1(\delta_{xx}+\delta_{yy}+\frac{1}{\tau_1\gamma})u^{1,k+1}_{i,j}
      =&-(T^k-R)_{i,j}\delta_x T^k_{i,j}+\frac{1}{\gamma}u^{1,k}_{i,j}\\
            +&\lambda^k \big[(\delta_{y}\varphi^{2,k}_{i,j}+\delta_{y}u^{2,k}_{i,j}) \delta_{x} \mathcal{C}_{i,j}^k - (\delta_{x} \varphi^{2,k}_{i,j} +\delta_{x} u^{2,k}_{i,j}) \delta_{y} \mathcal{C}_{i,j}^k \big],\\
      -\tau_1(\delta_{xx}+\delta_{yy}+\frac{1}{\tau_1\gamma})u^{2,k+1}_{i,j}
      =&-(T^k-R)_{i,j}\delta_y T^k_{i,j}+\frac{1}{\gamma}u^{2,k}_{i,j}\\ 
            -&\lambda^k \big[(\delta_{y}\varphi^{1,k}_{i,j}+\delta_{y}u^{1,k}_{i,j})  \delta_{x} \mathcal{C}_{i,j}^k - (\delta_{x} \varphi^{1,k}_{i,j}+ \delta_{x} u^{1,k}_{i,j}) \delta_{y} \mathcal{C}_{i,j}^k \big],\\
    \end{aligned}
  \end{array}\right.
\end{equation}
where $T^k:=T(\bm{\varphi}^{k}+\bm{u}^k)$ and the finite difference approximation of the $f$-subproblem \eqref{eq:fEuler} is given by
\begin{equation}\label{eq:discretization-f2}
  \begin{aligned}
    -\tau_3 (\delta_{xx}+\delta_{yy} )f^{k+1}_{i,j} + \lambda^k f^{k+1}_{i,j} =&~ \lambda^k \big(\delta_{x}\varphi^{1,k+1}_{i,j}+\delta_{x}u^{1,k+1}_{i,j})(\delta_{y}\varphi^{2,k+1}_{i,j}+\delta_{y}u^{2,k+1}_{i,j})\\
    &-\lambda^k(\delta_{y}\varphi^{1,k+1}_{i,j}+\delta_{y}u^{1,k+1}_{i,j})(\delta_{x}\varphi^{2,k+1}_{i,j}+\delta_{x}u^{2,k+1}_{i,j}\big)  - \tau_2 d\phi(f^k_{i,j}),
  \end{aligned}
\end{equation}
where $d\phi(f^k_{i,j}):=d\phi_1(f^k_{i,j}) = 1 - \frac{1}{(f^k_{i,j})^2} $ or $d\phi(f^k_{i,j}):=d\phi_2(f^k_{i,j}) =\log(f^k_{i,j})+1-\frac{1}{f^k_{i,j}} $.

\subsubsection{Discretizing the Jacobian determinant}\label{DJd}
Similar to our previous work \cite{zhang2021diffeomorphic}, 
the Jacobian determinant $\det(\nabla\bm{\varphi})|_{o} $ of the deformation $ o:=\varphi_{i,j} $ at the cell center $ (i, j) $ (\emph{see} \Cref{fig:mesh-2D}) is given by
 \begin{align*}\label{Jac-Dis}
  \left.\operatorname{det}(\nabla \bm{\varphi})\right|_{o}&=
  \left|\begin{array}{cc}
    \delta_{x} \varphi_{i, j}^{1} & \delta_{y} \varphi_{i, j}^{1} \\
    \delta_{x} \varphi_{i, j}^{2} & \delta_{y} \varphi_{i, j}^{2}
  \end{array}\right|
  =\cfrac{1}{4 h_{x} h_{y}}\left|\begin{array}{cc}
    \varphi_{i+1, j}^{1}-\varphi_{i-1, j}^{1} & \varphi_{i, j+1}^{1}-\varphi_{i, j-1}^{1} \\
    \varphi_{i+1, j}^{2}-\varphi_{i-1, j}^{2} & \varphi_{i, j+1}^{2}-\varphi_{i, j-1}^{2}
  \end{array}\right|\\
  &=\cfrac{1}{4 h_{x} h_{y}}\left(\left|\begin{array}{cc}
    \varphi_{i+1, j}^{1}-\varphi_{i, j}^{1} & \varphi_{i, j+1}^{1}-\varphi_{i, j}^{1} \\
    \varphi_{i+1, j}^{2}-\varphi_{i, j}^{2} & \varphi_{i, j+1}^{2}-\varphi_{i, j}^{2}
  \end{array}\right|+\left|\begin{array}{cc}
    \varphi_{i, j+1}^{1}-\varphi_{i, j}^{1} & \varphi_{i-1, j}^{1}-\varphi_{i, j}^{1} \\
    \varphi_{i, j+1}^{2}-\varphi_{i, j}^{2} & \varphi_{i-1, j}^{2}-\varphi_{i, j}^{2}
  \end{array}\right|\right.\\
  &\left.\qquad \qquad +\left|\begin{array}{cc}
    \varphi_{i-1, j}^{1}-\varphi_{i, j}^{1} & \varphi_{i, j-1}^{1}-\varphi_{i, j}^{1} \\
    \varphi_{i-1, j}^{2}-\varphi_{i, j}^{2} & \varphi_{i, j-1}^{2}-\varphi_{i, j}^{2}
  \end{array}\right|+\left|\begin{array}{cc}
    \varphi_{i, j-1}^{1}-\varphi_{i, j}^{1} & \varphi_{i+1, j}^{1}-\varphi_{i, j}^{1} \\
    \varphi_{i, j-1}^{2}-\varphi_{i, j}^{2} & \varphi_{i+1, j}^{2}-\varphi_{i, j}^{2}
  \end{array}\right|\right)\\
  &=\cfrac{1}{2}\left(R_{\Delta obd}^{i, j}+R_{\Delta oda}^{i, j}+R_{\Delta oac}^{i, j}+R_{\Delta ocb}^{i, j}\right) \text {, }
\end{align*}
where $ R_{\Delta obd}^{i, j} $ is an area ratio of the triangle signed area
\begin{equation*}
\frac{1}{2}\left|\begin{array}{cc}
\varphi_{i+1, j}^{1}-\varphi_{i, j}^{1} & \varphi_{i, j+1}^{1}-\varphi_{i, j}^{1} \\
\varphi_{i+1, j}^{2}-\varphi_{i, j}^{2} & \varphi_{i, j+1}^{2}-\varphi_{i, j}^{2}
\end{array}\right|
\end{equation*}
to the area element $ h_xh_y $. 

\begin{figure}[h]
  \centering
  \includegraphics[width=0.4\linewidth]{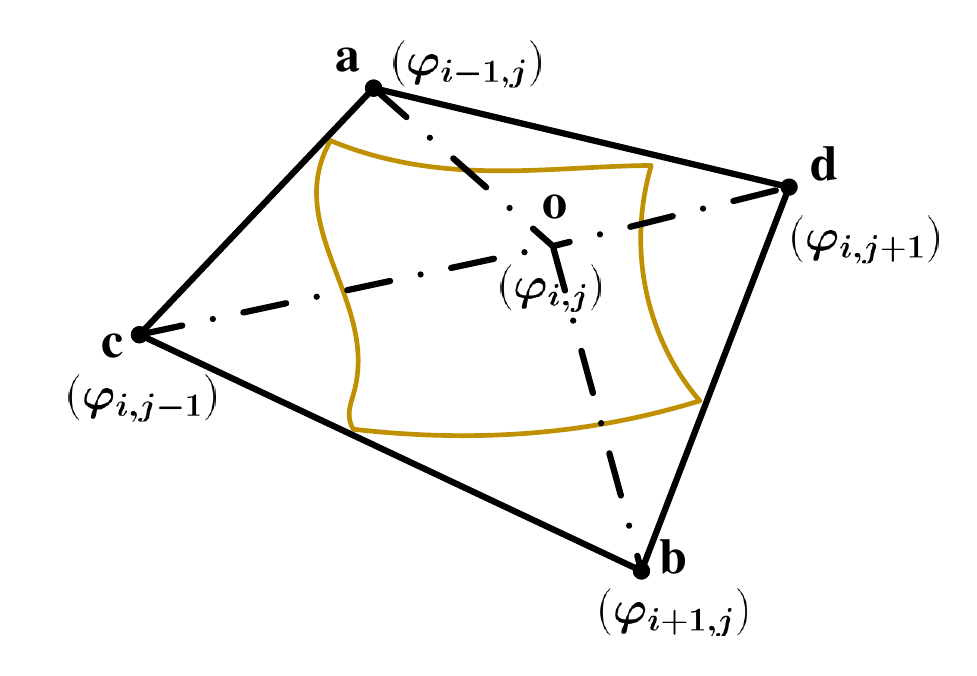}
  \caption{ Finite difference computation involved the Jacobian determinant $\det(\nabla\bm{\varphi})|_{o} $ of the deformation $ o:=\varphi_{i,j} $ at the cell center $ (i, j)$ for $ d=2 $.}
  \label{fig:mesh-2D}
\end{figure}

\begin{definition}
The folding degree \text{Deg}(p) is the number of edges that contain the point $p$ in a triangle with a negative area ratio, e.g., the point $b$ is a folding key point due to $\text{Deg}(b)=5,\text{Deg}(g)=\text{Deg}(o)=3$ in \Cref{fig:correction2}(a).\textbf{}
\end{definition}

If one of the ratios $R_{\Delta\bm{\cdot}}^{i, j}$ is negative while the others are positive, the Jacobian determinant $\det(\nabla\bm{\varphi})_{o}$ may still be positive, even though a twist has occurred. To detect such twists, we introduce a signed area minimum ratio indicator $R_{i, j}$ that is defined by \begin{equation}\label{grid_unfolding_indicator} 
R_{i, j}=\min \left\{R_{\Delta obd}^{i, j}, R_{\Delta oda}^{i, j}, R_{\Delta oac}^{i, j}, R_{\Delta ocb}^{i,j}\right\}. 
\end{equation} 
If $R_{i,j}$ is negative, we will use a deformation correction strategy to correct triangles with negative area ratios. This involves collecting the grid folding points, such as points $g$, $o$ and $b$ in \Cref{fig:correction2}(a), selecting a folding key point with a large folding degree and then moving it to a new location to ensure that the grid unfolding indicators of the five adjacent center points ($e, f, g, o$ and $b^\prime$, respectively), are all positive, \emph{see} \Cref{fig:correction2}(b) and \Cref{fig:correction2}(c). The corresponding grid correction strategy is described in \Cref{alg:DC}. This approach can be extended to 3D, as discussed in \cite{zhang2021diffeomorphic}.
\begin{figure}
    \centering
    \subfigure[Seeking the folding key point]{
        \begin{minipage}[t]{0.32\linewidth}
            \centering
            \includegraphics[width=4.2cm]{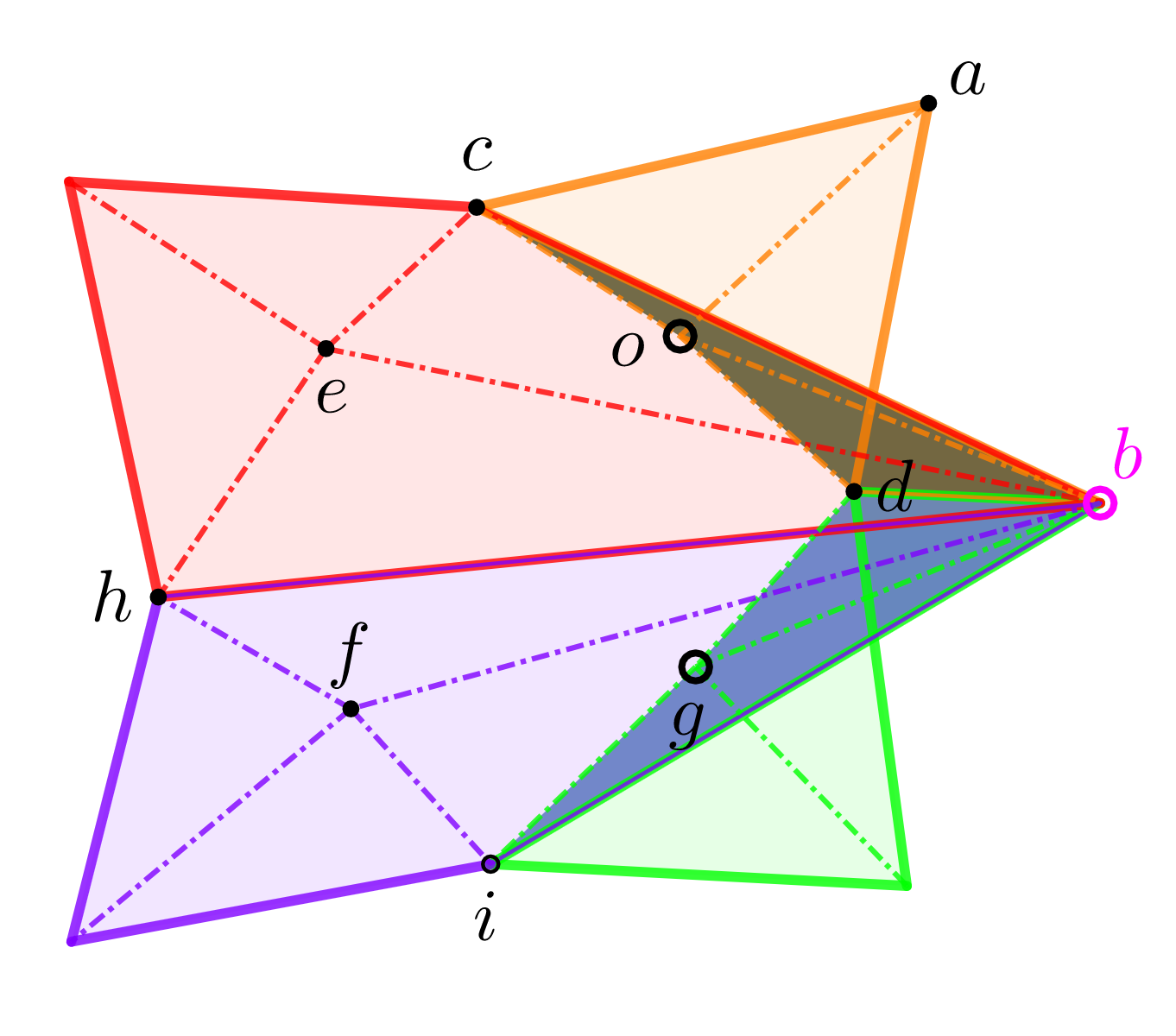}
        \end{minipage}
    }%
    \subfigure[Moving the folding key point]{
        \begin{minipage}[t]{0.32\linewidth}
            \centering
            \includegraphics[width=4cm]{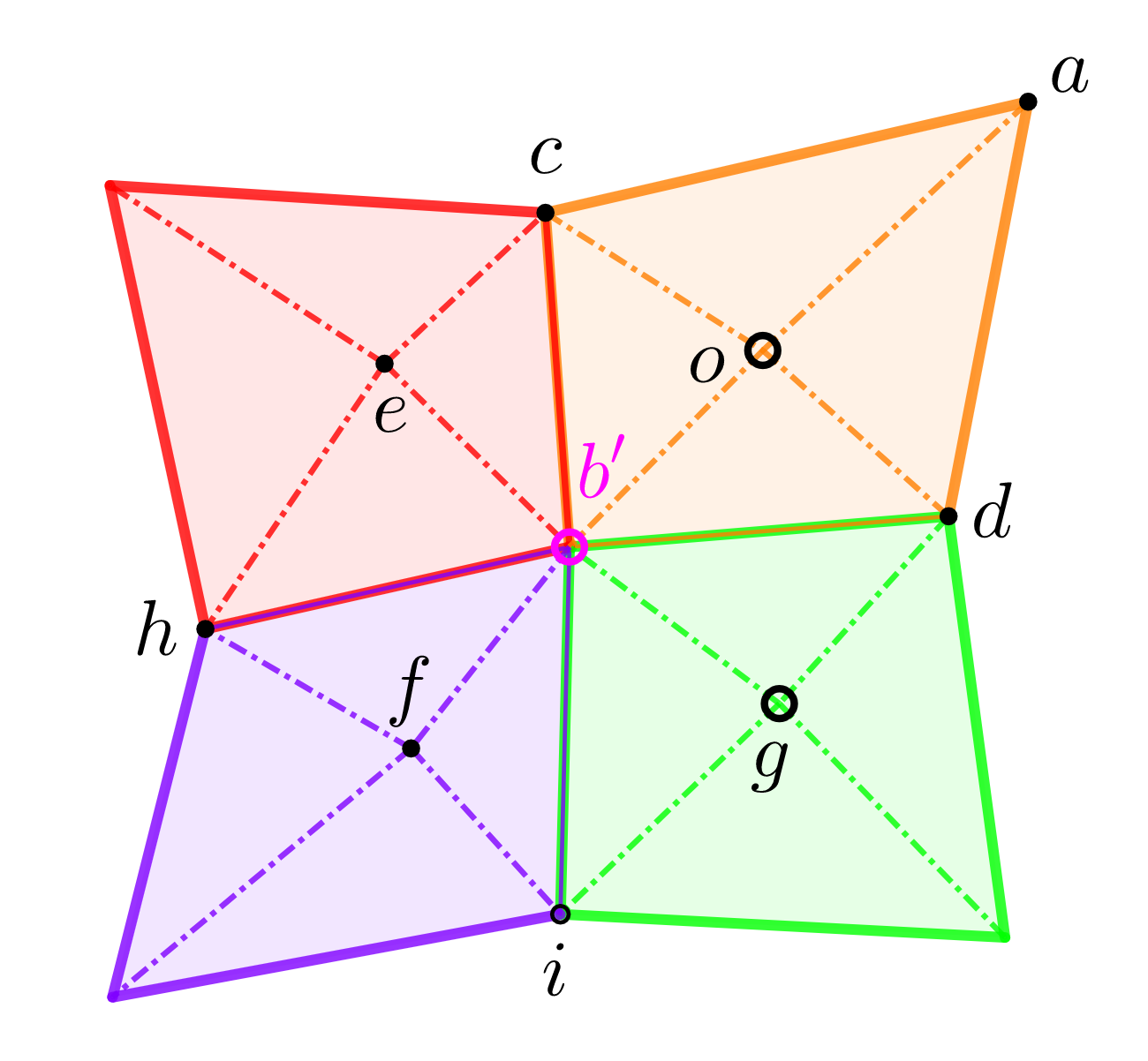}
        \end{minipage}
    }%
    \subfigure[Checking five adjacent cell centers]{
        \begin{minipage}[t]{0.32\linewidth}
            \centering
            \includegraphics[width=4cm]{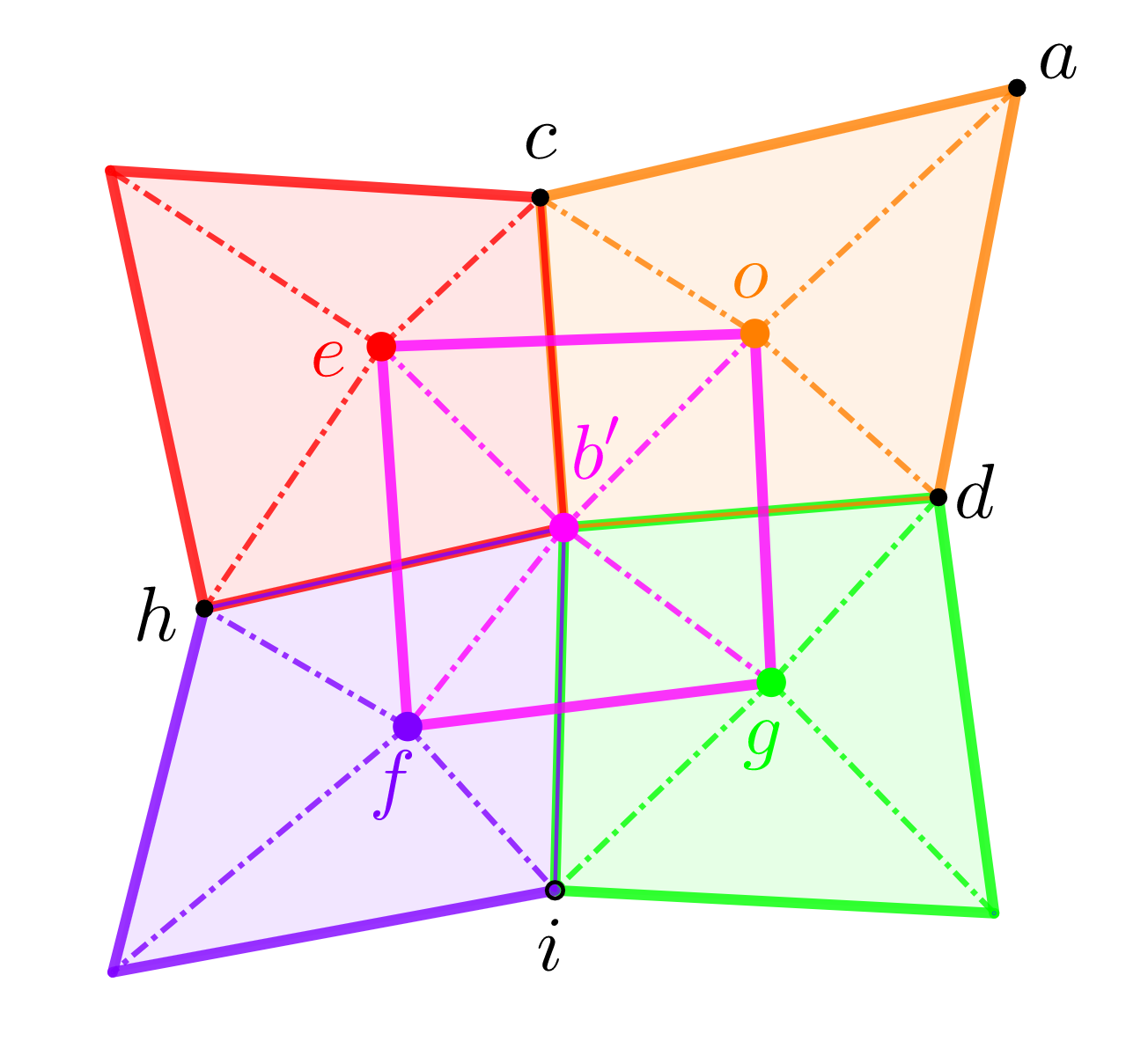}
        \end{minipage}
    }%
  \caption{Grid folding correction.}
  \label{fig:correction2}
\end{figure}

\begin{algorithm}[h]
\caption{Deformation correction}\label{alg:DC}
\setcounter{AlgoLine}{0}
\LinesNumbered
  \KwIn{ $\bm{\varphi}$;}
  
  Initialization: $ \epsilon = 10^{-2}$;
  
  Compute $R_{i, j}$ by using (\ref{grid_unfolding_indicator});
  
  Find $S := \{(i,j)~|~R_{i,j} < \epsilon \}$ and the folding points set $P$ that cause overlaps.\\
  $P := \big\{ P_{i,j} \in \{(i,j),(i-1,j),(i,j-1),(i+1,j),(i,j+1)~|~(i,j)\in S \} \big\}$;
  
  \uIf{$P = \varnothing$}{
    Exit this algorithm;
  }
  \Else{ 
    \ForEach{$(i,j)\in P$}{
      Optimize $\bm{\varphi}_{i,j}$ such that $\min \{ R_{i,j},R_{i-1,j},R_{i,j-1},R_{i+1,j},R_{i,j+1} \} \geq \epsilon$;
    }
  }
  $ R_{min} = \min\{ R_{i,j} \} $;

  \KwOut{ $ \bm{\varphi}, R_{min} $. }  
\end{algorithm}

\subsection{Numerical algorithms}
Finding the global minimum of the displacement field $\bm{u}$-subproblem in the optimization problem \eqref{eq:model2} is generally challenging and computationally difficult due to the nonconvex nature of the problem. To overcome this, we propose an algorithm that addresses the diffeomorphic registration model by incorporating the Jacobian constraint of deformation. In this algorithm, we optimize variants $\bm{u}$ and $f(\bm{x})$ by combining the penalty method with the deformation correction technique for diffeomorphic image registration. The steps to solve the diffeomorphic registration model with the Jacobian determinant constraint are summarized in \Cref{alg:DIRPMnew}. 
\begin{algorithm}[h]
  \caption{Diffeomorphic image registration based on penalty method (DIRPM)}\label{alg:DIRPMnew}
  \setcounter{AlgoLine}{0}
  \LinesNumbered
  
  \KwIn{$R, T, \bm{\varphi}^1, f^1, \tau_1, \tau_2, \tau_3, \lambda^{1},\gamma, \rho, MaxIter $;}
  Initialization: $ \bm{u}^1 = \bm{0}$, $ \epsilon_{\mathcal{L}} = 10^{-4} $, $ \epsilon_\varphi = 10^{-3} $;
  
  \For{$k=1,\ldots, MaxIter$}
  {Determine displacement field $\bm{u}^{k+1}$ by solving the equation \eqref{eq:UEuler-k}; 

    Update deformation field $\bm{\varphi}^{k+1}=\bm{\varphi}^{k}+ \bm{u}^{k+1} $;
    
    Correct $ \bm{\varphi}^{k+1} $ and obtain grid unfolding indicator $ R_{min}^{k+1} $ by \Cref{alg:DC};
  
Determine $ f^{k+1} $ by solving equation \eqref{eq:fEuler}; 

Update the penalty factor $\lambda^{k+1}=\rho\lambda^{k}$;

    \lIf{$\frac{ \|\mathcal{L}_{\lambda}^k(\bm{u}^{k+1},f^{k+1};\bm{u}^{k})-\mathcal{L}_{\lambda}^{k-1}(\bm{u}^{k},f^{k};\bm{u}^{k-1})\|}{\|\mathcal{L}_{\lambda}^1(\bm{u}^{1},f^{1};\bm{u}^{1})\|} \leq \epsilon_{\mathcal{L}} $~\rm{or}~ $\frac{\|\bm{\varphi}^{k+1}-\bm{\varphi}^{k}\|}{\|\bm{\varphi}^{1}\|} \leq \epsilon_\varphi $}
    { 
      break}
  } 

  $T^{new} = T(\bm{\varphi}^{k+1})$;
  
  $\bm{\varphi}^{new} = \bm{\varphi}^{k+1}$;
  
  $f^{new} =  f^{k+1}$;
  
  \KwOut{ $T^{new}, \bm{\varphi}^{new}, f^{new}$.}
\end{algorithm}

\begin{algorithm}
  \caption{Multilevel Registration}\label{alg:MultiLevelRegistration}
  \setcounter{AlgoLine}{0}
  \KwIn{$R, T, L, \tau_1, \tau_2, \tau_3, \lambda,\gamma,\rho, MaxIter$;}
  Initialization: $ R_1=R, T_1=T, \bm{\varphi}_L = \bm{x}, f_{L}=1 $;

  \For{$\ell=2,\ldots, L$}{
    Compute the \emph{reference} and \emph{template} images for the $ \ell$-level coarse grid:\\ $R_\ell$ = $I_h^H R_{\ell-1}, T_\ell$ = $I_h^H T_{\ell-1}$;
  }

  \For{$\ell=L,\ldots, 1$}{
    Call the diffeomorphic image registration \Cref{alg:DIRPMnew} based on penalty method: 
    
    $[T^{new}_\ell, \bm{\varphi}^{new}_\ell, f^{new}_\ell]={\rm{DIRPM}}(R_\ell, T_\ell, \bm{\varphi}_\ell, f_\ell,\tau_1, \tau_2, \tau_3, \lambda, \gamma,\rho, MaxIter$);
    
    \lIf{$ \ell > 1 $}
    {Interpolation (or prolongation): 
      $ \bm{\varphi}_{\ell-1} = I_H^h \bm{\varphi}^{new}_{\ell} $,
      $ f_{\ell-1} = I_H^h f^{new}_{\ell}$
    }
  }
  \KwOut{ $ T^{new}_1, \bm{\varphi}^{new}_1 $. }
\end{algorithm}

Furthermore, we have extended our method to incorporate multilevel registration for liver CT scans, which involves aligning two volumes. This approach, known as multilevel image-to-image registration, aims to reduce computational cost and prevent the algorithm from being trapped in a suboptimal solution. The original registration problem is divided into smaller and faster-to-solve subproblems at different levels. This strategy involves utilizing a coarser level of discretization to obtain an initial approximation for the subsequent finer level. The steps of the multilevel scheme are summarized in Algorithm \ref{alg:MultiLevelRegistration}.

\section{Numerical experiments}\label{sec:Section5}
WWe will assess the effectiveness of the proposed approach and other commonly used methods for 2D and 3D images through qualitative and quantitative evaluations. Qualitative evaluations will involve visualizations of registration results, such as registered images, errors between the registered image $T(\bar{\bm{\varphi}}({\bm{x}}))$ and $R(\bm{x})$, deformation grids, displacement fields, and hotmaps of the Jacobian determinant of the deformation fields. Quantitative evaluations will be conducted using standard metrics, which will include the use of
\begin{itemize} 
    \item[\textbf{i}.)]
    the grid unfolding indicator $R_{\min }$ and the grid folding ratio (GFR) which are defined by
     \begin{equation*}
    R_{\min }=\min\limits _{i, j} R_{i, j},\quad \text{ and }\quad
    {\rm{GFR}}=\frac{\sharp\left(\mathcal{G}\right)}{{\sharp (\Omega_h)}},
    \end{equation*}
    where $\mathcal{G}:=\{(x_i, y_j)\in\Omega_{h}|R_{i,j}<0\}$, $\sharp\left(\mathcal{G}\right)$ and $\sharp(\Omega_h)$ denote the numbers of nodes in $\mathcal{G}$ and $\Omega_h$, respectively.
    \item[\textbf{ii}.)] the Jacobian determinant measures, which are defined by
    \begin{equation*}
    \begin{aligned}
    {\det}(\nabla \bar{\bm{\varphi}})|_{i, j}, \qquad & \overline{{\det}}(J(\bar{\bm{\varphi}}))=\frac{1}{m n} \sum_{i, j} {\det}(\nabla \bar{\bm{\varphi}})|_{i, j}, \\
    {\det}_{\min }(J(\bar{\bm{\varphi}}))=\min _{i, j} {\det}(\nabla \bar{\bm{\varphi}})|_{i, j}, \qquad & {\det}_{\max }(J(\bar{\bm{\varphi}}))=\max _{i, j} {\det}(\nabla \bar{\bm{\varphi}})|_{i, j}.
    \end{aligned}
    \end{equation*}
    \item[\textbf{iii}.)] the relative sum of squared differences ($\rm{Re_{-}SSD}$), which is defined by
    \begin{equation*}
    \text{Re\_SSD}(T,R,T(\bar{\bm{\varphi}}))= \cfrac{\sum_{i,j}(T(\bar{\bm{\varphi}}_{i,j})-R_{i,j})^2}{\sum_{i,j}(T_{i,j}-R_{i,j})^2}.
    \end{equation*}
    \item[\textbf{iv}.)] the \emph{ssim} and \emph{psnr} metrics, which are widely used to describe the degree of structure distortion and pixel blurring, 
    respectively.
\end{itemize} 

Our aim is to compare our method to some of the most commonly used approaches. The code for these approaches is as follows: (a) Diffusion and curvature models\footnote{\url{https://github.com/C4IR/FAIR.m}} \cite{fischer2002fast,fischer2003curvature,Fair}; (b) Hyper-elastic model\footnote{\url{http://www.siam. org/books/fa06/}} \cite{Hyperelastic2013}; (c) LDDMM\footnote{\url{https://github.com/C4IR/FAIR.m/tree/master/add-ons/LagLDDMM}} \cite{MangA}; (d) Diffeomorphic Log Demons image registration\footnote{\url{https://www.mathworks.com/matlabcentral/fileexchange/39194-diffeomorphic-log-demons-image-registration}} \cite{DiffLogDemons}. We intend to demonstrate the effectiveness of our method by comparing it with these popular approaches.

\begin{figure}
  \begin{center}
    \includegraphics[width=15.6cm]{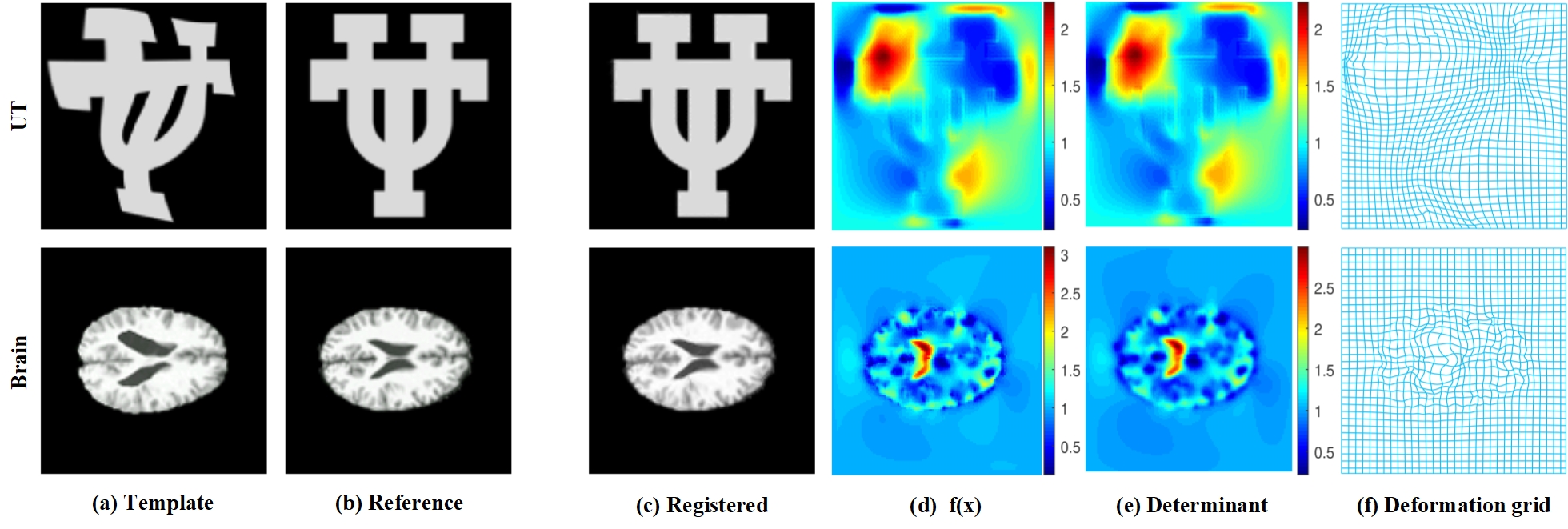}
  \end{center}
  \caption{Registration of 2D images with penalty function $\phi_1 $: (a) template images; (b) reference images; (c) registered results;  (d) the hotmap of $f(\bm{x})$; (e) the hotmap of the Jacobian determinant; (f) the deformation grids. Parameters of the two sets of images. UT ($\phi_1$): $\tau_1=1.20$, $\tau_2= 1e-3$, $\tau_3=5e-2$, $\lambda=1.2$, $\rho=1.2$, $\gamma=120$; UT ($\phi_2$): $\tau_1=1.40 $, $\tau_2= 1e-3$, $ \tau_3=5e-2 $, $\lambda=1.2$, $\rho=1.2$, $ \gamma=100$; UT ($ \phi_3 $): $\tau_1=1.40$, $\tau_2=1e-3$, $ \tau_3=5e-2$, $\lambda=1.2$, $\rho=1.2$, $\gamma=100 $;  Brain ($\phi_1$): $\tau_1=0.32$, $\tau_2= 1e-3$, $\tau_3=1e-2$, $\lambda=0.6$, $\rho=1.1$, $\gamma=100 $; Brain ($\phi_2$): $\tau_1=0.32$, $\tau_2=1e-3$, $\tau_3=1e-2$, $\lambda=0.6$, $\rho=1.1$, $\gamma=100 $; Brain ($\phi_3$): $\tau_1=0.32$, $\tau_2=1e-3$, $\tau_3=1e-2$, $\lambda=0.6$, $\rho=1.1$, $\gamma=100$.
  }\label{fig:phi1-phi2}
\end{figure}

\begin{table}[pos=H]
  \centering
  \newsavebox{\mybox}
  \begin{lrbox}{\mybox} 
    \begin{tabular}{ccccccccc}
      \toprule[1.5pt]
      \textbf{Examples} & $ \phi(\cdot) $  &$ R_{min} $& $\overline{\det}(J(\bar{\bm{\varphi}}))$& $\det_{\min}(J(\bar{\bm{\varphi}}))$ & $\det_{\max}(J(\bar{\bm{\varphi}}))$ & \emph{ssim}  & $\rm{Re_{-}SSD}$ &\emph{psnr} \\
      \hline
      \multirow{3}{*}{\textbf{UT}}  & $ \phi_1 $ & +& 1.000 &  0.23 & 2.23 & \textbf{0.9961} & \textbf{0.080\%}  & \textbf{35.49}\\

            &  $ \phi_2 $ &   + & 1.000&   0.09  &  2.10 &  0.9953 &  0.083\%  & 35.34\\

            &  {$ \phi_3 $} &  {+} & {1.000}&   {0.31}  &  {2.05} &  {0.9961} &  {0.084\%}  & {35.29} \\
            
      \hline                                       
      \multirow{3}{*}{\textbf{Brain}}   & $ \phi_1 $ &  + &1.000  &  0.22  &3.00 & \textbf{0.9868}     &  \textbf{2.67}\%  & \textbf{26.89}\\
    
           & $ \phi_2 $   & + & 1.000 &  0.30  &   2.78  & 0.9866 &  2.72\%  & 26.81\\

           & {$ \phi_3 $}   & {+} & {1.000} &  {0.26}  &   {2.80}  & {0.9866} &  {2.76\%}  & {26.76} \\
            \hline
      \toprule[1.5pt]
    \end{tabular}
  \end{lrbox}
  \caption{Quantitative comparisons between registration results of both different penalty functions $\phi(\cdot)$. The best value is highlighted by the \textbf{bold}.}
  \scalebox{0.92}{\usebox{\mybox}}\label{phi1-phi2} 
\end{table}

\subsection{Comparisons of different penalty functions}
In this section, we examine how the proposed registration model \eqref{eq:model1} is affected by different control functions $\phi_1$, $\phi_2$ and $\phi_3$. The control function $\phi_3$ is defined in \eqref{Hyper} and possesses the property stated in \eqref{eq:eqpenaltyf2}, which implies that $\phi(z)=\phi(1/z)$. This property allows flexible shrinkage and enlargement of deformations. To illustrate the differences between these control functions, we conducted experiments using $\phi_3$. In image registration problems involving large local deformations, enforcing a uniform Jacobian determinant constraint may not lead to an optimal transformation between the template and the reference. Fixing the value of $\det(\nabla\bar{\bm{\varphi}})=1$ at each pixel results in a deformation field $\bar{\bm{\varphi}}(\bm{x})$ that is point-wise volume-preserving but may not accurately match the local features of the reference and template images. Alternatively, by relaxing the constraint to $\det(\nabla\bar{\bm{\varphi}})=f(\bm{x})$, where $f(\bm{x})$ depends on the values of the function over the entire image range $\Omega$, the transformation $\bar{\bm{\varphi}}$ becomes non-locally volume-preserving. Therefore, the selection of an appropriate control function $\phi$ is crucial. It serves as a penalty measure to ensure that $f(\bm{x})$ is positive and significantly influences the quality of transformation grids and the performance of image registration.

In this experiment, images with a resolution of $128\times128$ from UT and Brain were utilized. To ensure a fair comparison, the parameter configurations were adjusted to optimize the registration results for different penalty functions $\phi(\cdot)$. The results of the three penalty functions can be seen in \Cref{phi1-phi2}. All three penalty functions produced diffeomorphic solutions with $R_{min}>0$ and exhibited minimal volume changes. However, there were noticeable differences in various quantitative indicators among models employing different control functions. The penalty function $\phi_1$ yielded superior values for $\rm{Re_{-}SSD}$, \emph{ssim}, and \emph{psnr}. As a result, the model based on the penalty function $ \phi_1 $ was selected as the default version for all subsequent experiments.

The visualizations of the registered results obtained by the proposed model with penalty function $\phi_1$, including the template and reference images, the transformed template images, the hotmaps of the relaxation function $f(\bm{x})$ and the Jacobian determinant $ \det(\nabla\bar{\bm{\varphi}})$, and the transformed grids, are evaluated in \cref{fig:phi1-phi2}. It can be observed that the proposed scheme with the penalty term $\phi_1$ produces a transformed template that is visually indistinguishable from the reference image. The hotmap distributions for the Jacobian determinant $\det(\nabla\bar{\bm{\varphi}})$ and relaxation function $f(\bm{x})$ are also similar. Additionally, the deformation grids are smooth. This demonstrates the robustness of the proposed method to challenging conditions such as heavy occlusion of illumination and large background clutters.

\subsection{Evaluation for the penalty term}
The purpose of introducing the penalty term $\int_{\Omega}\phi(f(\bm{x}))d\bm{x}$ in the relaxation function $f(\bm{x})$ is to prevent folding and unreasonable changes in volume. In this subsection, we analyze the effect of the penalty term by varying the parameter $\tau_2$ between zero and non-zero values, while keeping all other optimal parameters constant. To evaluate the influence of the term $\int_{\Omega}\phi(f(\bm{x}))d\bm{x}$ on registration, we use Circle-Square and Watermelon images with a resolution of $128\times128$, as well as Brain MRI images with a resolution of $256\times256$.

\begin{table}[pos=b]
  \centering
  \begin{lrbox}{\mybox} 
    \begin{tabular}{lcccccccc}
      \toprule[1.5pt]
      \textbf{Examples} & $ \tau_2 $ & \textbf{Correction} & \textbf{GFR} & $\det_{\min}(J(\bar{\bm{\varphi}}))$ & $\det_{\max}(J(\bar{\bm{\varphi}}))$ & \textbf{\emph{ssim}}  & $\rm{Re_{-}SSD}$ &\textbf{\emph{psnr}} \\
      \hline
      \multirow{4}{*}{\textbf{Circle-Square}}
            & 0     & no &   18.66\%   &  -4.98  & 21.56 &     0.9981      & 0.052\%           &  37.51\\
            & 0.01  & no &   3.17\%    &  -1.00  & 14.09 &     0.9981      & 0.042\%           & 38.40\\
      & 0     & yes &      0      &  0.03   & 5.21  &     0.9956      & 0.135\%           & 33.33\\
      & 0.01  & yes &      0      &  0.75   & 1.95  & \textbf{0.9984} &\textbf{ 0.036\%}  & \textbf{39.01}\\
      \hline                                       
      \multirow{4}{*}{\textbf{Watermelon}} 
            & 0     & no &   0.97\%    &  -0.10  &  3.19 &     0.9530      &     0.892\%       & 24.50\\ 
            & 0.001 & no &      0      &  0.34   & 2.17  & \textbf{0.9620} & \textbf{0.695\%}  & \textbf{25.58}\\ 
      &   0   & yes &      0      &  0.08   & 3.88  &     0.9571      &     0.780\%       & 25.07\\
      & 0.001 & yes &      0      &  0.34   & 2.17  & \textbf{0.9620} & \textbf{0.695\%}  & \textbf{25.58}\\                
      \hline
      \multirow{4}{*}{\textbf{BrainMR}}  
      &  0   & no &   2.44\%    &   -11.30  & 21.64 &    0.9550       &    7.904\%         & 16.99\\
      &0.001 & no &     0       &  0.07     &  8.67 &    \textbf{0.9651}     &      \textbf{ 3.699}\%       &  \textbf{20.25}\\ 
      &  0   & yes &     0       &   0.02    & 15.96 &    0.9601      &       3.998\%        & 19.90\\
      &0.001 & yes &     0       &  0.30     &  8.12 &     0.9648   &      3.785\%      &20.16 \\
      \toprule[1.5pt]
    \end{tabular}
  \end{lrbox}
  \caption{The relevant metrics are obtained from the three sets of experiments. ``yes" and ``no" represent the results of performing the deformaton correction and not performing deformaton correction, respectively. The best value is highlighted by \textbf{bold}.}
  \scalebox{0.84}{\usebox{\mybox}}\label{phi-revised3} 
\end{table}

\begin{figure}[ht]
  \begin{center}
    \includegraphics[width=15.6cm,height=7.8cm]{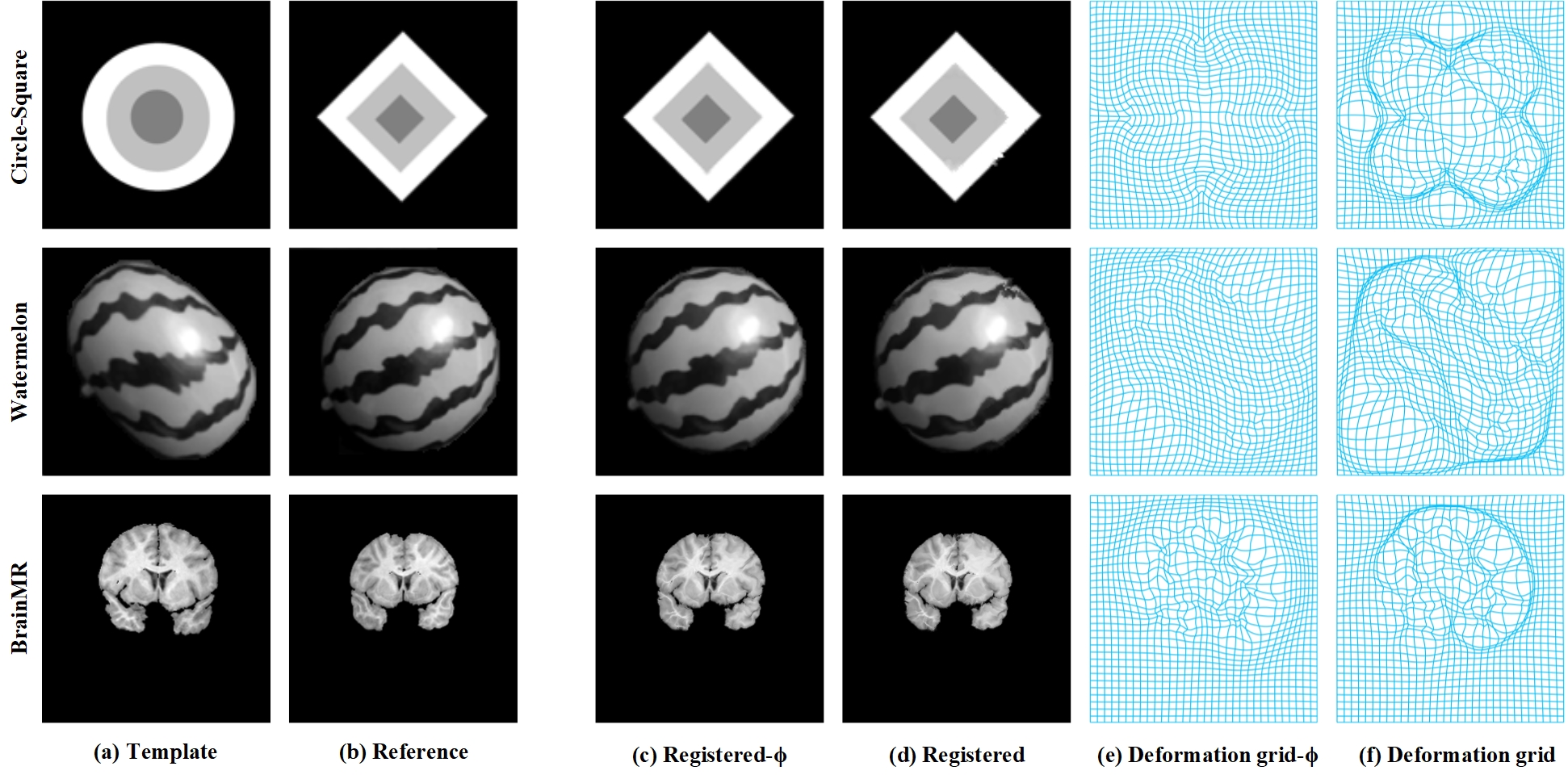}
  \end{center}
  \caption{Registration results with and without penalty term using the deformaton correction: (a) template images; (b) reference images; (c) registered images with penalty term;  (d) registered images without penalty term; (e)-(f) a visualization of the deformed grids with and without penalty term.
  Parameters of the three sets of images. Circle-Square: $\tau_1=0.3 $, $ \tau_2= 1e-2 $, $ \tau_3=1e-3 $, $ \lambda=0.8 $, $ \rho=1.08 $, $ \gamma=18 $; Watermelon: $ \tau_1=0.2 $, $ \tau_2= 1e-3 $, $ \tau_3=1e-3 $, $ \lambda=1.06 $,  $ \rho=1.06 $, $ \gamma=16 $; BrainMR: $ \tau_1=0.4 $, $ \tau_2= 5e-3 $, $ \tau_3=1e-3 $, $ \lambda=0.4 $, $ \rho=1.16 $, $ \gamma=20$.
  }\label{fig:phi-tau22}
\end{figure}

We investigated the impact of the penalty term $\int_{\Omega}\phi(f(\bm{x}))d\bm{x}$ on the registration results by conducting numerical experiments with $\tau_2=0$ and $\tau_2\neq0$, both with and without deformation correction. The results of the comparison of the grid folding ratio are shown in \Cref{phi-revised3}. It is evident that in all three examples without deformation correction, local grid folding cases ($\det_{\min}(J(\bar{\bm{\varphi}}))<0$ or $\text{GFR}>0$) occur when the hyper-parameter $\tau_2=0$ of the function $\phi(\cdot)$ is fixed. However, the grid folding ratio can be reduced or eliminated when $\tau_2\neq0 $, thus improving the accuracy of the registration. Nevertheless, grid folding still occurs when using the penalty term in the Circle-Square example. As mentioned in Remark \ref{remk1}, inappropriate choices of hyper-parameters and inadequate numerical solution techniques for the nonlinear equation in practice may result in $f(\bm{x})<0$, leading to grid folding. Additionally, the grid correction in \Cref{alg:DC} can be employed to generate a desirable diffeomorphic deformation.

The registration results for three examples are demonstrated, both with and without the penalty term $\int_{\Omega}\phi(f(\bm{x}))d\bm{x}$, which incorporates the correction for deformation. As shown in \Cref{fig:phi-tau22}, the proposed model, when equipped with the penalty term, not only produces better registered images, but also generates smoother and more realistic transformations. This can be attributed to two factors. Firstly, the penalty term constrains the range of $f(\bm{x})$ to be in proximity to 1, consequently determining the range of $\det\big(\nabla\bar{\bm{\varphi}}(\bm{x})\big)$. Secondly, the regularization term $\int_{\Omega} \|\nabla f(\bm{x})\|^2 d\bm{x}$ in this study ensures the smoothness of $f(\bm{x})$, which is also associated with the smoothness of the transformation $\bar{\bm{\varphi}}(\bm{x})$. These findings validate the significance of the penalty term that incorporates the relaxation function $f(\bm{x})$ in the proposed model, as it leads to smoother deformations that tend to preserve volume.

\subsection{Evaluations for algorithm convergence}
To demonstrate the convergence and volume-preserving properties of our proposed algorithm, we conducted a comparison with diffusion-based and curvature-based registrations \cite{fischer2002fast,fischer2003curvature,Fair}. The comparison was performed using an IC image of size $200\times 200$. To ensure fairness, all three methods were iterated for 100 steps at each level.

The registration results of our method, as well as the diffusion-based and curvature-based methods, are shown in \Cref{fig:testIC}. It is evident that our method produces a registered image that closely resembles the reference image, while the diffusion and curvature models fail to achieve satisfactory registration for the IC example (see the zoom-in regions in \Cref{fig:testIC}(b)). The transformation grids of the diffusion and curvature models are severely distorted. The findings in \Cref{tableIC} support the validity of the global volume-preserving assumption ($\int_\Omega\det(J(\bar{\bm{\varphi}}))d\bm{x}/|\Omega|=1$), as even objects in the template that are smaller than those in the reference have been successfully registered, enabling their analysis and inference. Moreover, the numerical values of $\det(\nabla\bar{\bm{\varphi}}) \in[-0.3, 1.7] $ and $\det(\nabla\bar{\bm{\varphi}}) \in[-0.14, 1.07] $ demonstrate that the mappings obtained by the diffusion and curvature models are non-diffeomorphic.

\begin{figure}
  \begin{center}
    \includegraphics[width=15cm,height=8cm]{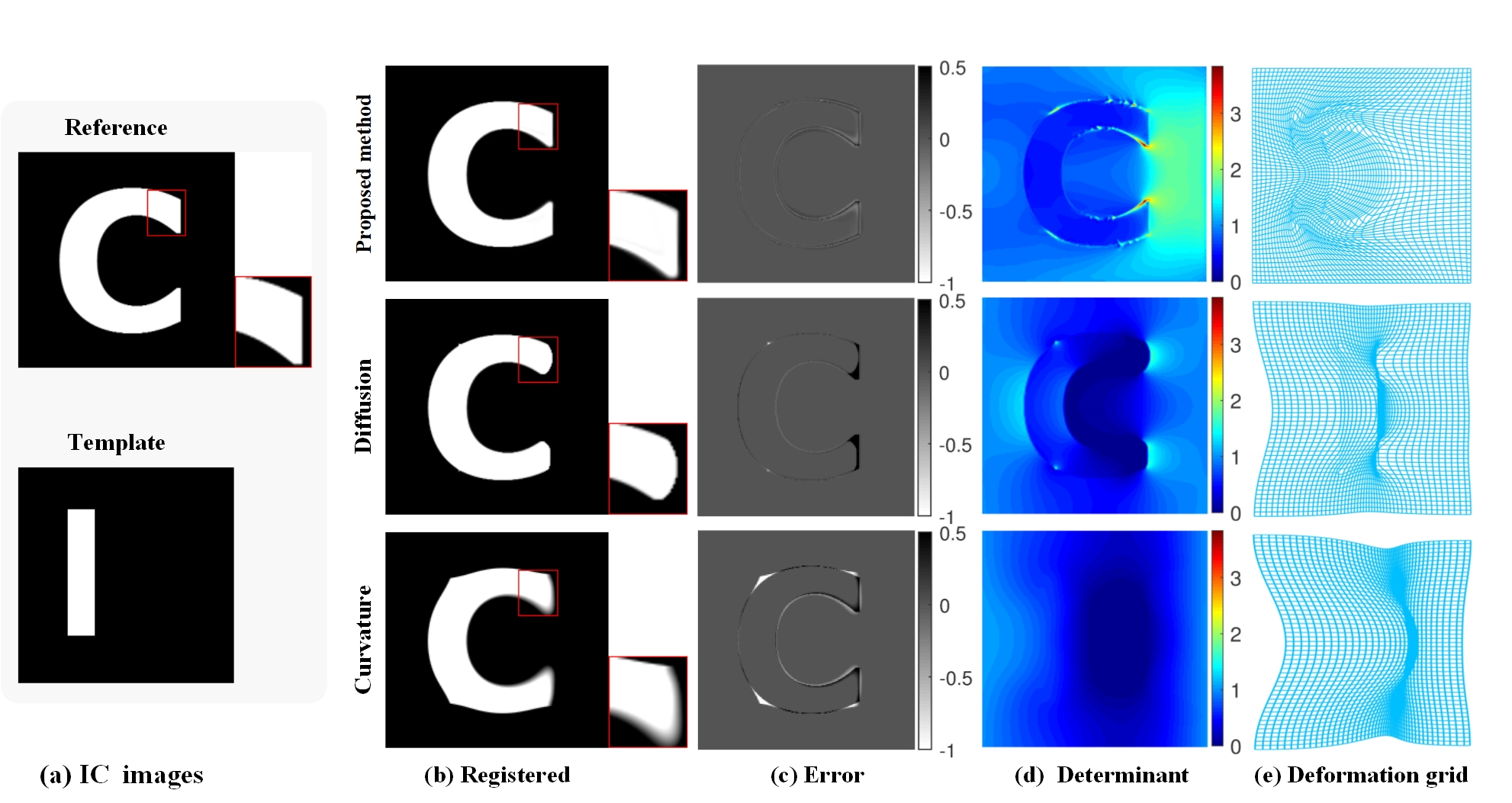}
  \end{center}
  \caption{Comparisons of the proposed, diffusion, and curvature models. (a) the reference and template images; (b) the deformed template images of the three models with the optimal parameters (IC: $\tau_1=3 $, $\tau_2=1e-2 $, $\tau_3 =1e-3$, $\lambda=1$, $ \gamma=100 $, $ \rho=1.06 $ for the proposed model; $\alpha=8800$ for diffusion model; $\alpha=200$ for curvature model); (c) the registration errors of ~$T(\bar{\bm{\varphi}})-R(\bm{x})$; (d) the Jacobian determinant hotmaps of the deformation fields; (e) the deformation grids.}\label{fig:testIC}
\end{figure}

\begin{table}[pos=H]
  \centering
  \begin{lrbox}{\mybox} 
    \begin{tabular}{ccccccccc}
      \toprule[1.5pt]
      \textbf{Examples} & \textbf{Methods}  & $\overline{\det}(J(\bar{\bm{\varphi}}))$&$ R_{min} $& $\det_{\min}(J(\bar{\bm{\varphi}}))$ & $\det_{\max}(J(\bar{\bm{\varphi}}))$ & \emph{ssim}  & $\rm{Re_{-}SSD}$ &\emph{psnr}\\
      \hline
      \multirow{3}{*}{\textbf{IC}}  & Proposed    &   0.999 & + &  0.29  & 3.84  & \textbf{0.9828}  & \textbf{0.23}\%  & \textbf{27.85}\\
      &   Diffusion  & 0.677    & $-$ &  \underline{-0.30}  &  1.70  &  0.9775    & 3.51\%   & 15.80\\
      &   Curvature   & 0.504    & $-$ &  \underline{-0.14}  &  1.07  &  0.9553    & 3.67\%   & 15.73\\
      \toprule[1.5pt]
    \end{tabular}
  \end{lrbox}
  \caption{The quantitative evaluation comparisons of the proposed, diffusion, and curvature models. The negative Jacobian determinant and best metrics values are highlighted by \underline{\rm{underline}} and \textbf{bold}.}
  \scalebox{0.86}{\usebox{\mybox}}\label{tableIC} 
\end{table}

In \Cref{fig:convergence}, we present the change curves of three different metrics: the similarity measure $\mathcal{D}(\bm{u})$, the relative objective function $\hat{\mathcal{L}}^{\ell,k}$, and the average Jacobian determinant $\overline{\det}(J(\bar{\bm{\varphi}}))$. The relative objective function is defined as the ratio of the objective function value $\mathcal{L}^{\ell,k}$ at the $k$-th iteration and the initial objective function value $\mathcal{L}^{L,1}$ at the coarsest $L$-th level. From \Cref{fig:convergence}(a)-(b), it can be observed that the similarity measure and the relative objective function value of our proposed algorithm consistently decrease at each level and eventually stabilize, indicating convergence. Although the diffusion and curvature methods converge faster, their similarity measures and relative objective function values are inferior to ours, suggesting that our method is more satisfactory. Furthermore, \Cref{fig:convergence}(c) demonstrates that our proposed algorithm ensures that $\overline{\det}(J(\bar{\bm{\varphi}}))$ remains close to 1, which aligns with the conclusion of volume preservation in the average sense.

\begin{figure}[pos=H]
    \centering
    \subfigure[Similarity measure]{
  \includegraphics[width=0.32\linewidth]{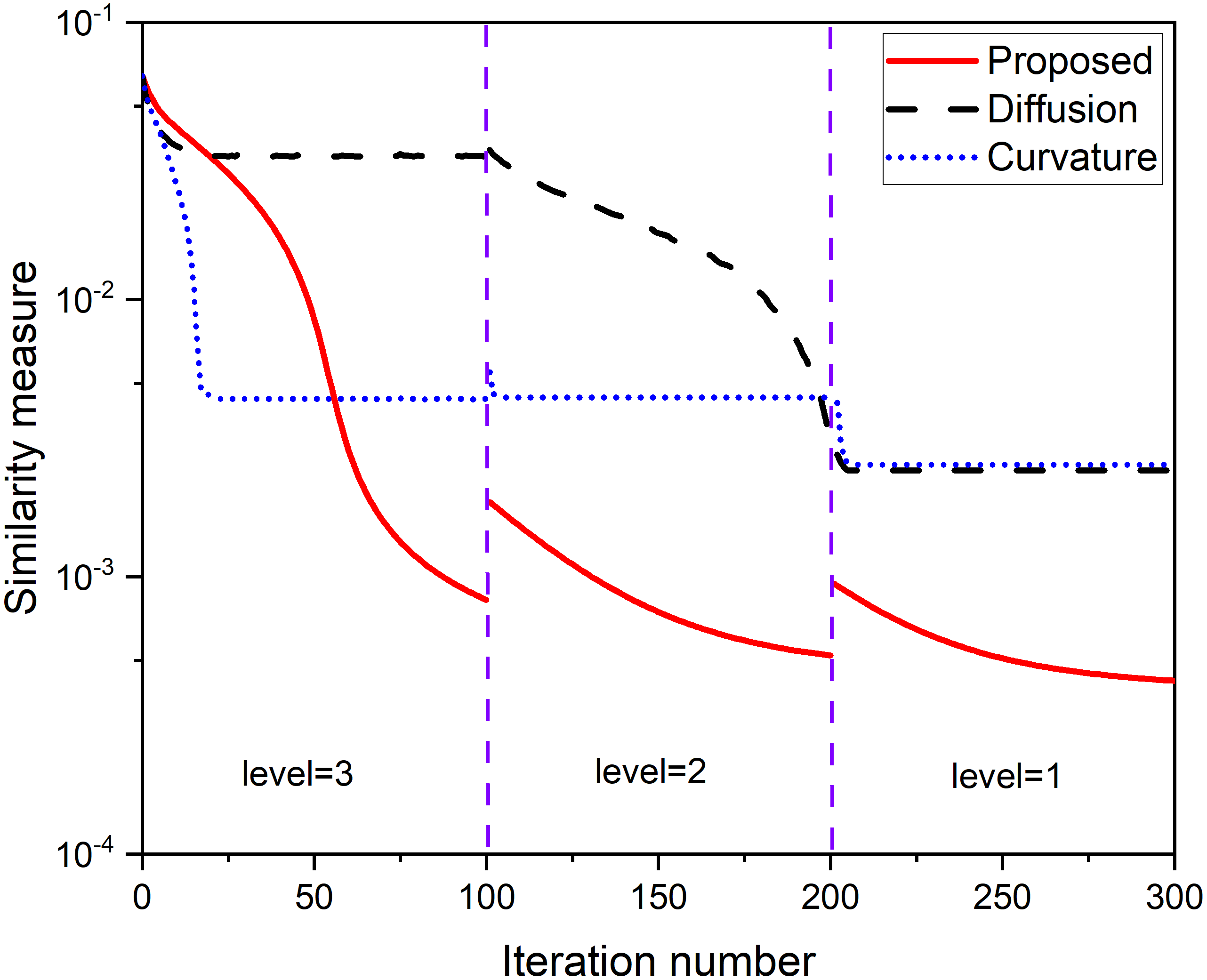}}
     \subfigure[Relative objective function]{
  \includegraphics[width=0.32\linewidth]{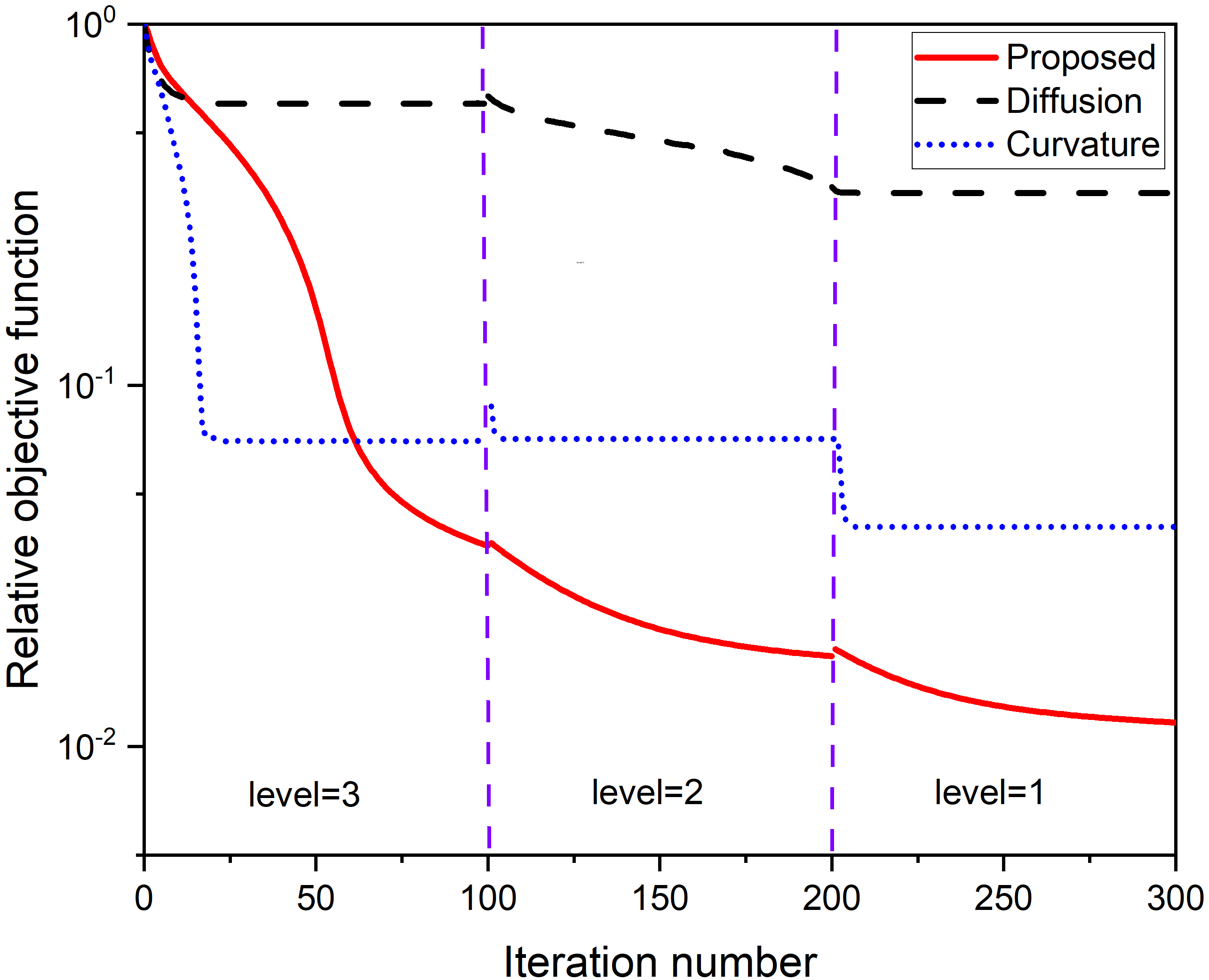}}
    \subfigure[Average Jacobian determinant]{
  \includegraphics[width=0.32\linewidth]{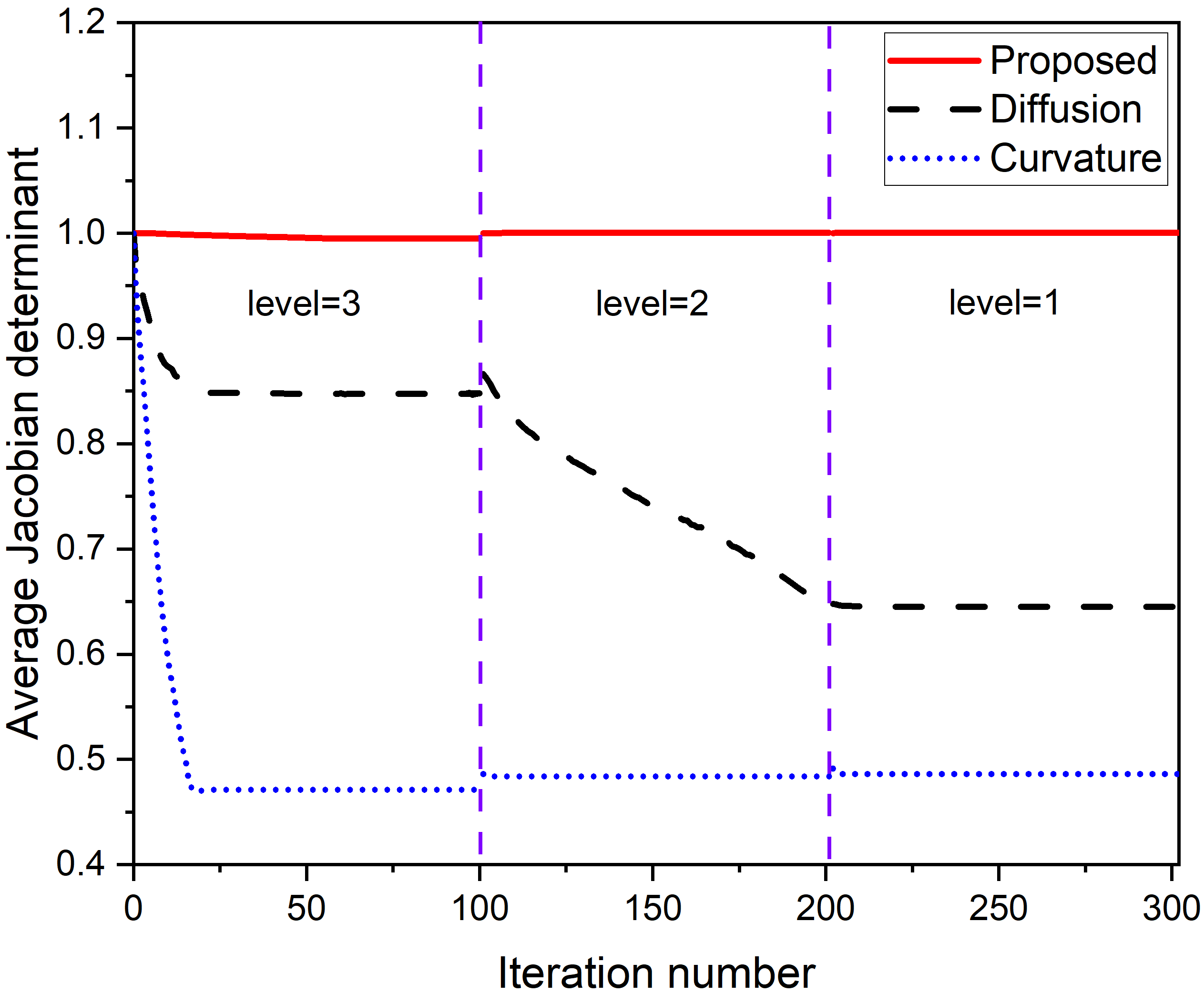}}
  \caption{Convergence and volume-preserving comparisons of the proposed, diffusion, and curvature models for the IC example. (a) similarity measure $ \mathcal{D}(\bm{u}) $ (the y-coordinate is logarithmic); (b) relative objective function $ \hat{\mathcal{L}}^{\ell,k} $  (the y-coordinate is logarithmic); (c) average Jacobian determinant $\overline{\det}(J(\bar{\bm{\varphi}}))$. The red solid, black dash, and blue dot lines indicate the results of the proposed, diffusion, and curvature models, respectively.}
  \label{fig:convergence}
\end{figure}

\subsection{Evaluations of parameter selections.}
The purpose of this section is to demonstrate the sensitivity of the parameters $\tau_1, \tau_2, \tau_3, \lambda$, and $\gamma$ in our model and how they affect the convergence of the algorithm. To obtain a more reliable and general model, it is necessary to test and adjust the parameters accordingly. It is important to note that the optimal parameters may vary for different images. In this example, we used a synthetic Circle-Square image with a resolution of $128 \times 128$ (refer to Figure \ref{fig:phi-tau22}) and employed the grid search method to explain how to select the appropriate parameters. We tested the impact of different parameters ($\tau_1, \tau_2, \tau_3, \lambda$, and $\gamma$) on the registration results, with a growth factor $\rho$ set to 1.08. For comparison purposes, we used a multilevel strategy with $L=3$ and performed 100 iterations at each level. The parameter tuning process is as follows.

The impact of different parameter changes on the registration results is shown in \Cref{tab:tau1}. To find the optimal regularization parameter, $\tau_1$, values ranging from 0.1 to 10 are tested, while keeping $\tau_2=1e-2$, $\tau_3=1e-3$, $\lambda=0.8$, and $\gamma=18$ fixed. From the results in \Cref{tab:tau1}, it can be seen that satisfactory solutions of the proposed model are achieved when $\tau_1$ is between 0.3 and 2. The best values for $\rm{Re_{-}SSD}$ and \emph{psnr} are obtained at $\tau_1=0.3$, so we recommend selecting this value as the regularization parameter for this example. Similarly, the same method is used to search for suitable parameters $\tau_2$, $\tau_3$, $\lambda$, and $\gamma$. For the Circle-Square images, $\tau_2$ is varied from $1e-4$ to $1e-2$, $\tau_3$ is varied from $1e-4$ to $1e-1$, $\lambda$ is varied from 0.05 to 1, and $\gamma$ is varied from 10 to 50. Consequently, a set of optimal parameters $\tau_1=0.3$, $\tau_2=1e-3$, $\tau_3=5e-3$, $\lambda=0.5$, and $\gamma=18$ can be obtained for the Circles-Square images. This approach can be applied similarly to other images.

\begin{table}[H]
  \centering
  \setlength{\tabcolsep}{15.0pt} 
  \renewcommand\arraystretch{1.2} 
  \begin{lrbox}{\mybox} 
    \begin{tabular}{cccccc}
      \toprule[1.5pt]
      Parameters& Variation & $ R_{min} $ & $\rm{Re_{-}SSD}$ &\emph{psnr} & \emph{ssim} \\
      \hline  
      \multirow{8}{*}{ \tabincell{c}{$\tau_2=\text{1e-2}$\\$\tau_3=\text{1e-3}$\\$\lambda=0.8$\\$\gamma=18 $\\$ \tau_1 $} } 
        &0.1    &+  &6.0159\%   &16.9327    &0.9660 \\
        &\textbf{0.3}   &+  &\textbf{0.0364\%}  &\textbf{39.0132}   &0.9981 \\
        &0.5    &+  &0.0382\%   &38.8078    &\textbf{0.9982} \\
        &0.7    &+  &0.0434\%   &38.2567    &0.9980 \\
        &1  &+  &0.0518\%   &37.4848    &0.9976 \\
        &2  &+  &0.0854\%   &35.3156    &0.9956 \\
        &4  &+  &0.2067\%   &31.4842    &0.9890 \\
        &10 &+  &0.8112\%   &25.5227    &0.9610 \\
        \hline 
        \multirow{7}{*}{ \tabincell{c}{$\tau_1=0.3$\\$\tau_3=\text{1e-3}$\\$\lambda=0.8$\\$\gamma=18 $\\$ \tau_2 $} }
        &1e-4   &+  &0.0966\%   &34.7929    &0.9965  \\
        &5e-4   &+  &0.0294\%   &39.9573    &0.9988  \\
        &\textbf{1e-3}  &+  &\textbf{0.0266\%}  &\textbf{40.3973}   &\textbf{0.9989}  \\
        &5e-3   &+  &0.0266\%   &40.3816    &0.9988  \\
        &1e-2   &+  &0.0364\%   &39.0132    &0.9981  \\
        &5e-2   &+  &0.4501\%   &28.0562    &0.9833  \\
        &1e-1   &+  &3.9650\%   &18.5983    &0.9366  \\
        \hline 
        \multirow{8}{*}{ \tabincell{c}{$\tau_1=0.3$\\$\tau_2=\text{1e-3}$\\$\lambda=0.8$\\$\gamma=18 $\\$ \tau_3 $} }
        &1e-4   &+  &0.0414\%   &38.4673    &0.9986  \\
        &5e-4   &+  &0.0266\%   &40.3918    &0.9989  \\
        &1e-3   &+  &0.0266\%   &40.3973    &0.9989  \\
        &\textbf{5e-3}  &+  &\textbf{0.0239\%}  &\textbf{40.8621}   &\textbf{0.9991}  \\
        &1e-2   &+  &0.0266\%   &40.3898    &0.9990  \\
        &5e-2   &+  &0.0320\%   &39.5926    &0.9988  \\
        &1e-1   &+  &0.0349\%   &39.2089    &0.9987  \\
        &5e-1   &+  &0.6129\%   &26.7790    &0.9880  \\
        \hline 
        \multirow{8}{*}{ \tabincell{c}{$\tau_1=0.3$\\$\tau_2=\text{1e-3}$\\$\tau_3=\text{5e-3}$\\$\gamma=18 $\\$ \lambda $} }
        &0.01   &+  &0.1720\%   &32.2823    &0.9978  \\
        &0.05   &+  &0.0261\%   &40.4741    &0.9990  \\
        &0.1        &+  &0.0394\%   &38.6851    &0.9989  \\
        &\textbf{0.5}       &+  &\textbf{0.0225\%}  &\textbf{41.1141}   &\textbf{0.9991}  \\
        &0.8        &+  &0.0239\%   &40.8621    &0.9991  \\
        &1      &+  &0.0785\%   &35.6879    &0.9987  \\
        &5      &+  &0.7041\%   &26.1720    &0.9916  \\
        &10     &+  &0.4738\%   &27.8838    &0.9946  \\
        \hline 
        \multirow{8}{*}{ \tabincell{c}{$\tau_1=0.3$\\$\tau_2=\text{1e-3}$\\$\tau_3=\text{5e-3}$\\$\lambda=0.5 $\\$ \gamma $} }
        &1  &+  &2.8639\%   &20.1179    &0.9762 \\
        &10 &+  &0.1778\%   &32.1388    &0.9977 \\
        &15 &+  &0.2436\%   &30.7727    &0.9973 \\
        &\textbf{18}    &+  &\textbf{0.0225\%}  &\textbf{41.1141}   &\textbf{0.9991} \\
        &20 &+  &0.0232\%   &40.9837    &0.9991 \\
        &50 &+  &0.0595\%   &36.8879    &0.9987 \\
        &80 &+  &0.9717\%   &24.7479    &0.9892 \\
        &100    &+  &1.5259\%   &22.7695    &0.9841 \\
      \bottomrule[1.5pt]
    \end{tabular}
  \end{lrbox}
  \caption{The effect of the different parameters on the registration results. The best metrics values are highlighted by the \textbf{bold}.}\label{tab:tau1} 
  \scalebox{0.92}{\usebox{\mybox}}
\end{table}

To demonstrate the impact of parameters on the convergence of the model, we present the variation curve of $\rm{Re_{-}SSD}$ with iteration steps for different combinations of parameters at the finest level in \Cref{fig:parameters-ressd}(a)-(e). It is clear that changes in these parameters have a significant effect on both the convergence and the registration accuracy of the model. Specifically, \Cref{fig:parameters-ressd}(a) illustrates that altering the value of the parameter $ \tau_1 $, while keeping other parameters constant, has an impact on both $\rm{Re_{-}SSD}$ and the convergence of the model. This observation is further supported by \Cref{fig:parameters-ressd}(b)-(e), which indicates that selecting appropriate parameter values is crucial to ensuring the convergence and registration accuracy of the model. Therefore, by carefully selecting parameters, it is possible to optimize the performance of the model, and achieve better convergence and accuracy.

\begin{figure}
\centering
    \subfigure[$ \tau_1 $]{\label{fig:tau1-ressd}
        \begin{minipage}[t]{0.33\linewidth}
            \centering
            \includegraphics[width=5cm]{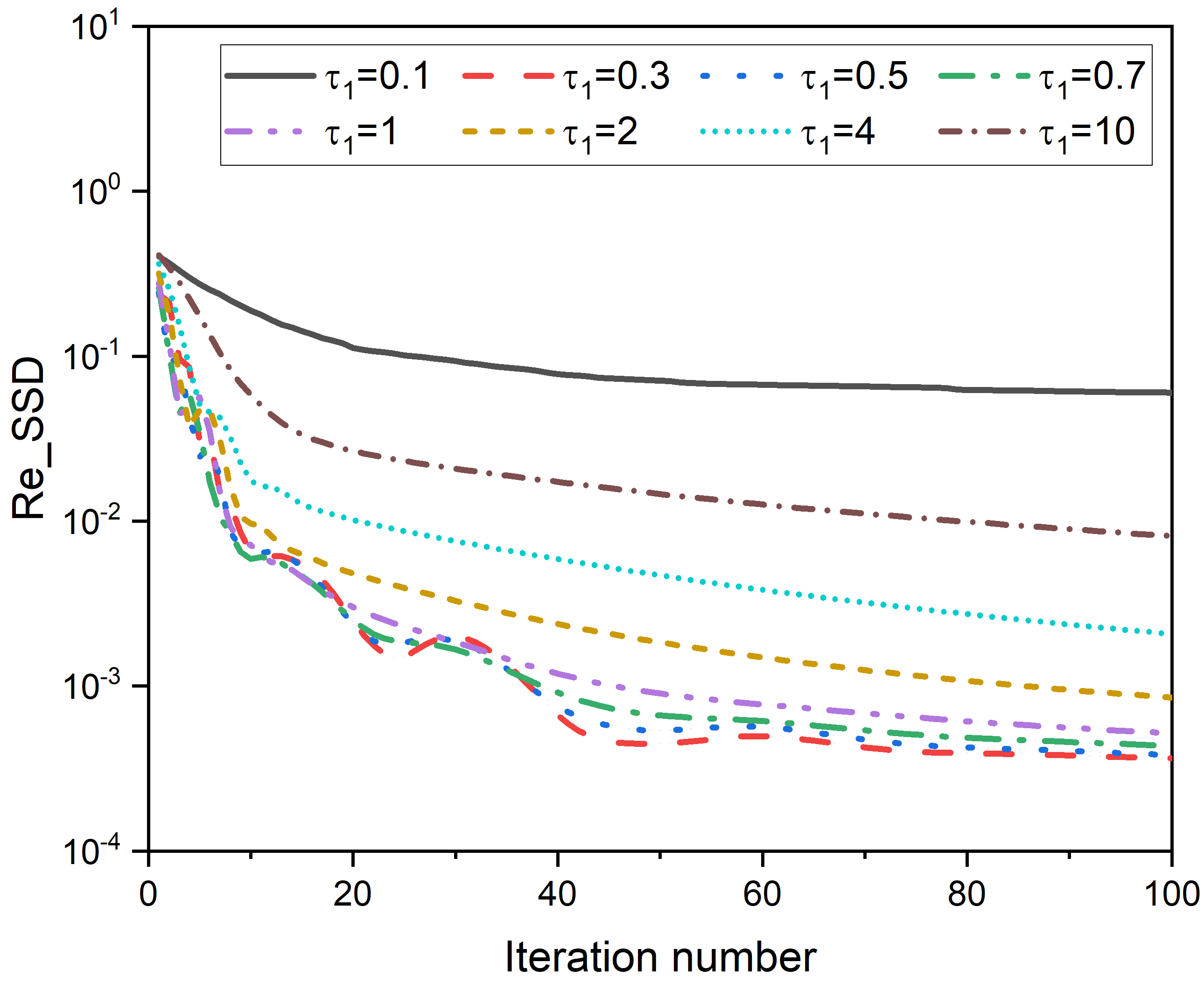}
        \end{minipage}
    }%
    \subfigure[$ \tau_2 $]{\label{fig:tau2-ressd}
        \begin{minipage}[t]{0.33\linewidth}
            \centering
            \includegraphics[width=5cm]{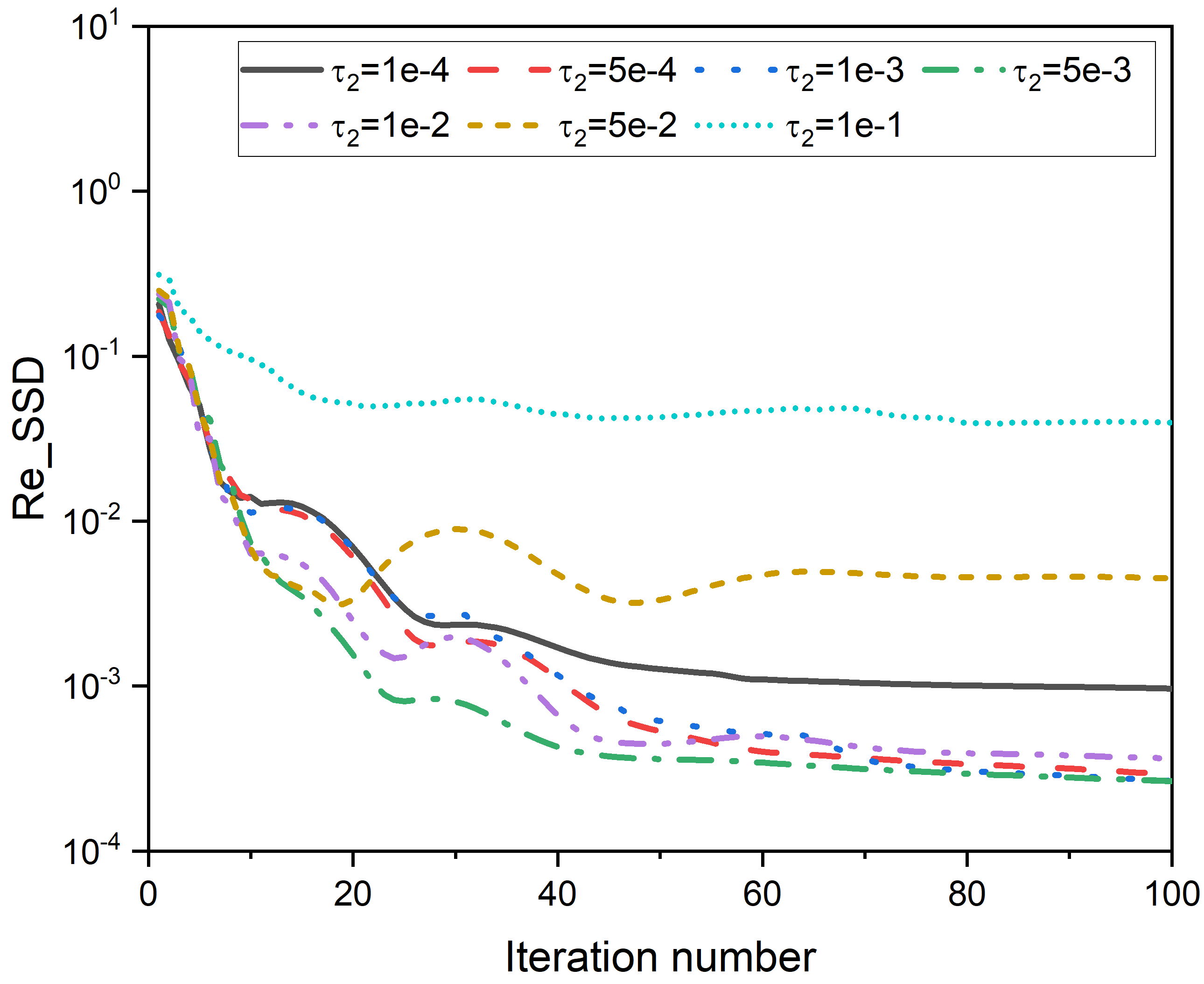}
        \end{minipage}
    }%
    \subfigure[$ \tau_3 $]{\label{fig:tau3-ressd}
        \begin{minipage}[t]{0.33\linewidth}
            \centering
            \includegraphics[width=5cm]{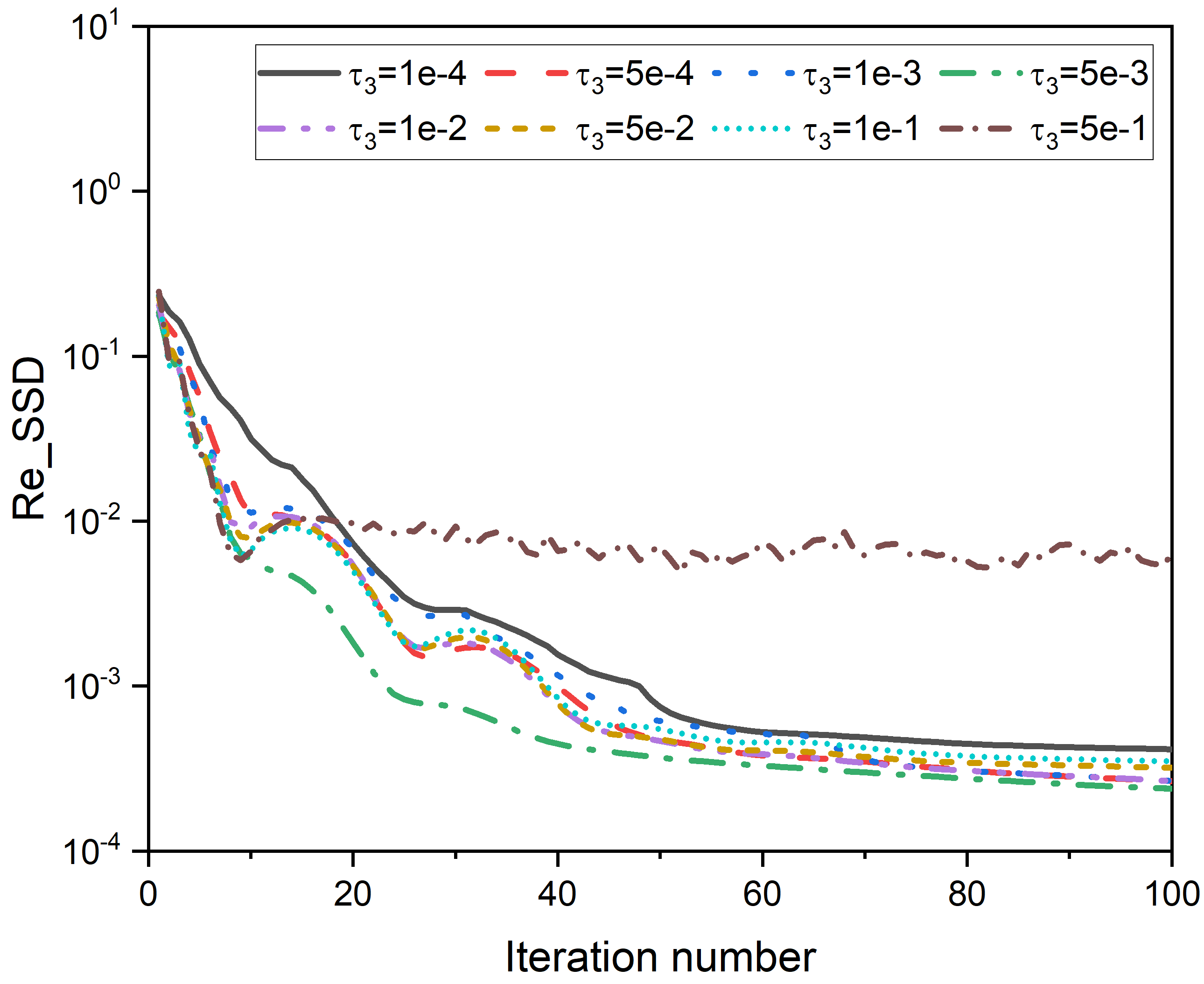}
        \end{minipage}
    }%

    \subfigure[$ \lambda $]{\label{fig:lambda-ressd}
        \begin{minipage}[t]{0.33\linewidth}
            \centering
            \includegraphics[width=5cm]{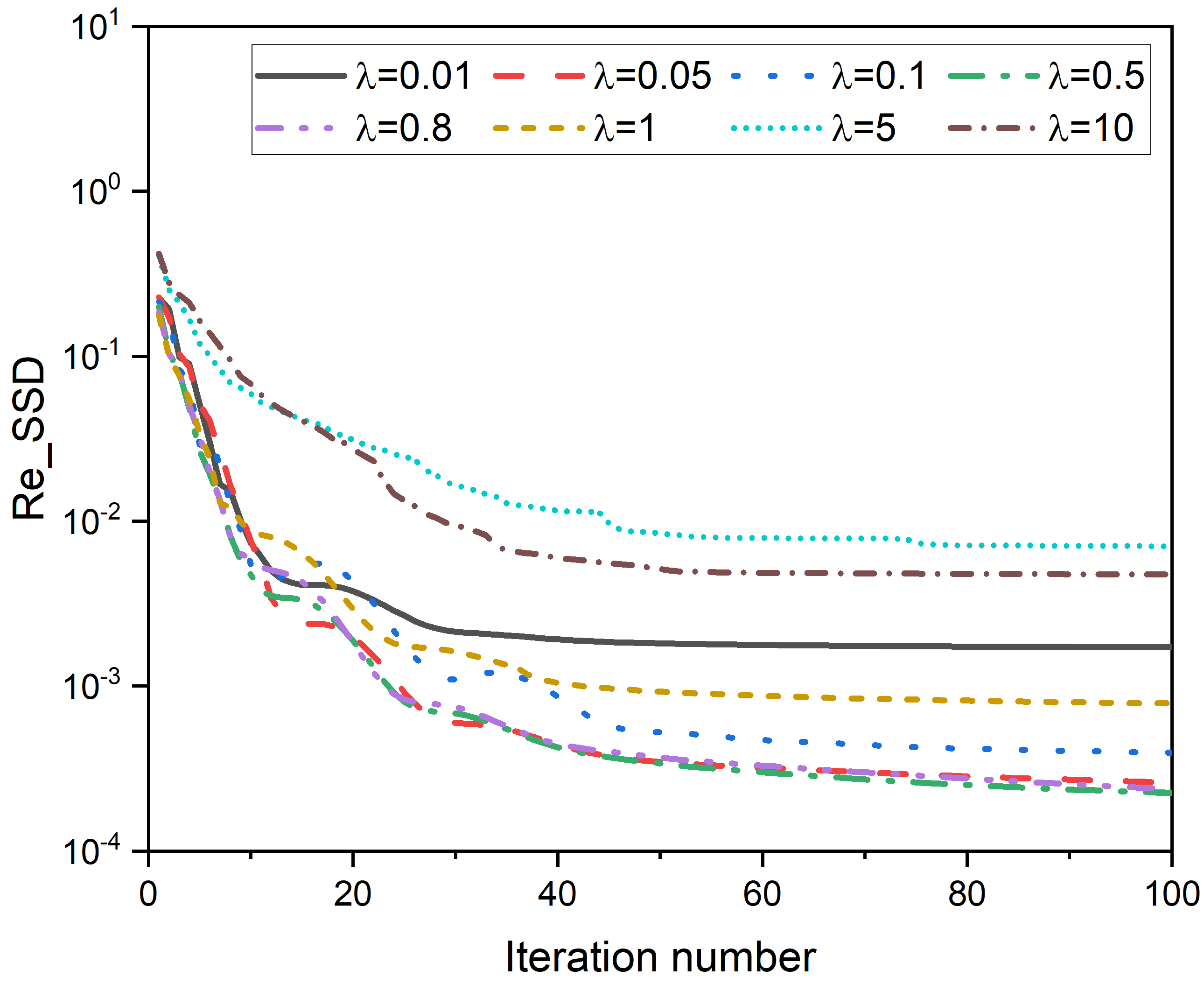}
        \end{minipage}
    }%
    \subfigure[$ \gamma $]{\label{fig:gamma-ressd}
        \begin{minipage}[t]{0.33\linewidth}
            \centering
            \includegraphics[width=5cm]{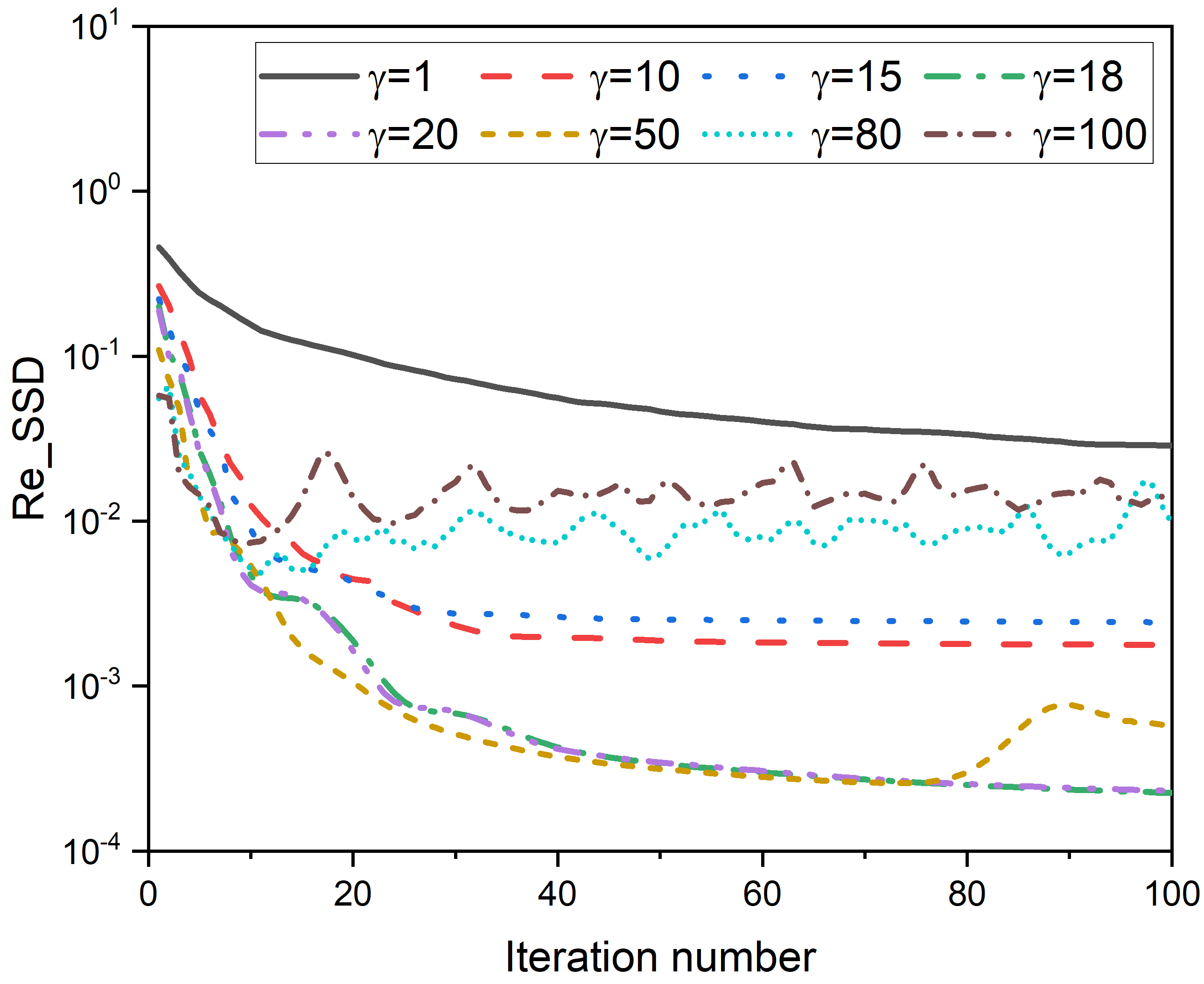}
        \end{minipage}
    }%
    \caption{The impact of univariate parameter variations on registration accuracy and convergence, where the y-coordinate is logarithmic.}
    \label{fig:parameters-ressd}
\end{figure}

\subsection{Comparisons with diffusion- and curvature-based methods}
In this part, we assess the accuracy of the proposed diffeomorphic registration technique using Lena and Hand images of size $ 128\times 128 $ as two sets of examples. We compare our results quantitatively and qualitatively with the diffusion and curvature registration methods described in \cite{fischer2002fast,fischer2003curvature, Fair}. Here, we employ a multilevel strategy (${L}=3$) and set $\text{MaxIter}= 200$ for all models.

\begin{figure}
  \begin{center}
    \includegraphics[width=15cm,height=8cm]{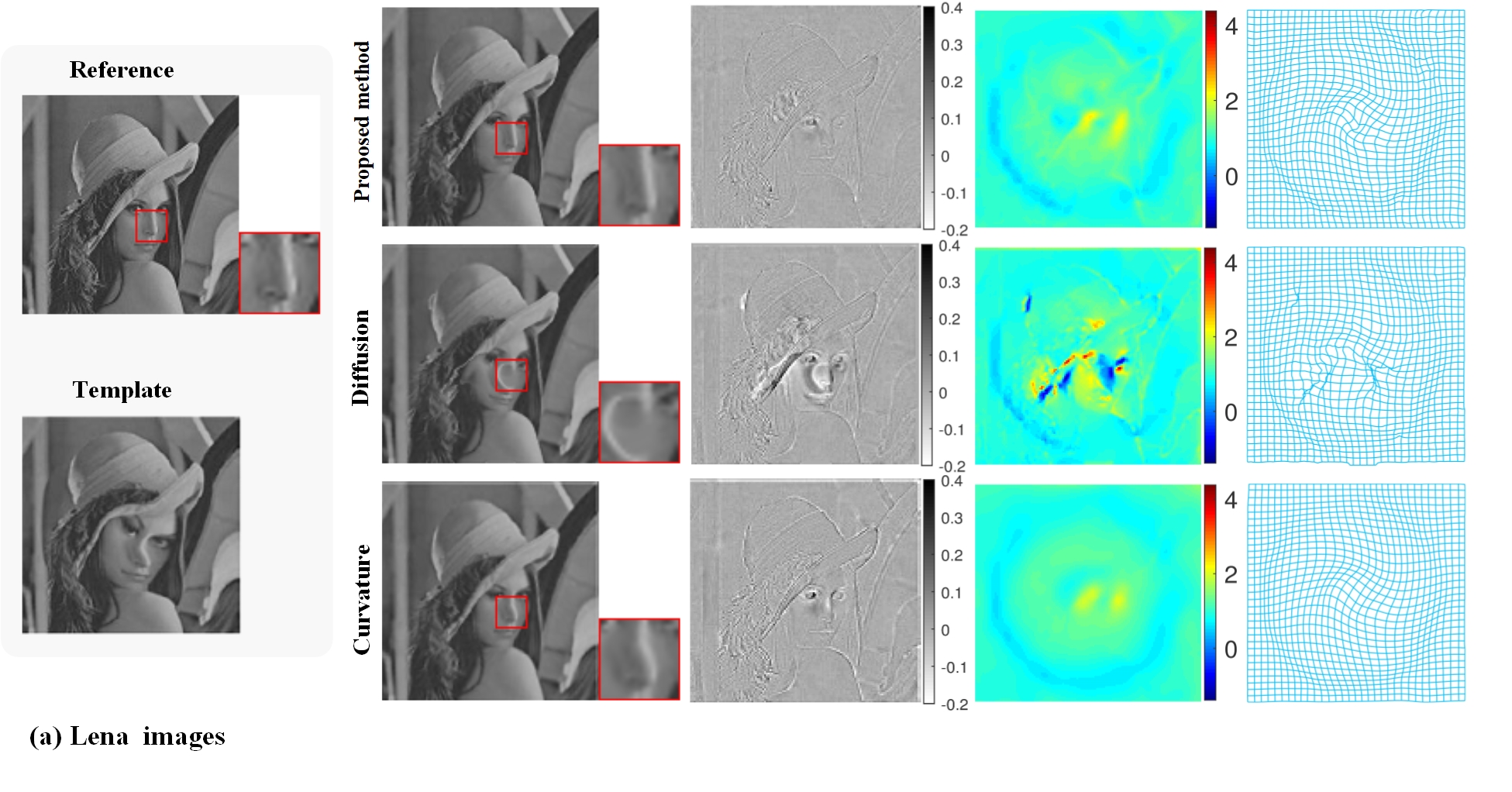}    
    \includegraphics[width=15cm,height=8cm]{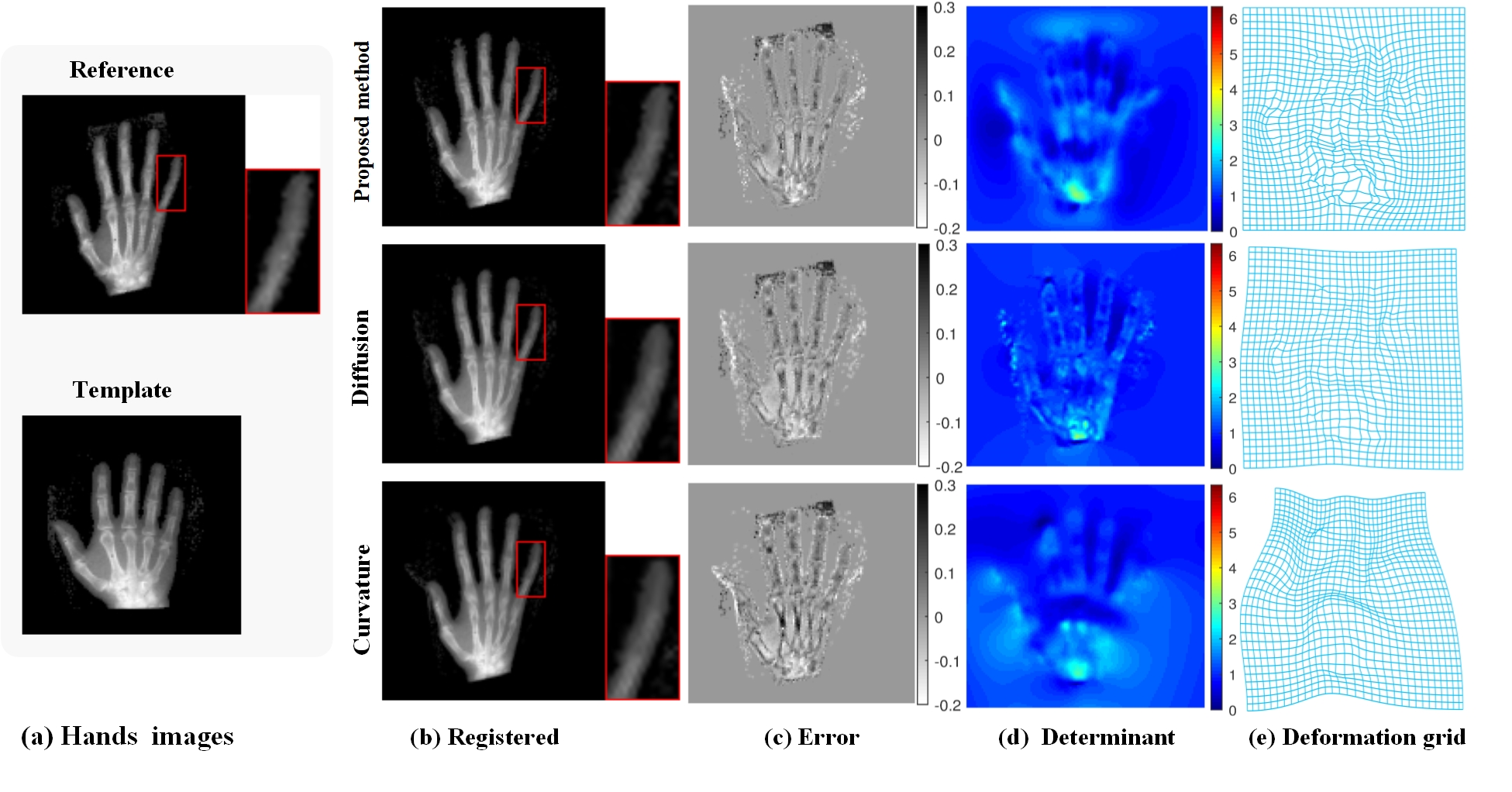}     
  \end{center}
  \caption{Comparisons of the proposed, diffusion, and curvature models. (a) the reference and template images (Lena and Hands); (b) the deformed template images of the three models with the optimal parameters (Hands: $\tau_1=0.4 $, $\tau_2=1e-2 $, $\tau_3 =1e-3$, $ \lambda=0.2 $, $ \gamma=400 $, $ \rho=1.02$ for the proposed model; $\alpha=500$ for diffusion model; $\alpha=0.2$ for the curvature model. Lena: $\tau_1=0.6 $, $\tau_2=1e-2 $, $\tau_3 =1e-3$, $ \lambda=0.2 $, $ \gamma=200 $, $ \rho=1.02 $ for the proposed model; $\alpha=400$ for the diffusion model; $\alpha=1$ for curvature model); (c) the registration errors of ~$T(\bar{\bm{\varphi}})-R(\bm{x})$; (d) the Jacobian determinant hotmaps of the deformation fields; (e) the deformation grids.}\label{fig:testhandsLena}
\end{figure}

In \Cref{fig:testhandsLena}, we present a comparison of the visualizations of the registered results obtained using our proposed model and the diffusion and curvature registrations. The visualizations include template and reference images with zoom-in regions, transformed template images, registration errors of ~$T(\bar{\bm{\varphi}})-R(\bm{x})$, the hotmaps of the Jacobian determinant $\det(\nabla\bar{\bm{\varphi}})$, and transformations $\bar{\bm{\varphi}}$. Our model successfully produces desirable registration results while preserving the diffeomorphism of the transformations. However, it is evident from the zoom-in regions in \Cref{fig:testhandsLena}(b) that the diffusion and curvature models fail to achieve good registration for the Lena example. Furthermore, by observing the deformation of the Lena and Hands examples, it can be noted that the deformation generated by the curvature model is smoother compared to ours. This discrepancy arises because the curvature model utilizes second-order derivatives for the regularization of the displacement field, whereas our regularizer considers only the first-order derivatives.

\begin{table}
  \centering
  \begin{lrbox}{\mybox} 
    \begin{tabular}{ccccccccc}
      \toprule[1.5pt]
      \textbf{Examples} & \textbf{Methods}  & $\overline{\det}(J(\bar{\bm{\varphi}}))$&$ R_{min} $& $\det_{\min}(J(\bar{\bm{\varphi}}))$ & $\det_{\max}(J(\bar{\bm{\varphi}}))$ & \emph{ssim}  & $\rm{Re_{-}SSD}$ &\emph{psnr}\\
      \hline
      \multirow{3}{*}{\textbf{Lena}}        & Proposed    &   0.999  & +&  0.46  &  3.09  & \textbf{0.9193}   &  \textbf{3.31}\%  & \textbf{25.09}\\
      
      & Diffusion    &   1.013   & $-$ &  \underline{-1.39} &  4.38   & 0.8614  & 8.05\%  & 21.23\\
      & Curvature     &   1.012   & + &  0.56  &   2.02   &  0.8520   & 7.86\%  & 21.35\\
      \hline
      \multirow{3}{*}{\textbf{Hands}}  & Proposed    &   1.000  & +&  0.32  &  6.33   & \textbf{0.8943}   &  \textbf{3.04}\%  & \textbf{18.46}\\
      & Diffusion    &   1.008   & $-$ &  \underline{-0.09}  &  3.34   & 0.8845   & 4.20\%  & 17.14\\
      & Curvature     &   1.011   & $-$ &  \underline{-0.05}  &   2.81   &  0.8829  & 5.00\%  & 16.24\\
      \toprule[1.5pt]
    \end{tabular}
  \end{lrbox}
  \caption{The quantitative evaluation comparisons of the proposed, diffusion, and curvature models. The negative Jacobian determinant and best metrics values are highlighted by \underline{\rm{underline}} and \textbf{bold}.}
  \scalebox{0.86}{\usebox{\mybox}}\label{tablenondiffeomophic} 
\end{table}

The violation of the diffeomorphism principle $\det(\nabla\bar{\bm{\varphi}})>0$ is evident when diffusion- and curvature-based methods are applied to manual and natural images, as shown in \Cref{tablenondiffeomophic}. For the Lena example, the range of $\det(\nabla\bar{\bm{\varphi}})$ is $[-1.39, 4.38]$, and for the Hands example, it is $[-0.09, 3.34]$ when using the diffusion-based model. These values indicate that $\bar{\bm{\varphi}}$ is not diffeomorphic. Similarly, the curvature registration also produces a non-diffeomorphic mapping $\bar{\bm{\varphi}}$ for the Hands example. However, the proposed model not only yields diffeomorphic transformations with $\det(\nabla\bar{\bm{\varphi}}) \in[0.46, 3.09]$ and $\det(\nabla\bar{\bm{\varphi}}) \in[0.32, 6.33]$ for the two examples but also achieves the best $\rm{Re_{-}SSD}$, \emph{ssim}, and \emph{psnr} scores. These results demonstrate that the proposed method effectively prevents grid folding and achieves excellent registration performance.

\subsection{Comparisons with other diffeomorphic models}
Comparisons were also carried out on images of Bigcircle, Pineapple, and Chest with resolutions of $200\times 200$, $128\times 128$, and $256\times 256$, respectively. A comparison was made between our approach and the state-of-the-art diffeomorphic models, including Hyper-elastic \cite{Hyperelastic2013}, LDDMM \cite{MangA}, Diffeomorphic Log Demons \cite{DiffLogDemons}, and Hsiao model \cite{hsiao2014new}. For the Bigcircle and Pineapple examples, a three-level multilevel strategy was employed, while a four-level multilevel strategy was used for the Chest example. The maximum number of iterations, $\text{MaxIter}$, was set to 100 for all models.

\begin{figure}
  \begin{center}
    \includegraphics[width=14cm,height=11cm]{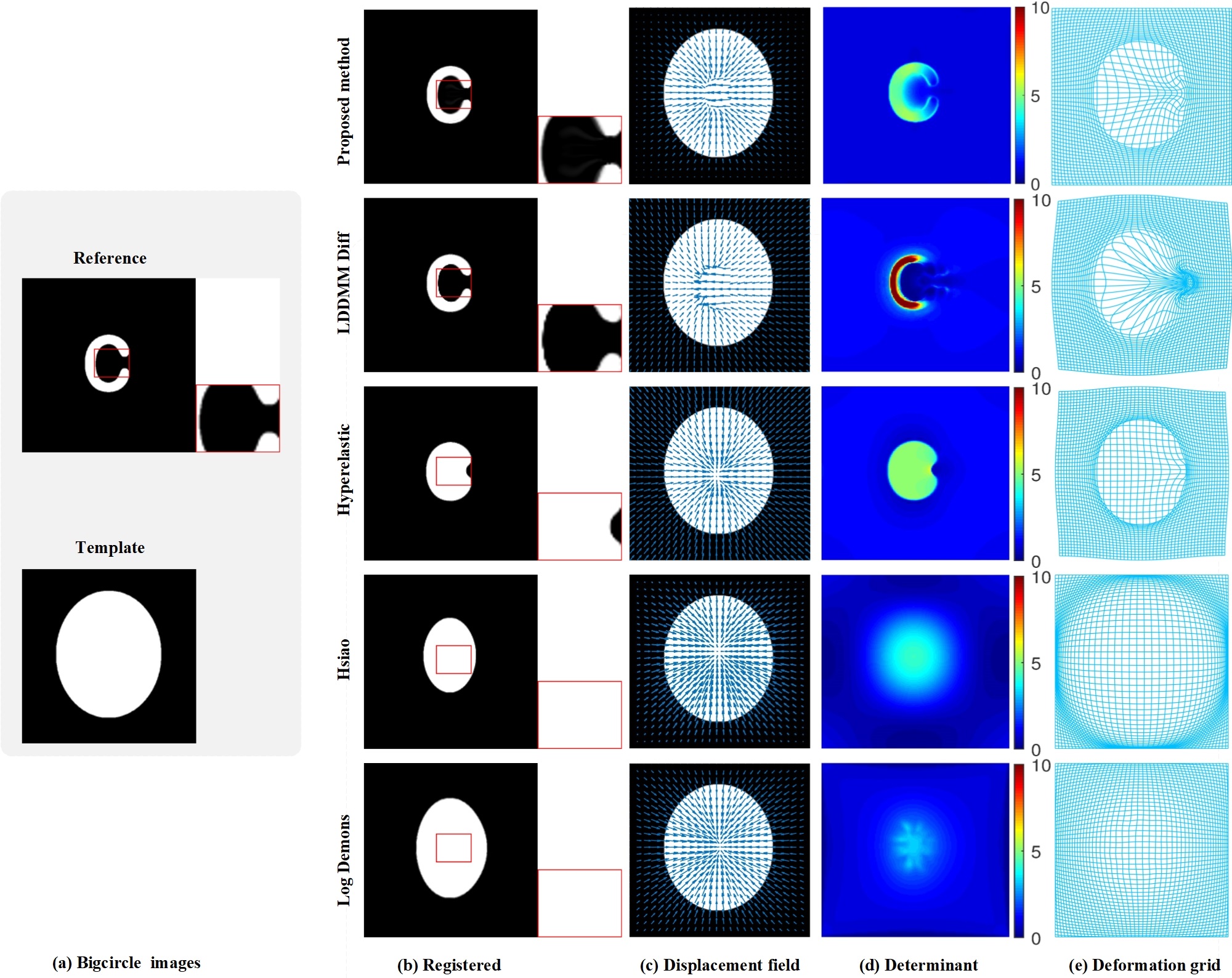}     
  \end{center}
  \caption{Comparisons of diffeomorphic models for the Bigcircle example. (a) the reference and template images; (b) the registered results of these models with the optimal parameters (Bigcircle: $\tau_1=2.2 $, $\tau_2=2e-2 $, $\tau_3 =1e-3$, $ \lambda=1 $, $ \gamma=160 $, $ \rho=1.02 $ for the proposed model, $\alpha=650$ for the LDDMM model, $\alpha_1=18$, $\alpha_2=90$, $\alpha_3=0$, $\alpha_4=6$ for the Hyper-elastic model, $\sigma_{fluid}=2.6$, $\sigma_{diffusion}=1.3$, $\sigma_{i}=1$, $\sigma_{x}=2$ for the Log Demons model); (c) the displacement fields; (d) the Jacobian determinant hotmaps; (e) the deformation grids.}\label{fig:testBigcircle}
\end{figure}

\begin{table}[H]
  \centering
 \setlength{\tabcolsep}{5.6pt} 
  \begin{lrbox}{\mybox} 
    \begin{tabular}{llccccccccc}
      \toprule[1.8pt]
       \textbf{Examples} & \textbf{Methods}  & $\overline{\det}(J(\bar{\bm{\varphi}}))$&$ R_{min} $& $\det_{\min}(J(\bar{\bm{\varphi}}))$ & $\det_{\max}(J(\bar{\bm{\varphi}}))$ & \emph{ssim}  & $\rm{Re_{-}SSD}$ &\emph{psnr} & Time(s) & Iters\\
      \hline
      \multirow{5}{*}{\textbf{Bigcircle}}  & Proposed    &    1.004  & +&   0.70  & 5.56  & 0.9652  & \textbf{ 0.08}\%  & \textbf{21.93} & \textbf{9.13} & 176\\
      & LDDMM        &   2.665  & +&  0.05  &  \emph{39.88}  & \textbf{0.9905} & 0.33\%  & 15.88 & 14.87  & 19\\
      
      & Hyper-elastic &   1.935  & +&  0.14   &   6.02         & 0.9436 & 9.48\%  & 3.86 & 13.22  & 24 \\
      
      & Log Demons   &   1.293  & +&  0.003  &  3.25          & 0.7908 & 48.68\%  & 1.14 & 9.29  & 208 \\
      
      &Hsiao     &   2.011  & +&  0.001  &  4.22          & 0.8819 & 19.42\%  & 2.32 & 41.86  & 118 \\
      \hline
      \multirow{5}{*}{\textbf{Pineapple}}  & Proposed    &   1.001  & + &  0.24  & 3.65  & \textbf{0.9770}  & \textbf{0.26}\%  & \textbf{31.16} & 3.71   & 96 \\
      &    LDDMM     &  1.072    & + &   0.11 &  3.73  &  0.9538  &  1.06 \%   & 25.03 & 3.90  &  10\\
      
      & Hyper-elastic & 1.040     & + &  0.06  &  2.53  &  0.9244   &  1.75\%   & 22.84 & \textbf{1.76}  & 12\\
      
      & Log Demons   &  1.112    & + &  0.002 &  3.40  &  0.8955   & 1.52 \%   & 19.49 & 4.87  & 204\\
      
      &Hsiao     &  1.077    & + &  0.38  &  2.11      & 0.8178    & 9.06\%   & 15.66 & 29.58 &  239\\
      \hline

      \multirow{5}{*}{{\textbf{Chest}} }     & {Proposed}    &   {1.000}   & {+} &   {0.42}  &    {2.61}  & {\textbf{0.9631}}  &  {\textbf{0.53}\%}  & {\textbf{29.74}} & {20.93} &  {139} \\     
      
      & {LDDMM}        &    {0.998}  & {+} &  {0.05}  &  {13.01}     & {0.9417}   &  {1.50\%}  & {25.23} & {60.02}  & {120} \\
      
      & {Hyper-elastic} &    {0.998}  & {+} &   {0.23}  &  {3.17}    & {0.9173}   &  {2.13\%}  & {23.67} & {16.03}  & {114} \\
      
      & {Log Demons}   &    {1.156}   & {+} &  {0.16}   &  {4.90}    & {0.9004}   &  {2.44\%}  & {20.11} & {\textbf{15.24}} & {400} \\
      
      & {Hsiao}        &    {1.024}   & {+} &  {0.63}   &  {1.43}    & {0.7655}   & {14.38\%}  & {15.42} &  {67.06}   & {373} \\ 

      \toprule[1.5pt]
    \end{tabular}
  \end{lrbox}
   \caption{The quantitative evaluation, time, and iteration number comparisons of the proposed and other diffeomorphic models. The best metrics values are highlighted by the \textbf{bold}.}
  \scalebox{0.86}{\usebox{\mybox}}\label{tablediffeomorphic} 
\end{table}

\begin{figure}
  \begin{center}
    \includegraphics[width=14cm,height=11cm]{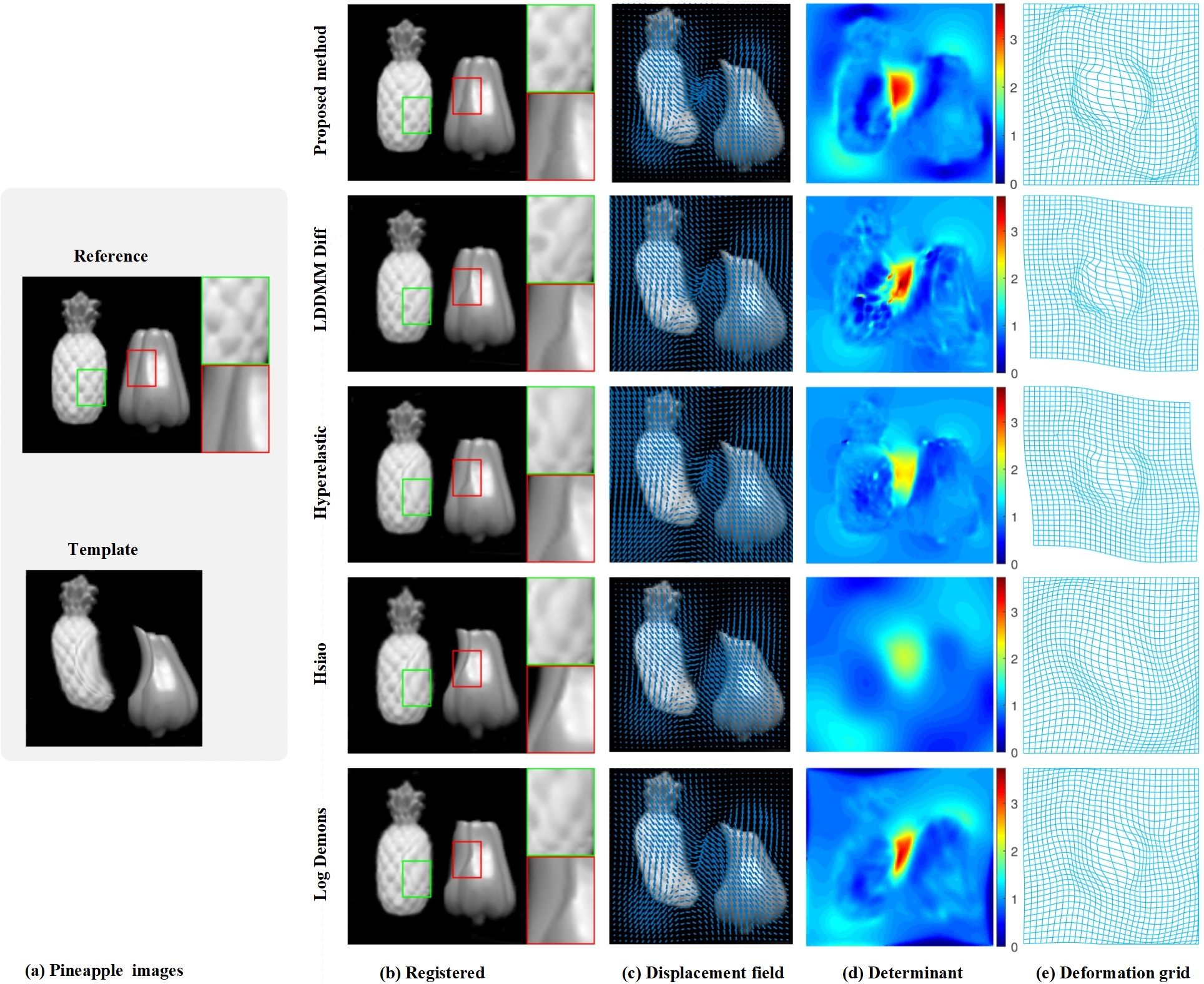}     
  \end{center}
  \caption{Comparisons of diffeomorphic models for the Pinealpple example. (a) the reference and template images; (b) the registered results of these models with the optimal parameters (Pinealpple: $\tau_1=0.2 $, $\tau_2=1e-3 $, $\tau_3 =1e-4$, $ \lambda=0.08 $, $ \gamma=130 $, $ \rho=1.01 $ for the proposed model, $\alpha=100$ for the LDDMM model, $\alpha_1=28$, $\alpha_2=22$, $\alpha_3=0$, $\alpha_4=1$ for the Hyper-elastic model, $\sigma_{fluid}=2.2$, $\sigma_{diffusion}=1.2$, $\sigma_{i}=1$, $\sigma_{x}=2$ for the Log Demons model); (c)~the displacement fields; (d) the Jacobian determinant hotmaps; (e) the deformation grids.}\label{fig:testPineapple}
\end{figure}

\begin{figure}[h]
  \begin{center}
    \includegraphics[width=14cm,height=11cm]{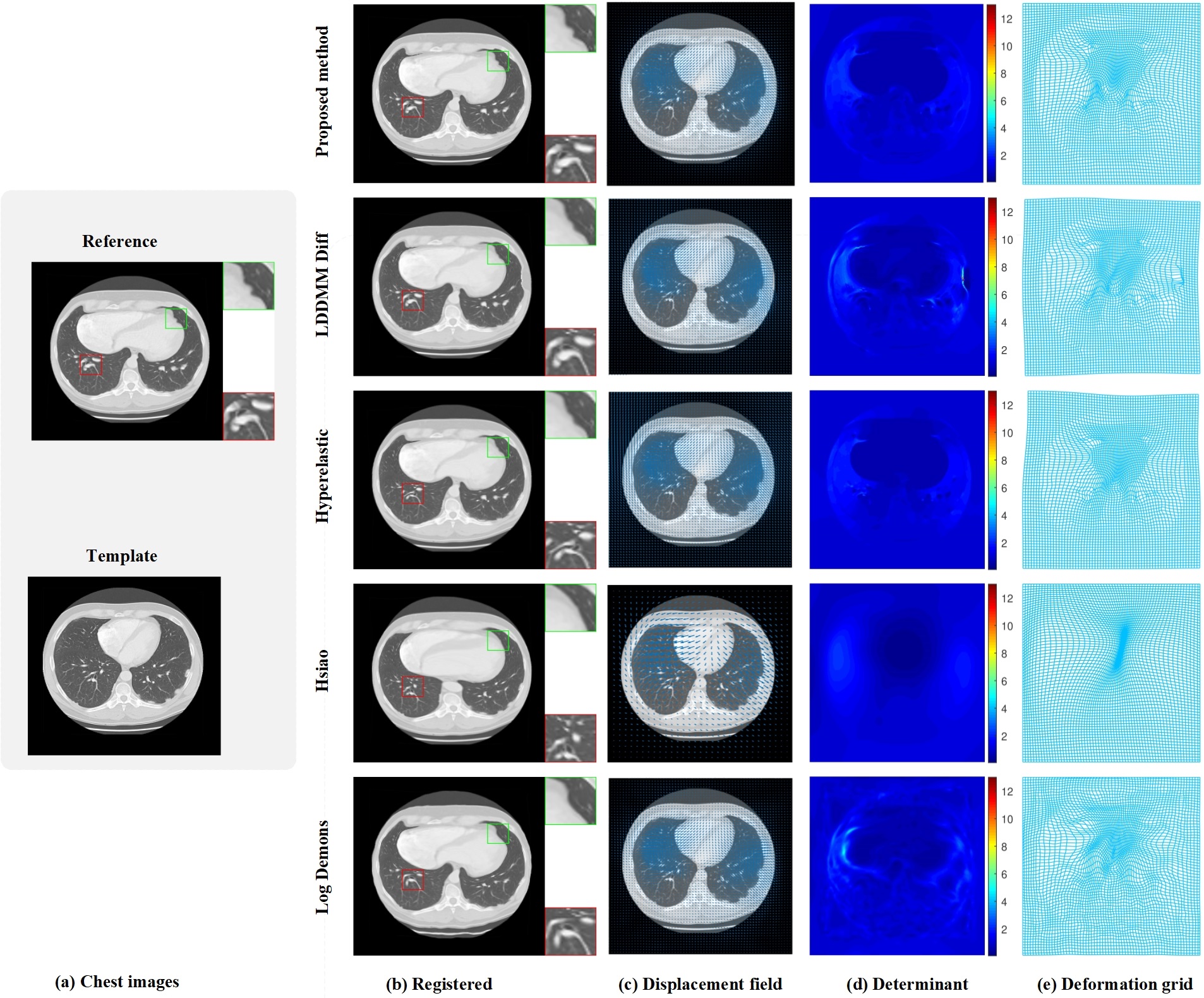}       
  \end{center}
  \caption{Comparisons of diffeomorphic models for the Chest example. (a) the reference and template images; (b) the registered results of these models with the optimal parameters (Chest: $\tau_1=1 $, $\tau_2=1e-2 $, $\tau_3 =1e-3$, $ \lambda=2 $, $ \gamma=200 $, $ \rho=1.02 $ for the proposed model, $\alpha=600$ for the LDDMM model, $\alpha_1=6$, $\alpha_2=100$, $\alpha_3=0$, $\alpha_4=20$ for the Hyper-elastic model, $\sigma_{fluid}=1.8$, $\sigma_{diffusion}=1.2 $, $\sigma_{i}=1$, $\sigma_{x}=3$ for the Log Demons model); (c) the displacement fields; (d) the Jacobian determinant hotmaps; (e) the deformation grids.}\label{fig:testchest}
\end{figure}

The results presented in \Cref{fig:testBigcircle}-\Cref{fig:testchest} indicate that the proposed method outperforms the other four methods. It is evident from \Cref{fig:testBigcircle} that the Hyper-elastic, Log Demons, and Hsiao models fail to accurately register the alphabet C. Although the LDDMM method produces an acceptable result, it significantly alters the volume. In contrast, the proposed method successfully registers the sharp edges of the large deformation image and generates a smooth deformation, as depicted in \Cref{fig:testBigcircle} (e) with the visualization of the deformation grid. Furthermore, \Cref{fig:testPineapple} (b) demonstrates that the proposed method is more effective in registering "low-contrast" regions, while \Cref{fig:testchest} (b) shows that it performs better in matching blood vessels and sharp edges. The zoom-in regions reveal that the internal structure of the registered images obtained by our method closely resembles the ground-truth reference images. However, the other four methods fail to produce satisfactory registration results. 

To further validate the effectiveness of our approach, we assess the average value of the Jacobian determinant $\overline{\det}(J(\bar{\bm{\varphi}}))$, the range of the Jacobian determinant, and registration accuracy. As depicted in~\Cref{tablediffeomorphic}, all methods produce diffeomorphic transformations. Specifically, our approach yields an average Jacobian determinant closest to 1. Additionally, compared to LDDMM and Hyper-elastic methods, our approach exhibits a smaller range for the Jacobian determinant. For example, for the Bigcircle example, the range of $\det(\nabla \bar{\bm{\varphi}})$ is [0.05, 39.88] for LDDMM, [0.14, 6.02] for Hyper-elastic, and [0.70, 5.56] for our proposed model. This is because our proposed model explicitly controls and penalizes volume change, thereby promoting volume preservation and smoother deformation. For the Chest example using the LDDMM method, the velocity field is modeled in a slightly larger padded spatial domain to ensure diffeomorphic deformation~\cite{MangA}. Furthermore, our proposed method achieves significantly better $\rm{Re_{-}SSD}$, \emph{ssim}, and \emph{psnr} scores compared to state-of-the-art diffeomorphic registration models. Notably, our approach achieves the best $\rm{Re_{-}SSD}$ of 0.08\% and \emph{psnr} of 21.93 for the Bigcircle example. Moreover, our proposed method demonstrates competitive advantages in terms of registration time, with both the proposed and LDDMM methods exhibiting lower time overheads for the Bigcircle example. In conclusion, these comparisons demonstrate the superiority of our approach in accurately handling large deformations and volume preservation on average. 

Moreover, the relative error ($\rm {Re_{-}SSD }$) versus the number of iterations is shown in \Cref{fig:three-ressd}, while the average CPU time per iteration at each level is displayed in \Cref{fig:three-average-time}. Furthermore, \Cref{tablediffeomorphic} presents the CPU time and the total number of iterations (Iters). Based on the results shown in \Cref{fig:three-ressd}, it is evident that the Log Demons and Hsiao methods exhibit a slower rate of convergence. These methods require more iterations to meet the stopping criteria and sometimes fail to converge altogether. Conversely, the LDDMM and Hyper-elastic methods converge quickly, making it easier to reach the stopping criteria and reducing the number of iterations. It is important to note that our method achieves a comparable convergence speed to LDDMM and Hyper-elastic. However, due to the adoption of a stricter stopping criterion, our method requires more iterations to satisfy this criterion, resulting in a better $ \rm{Re_{-}SSD} $. By examining \Cref{fig:three-average-time} and \Cref{tablediffeomorphic}, it can be observed that the proposed method yields the best $ \rm{Re_{-}SSD} $, indicating that our method produces good registration results. Although our method may not always be the fastest in terms of CPU time, each iteration step requires relatively short time, making our method somewhat competitive.

\begin{figure}
\centering
    \subfigure[Bigcircle]{
        \begin{minipage}[t]{0.31\linewidth}
            \centering
            \includegraphics[width=5cm]{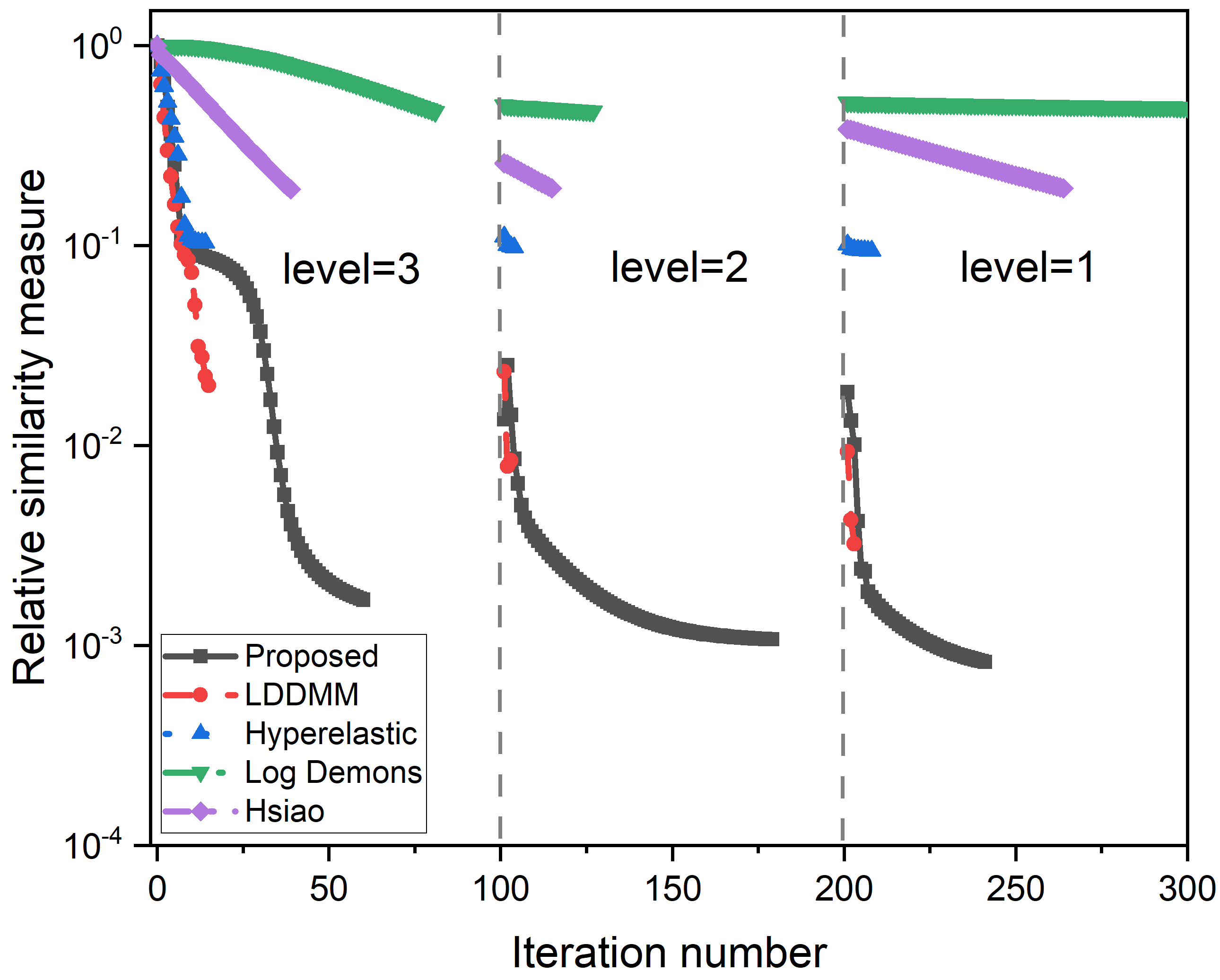}
        \end{minipage}
    }%
    \subfigure[Pineapple]{
        \begin{minipage}[t]{0.31\linewidth}
            \centering
            \includegraphics[width=5cm]{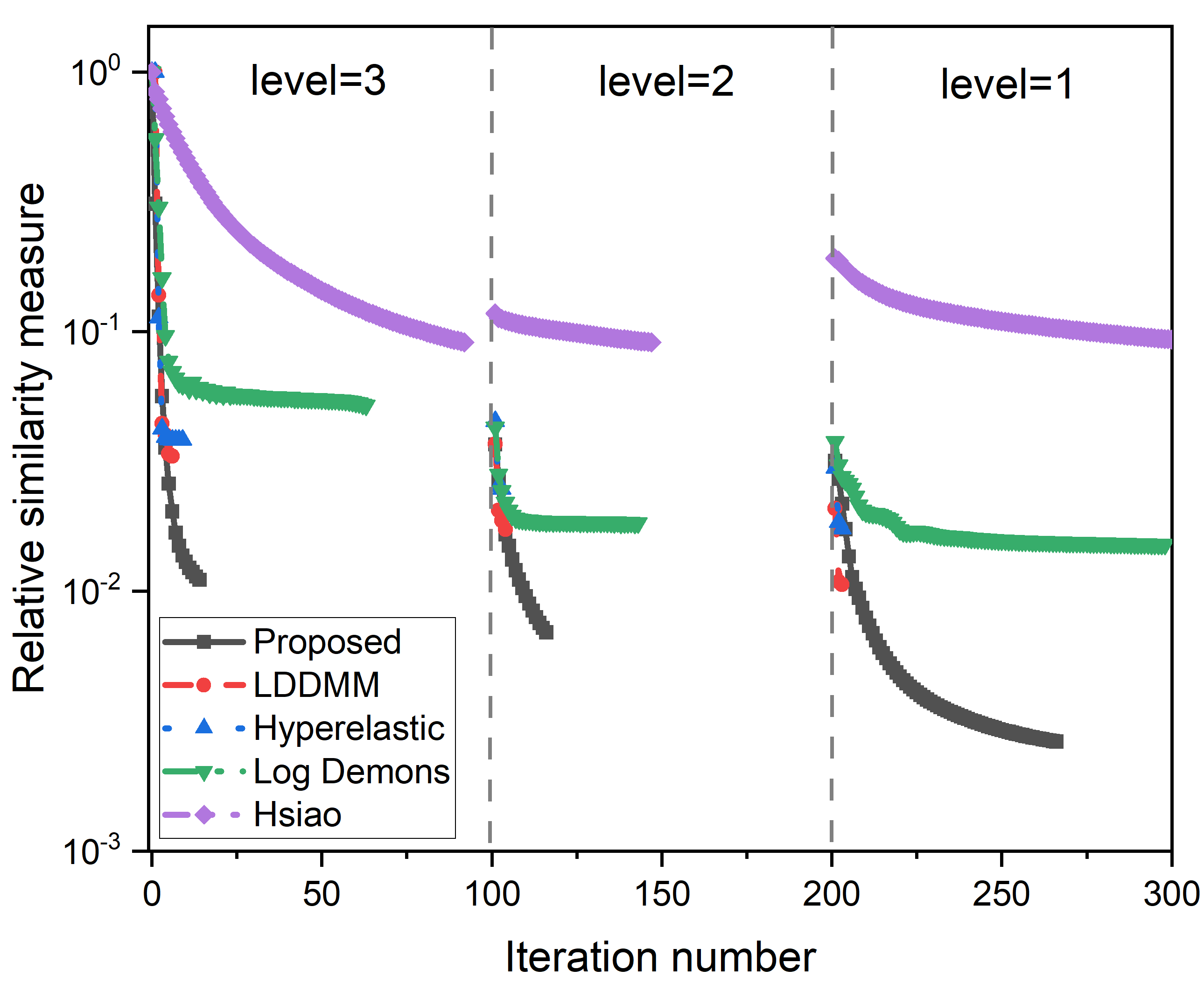}
        \end{minipage}
    }%
    \subfigure[Chest]{
        \begin{minipage}[t]{0.31\linewidth}
            \centering
            \includegraphics[width=5cm]{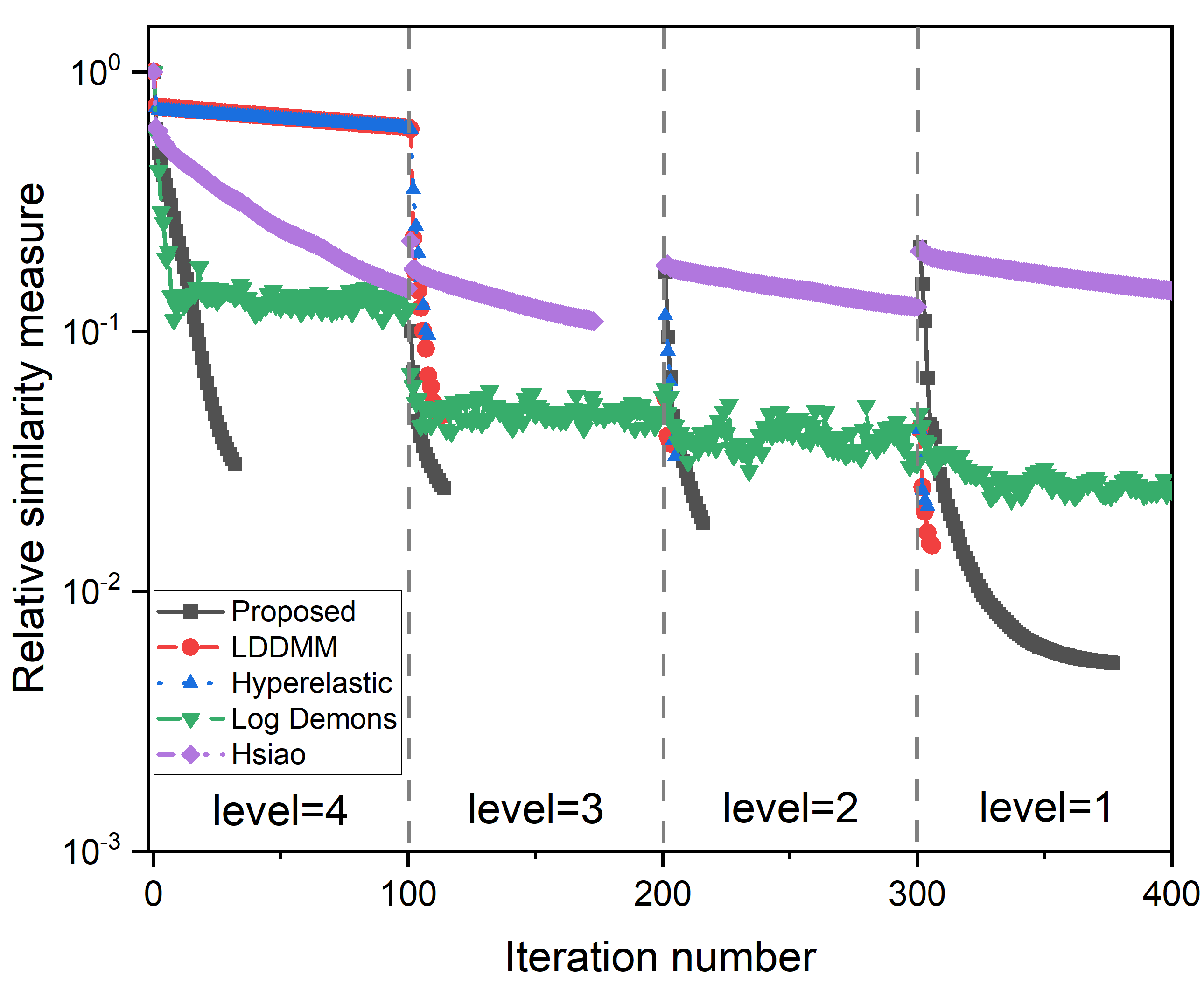}
        \end{minipage}
    }%
    \caption{The relative similarity measure ($\rm{Re_{-}SSD}$) versus iteration number, where the y-coordinate is logarithmic.}
    \label{fig:three-ressd}
\end{figure}

\begin{figure}
\centering
    \subfigure[Bigcircle]{
        \begin{minipage}[t]{0.30\linewidth}
            \centering
            \includegraphics[width=4.8cm]{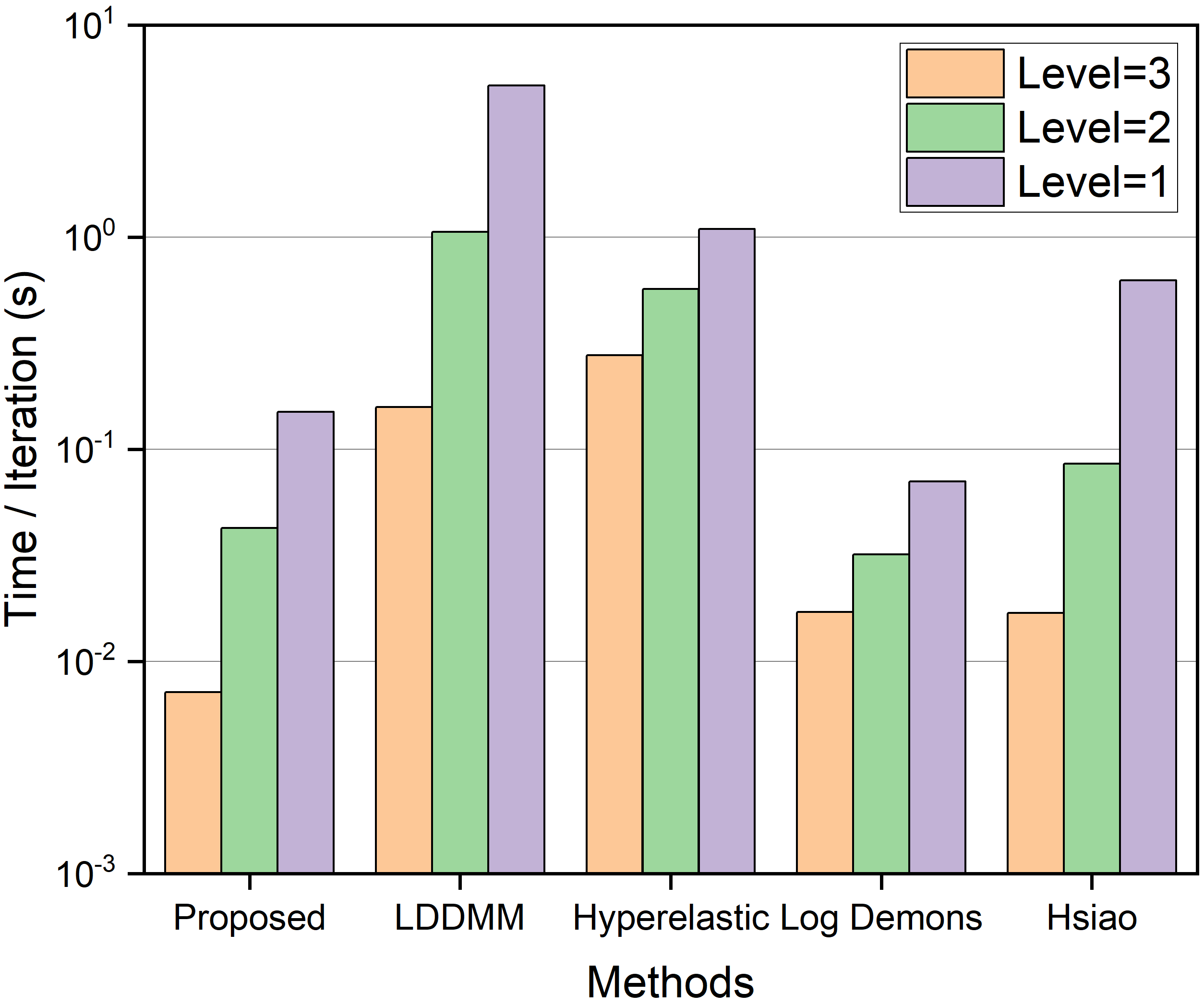}
        \end{minipage}
    }%
    \subfigure[Pineapple]{
        \begin{minipage}[t]{0.30\linewidth}
            \centering
            \includegraphics[width=4.8cm]{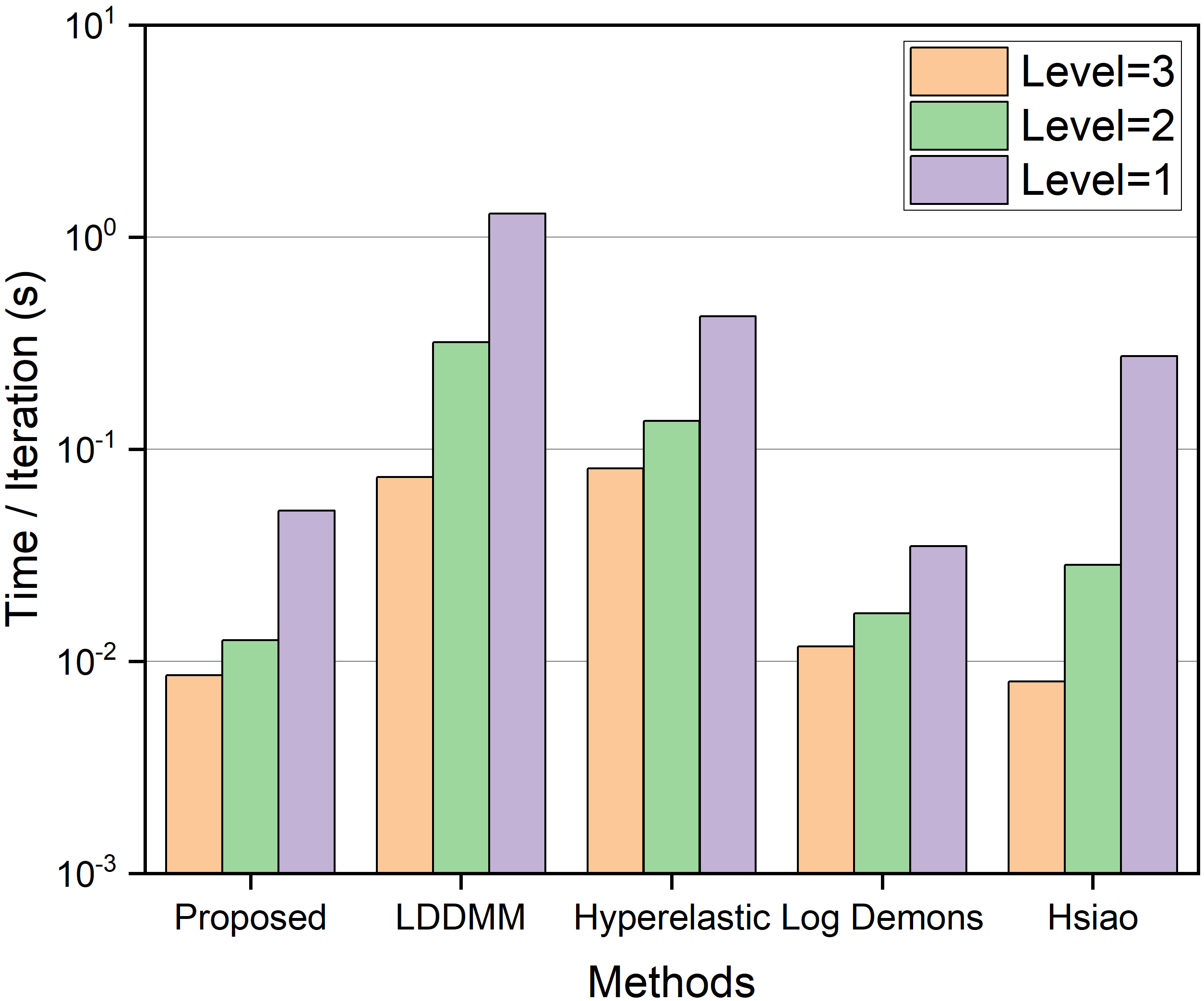}
        \end{minipage}
    }%
    \subfigure[Chest]{
        \begin{minipage}[t]{0.30\linewidth}
            \centering
            \includegraphics[width=4.8cm]{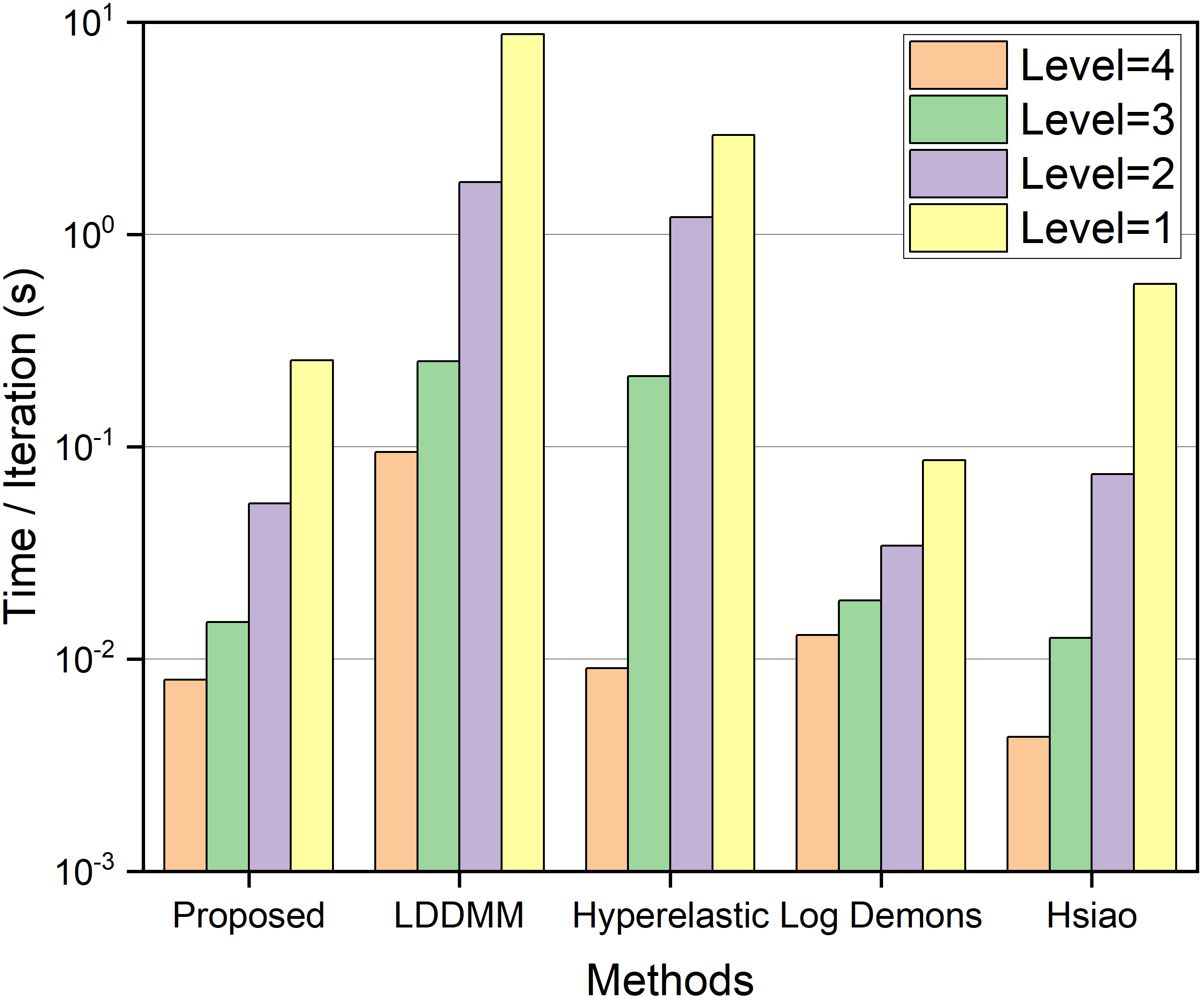}
        \end{minipage}
    }%
    \caption{Average time per iteration (CPU time / number of iterations in each level), where the y-coordinate is logarithmic.}
    \label{fig:three-average-time}
\end{figure}

\subsection{The generalization ability of the proposed model}
However, because traditional methods have limitations in obtaining model parameters and algorithm parameters, the ability to make generalizations refers to how accurately a model can predict results on new data, even if the data was obtained from a specific dataset. To evaluate the generalizability of the proposed method, we calculated the average of each model or algorithm parameter used in the aforementioned examples (11 sets of experiments) to obtain a set of universal parameters. The formula for calculating each universal parameter is as follows: 
\[{p}^\star = \frac{1}{n} \sum_{i=1}^{n} p_i,\] 
where $p_i \in \{\tau^i_1, \tau^i_2, \tau^i_3, \lambda^i, \gamma^i\}$ and $n= 11 $. We then applied these parameter configurations to make predictions on new images. It is widely recognized in the field of deep learning that model parameters can be learned and optimized by training on a large amount of data. We used these parameters to predict three sets of new images with different sources, resolutions, and deformation sizes (synthetic image: AR with a resolution of $128\times128$, natural image: Cameraman with a resolution of $256\times256$, medical image: Head with a resolution of $512\times512$). For comparison purposes, LDDMM, Hyper-elastic, and Log Demons methods underwent the same treatment. The relevant numerical results are as follows.


\begin{figure}[pos=h]
\centering
\includegraphics[width=1.0\linewidth]{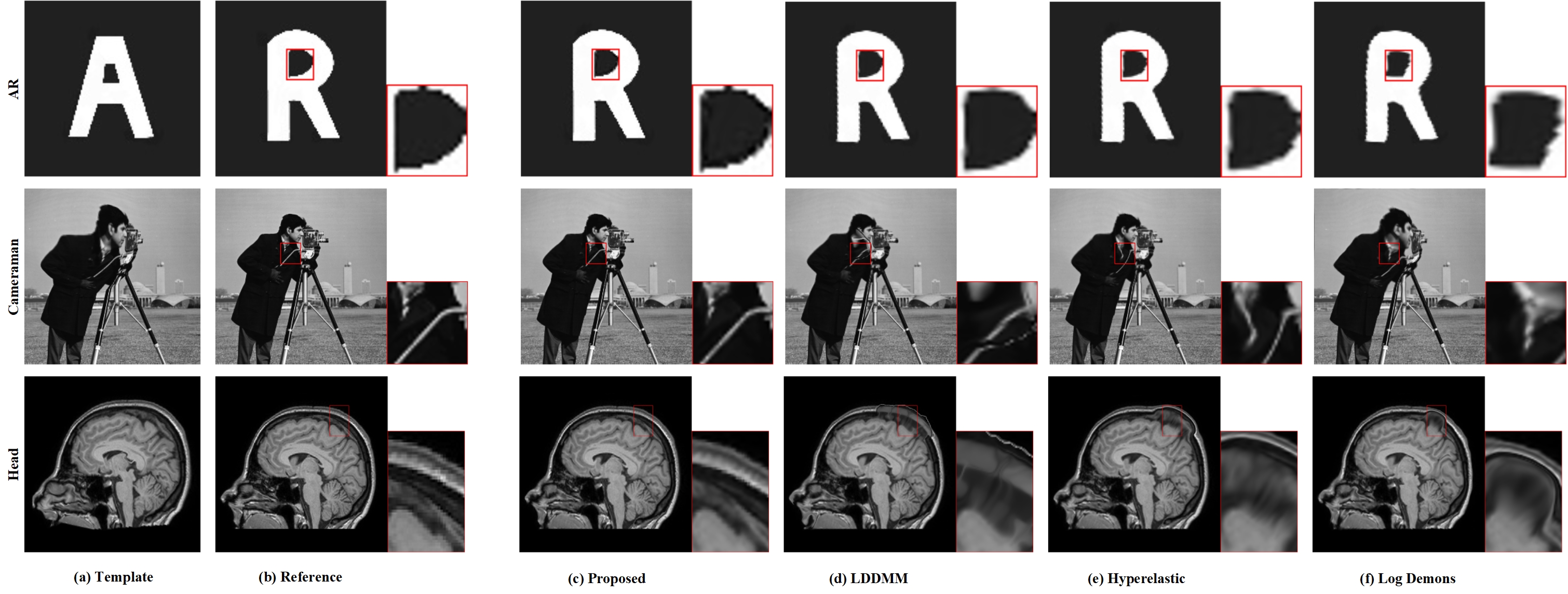}
\caption{Comparison of diffeomorphic models for new images using average parameters. (a) the template images; (b) the reference images; (c)
the registered results of the proposed model with the average parameters $ {\tau}_1^\star=0.9, {\tau}_2^\star=\text{7.2e-3}, {\tau}_3^\star=\text{6.2e-3}, {\lambda}^\star=0.8, {\gamma}^\star=133.1 $; (d) the registered results of the LDDMM with the average parameter ${\alpha}^\star=450$; (e) the registered results of the Hyper-elastic with the average parameters  ${\alpha}_1^\star=17$, ${\alpha}_2^\star=71$, ${\alpha}_3^\star=0$, ${\alpha}_4^\star=9$; (f) the registered results of the Log Demons with the average parameters ${\sigma}_{fluid}^\star=2.2$, ${\sigma}_{diffusion}^\star=1.23 $, ${\sigma}_{i}^\star=1$, ${\sigma}_{x}^\star =2.3$.}\label{fig:prediction-res}
\end{figure}

After closely examining the zoomed-in regions in \Cref{fig:prediction-res}, it is clear that our method consistently achieves satisfactory results when applied to three different sets of images using the average parameters. While the LDDMM and Hyper-elastic methods produce visually appealing results for synthetic images, their registration outcomes for natural and medical images do not align with the ground-truth reference images. Additionally, the results obtained using the Log Demons method are noticeably less satisfactory. Furthermore, \Cref{tab:new-experiments-predict} presents some quantitative metrics. It is important to note that our method achieves the best values for $\rm{Re_{-}SSD}$, \emph{ssim}, and \emph{psnr}, all of which have $ R_{min} > 0 $. This indicates that the deformation achieved by our model is diffeomorphic and leads to a superior registration performance. In contrast, the LDDMM method exhibits a significant variation in the range of the Jacobian determinant, with negative determinants observed in both natural and medical experiments, suggesting non-diffeomorphic deformation. It is worth mentioning that the LDDMM method requires the establishment of a velocity field within the padded domain to mitigate boundary effects \cite{MangA}. However, in the case of the Head and Cameraman images, the boundary significantly impacts registration accuracy and the Jacobian determinant of the LDDMM method. Despite extending the boundary of the domain by eight cells, it is still unable to prevent the occurrence of a negative Jacobian determinant. Similarly, the Log Demons method also yields negative Jacobian determinants in the Head example. Therefore, it can be concluded that when using the same average parameters on different new images, the LDDMM, Hyper-elastic, and Log Demons methods cannot guarantee satisfactory registration results for each experiment. In contrast, our method achieves favorable registration outcomes for three image pairs using the same average parameters, indicating its strong generalization capability.

As mentioned earlier, our approach can generate satisfactory results for new images using average (standard) parameters, despite being influenced by specific parameters. On the other hand, the LDDMM, Hyper-elastic, and Log Demons methods are highly sensitive to parameter variations under the same conditions. Therefore, our proposed model exhibits better generalization abilities compared to other advanced models.

\begin{table}[pos=H]
\centering
  \setlength{\tabcolsep}{6.6pt}
  \begin{lrbox}{\mybox} 
    \begin{tabular}{cccccccccc}
      \toprule[1.5pt]
       \textbf{Examples} & \textbf{Methods}  & $\overline{\det}(J(\bar{\bm{\varphi}}))$&$ R_{min} $& $\det_{\min}(J(\bar{\bm{\varphi}}))$ & $\det_{\max}(J(\bar{\bm{\varphi}}))$ & \emph{ssim}  & $\rm{Re_{-}SSD}$ &\emph{psnr} \\
      \hline
      \multirow{4}{*}{\textbf{AR}}      & Proposed    &  1.000    & + &   0.30 &    2.00  & \textbf{0.9930}  &  \textbf{0.18}\%  & \textbf{31.50} \\     
      
      & LDDMM        &    1.008   & + &  0.18  &  2.05    & 0.9819  &  1.61\%  & 21.96\\
      
      & Hyper-elastic &    0.947   & + &  0.41  &  1.77    & 0.9747  &  2.15\%  & 20.73 \\
      
      & Log Demons   &    1.051   & + &  0.36   &  2.11    & 0.9440   &  4.85\%  & 15.27 \\
      
      \hline
      \multirow{4}{*}{\textbf{Cameraman}}      & Proposed    &  1.000    & + &   0.44 &    2.48  & \textbf{0.9659}  &  \textbf{1.12}\%  & \textbf{31.09} \\     
      
      & LDDMM        &    0.994   & - &  -1.77  &  19.29    & 0.9450  &  5.80\%  & 23.93\\
      
      & Hyper-elastic &    0.996   & + &  0.40  &  2.40    & 0.9336  &  6.16\%  & 23.68 \\
      
      & Log Demons   &    1.054   & + &  0.23   &  2.65    & 0.8798   &  19.85\%  & 14.43 \\ 
  
      \hline
      \multirow{4}{*}{\textbf{Head}}      & Proposed    &  1.000    & + &   0.20  &    3.39  & \textbf{0.9402}  &  \textbf{2.06}\%  & \textbf{22.28} \\     
      
      & LDDMM        &    1.001   & - &  -3.10  &  30.35    & 0.9105  &  6.29\%  & 17.40\\
      
      & Hyper-elastic &   1.009   & + &  0.16  &  3.95    & 0.8782  &  11.28\%  & 14.87 \\
      
      & Log Demons   &    1.107   & - &  -0.07   &  6.01    & 0.8955   &  4.86\%  & 16.23 \\
      \toprule[1.5pt]
    \end{tabular}
  \end{lrbox}
  \caption{The quantitative evaluation comparisons of the proposed, LDDMM, Hyper-elastic, and Log Demons models. The best metrics values are highlighted by \textbf{bold}.}
  \scalebox{0.86}{\usebox{\mybox}}\label{tab:new-experiments-predict}
\end{table}

\subsection{3D registration experiment}
In \Cref{fig:3D} and \Cref{table3D}, we evaluate our method against state-of-the-art diffeomorphic methods such as LDDMM, Log Demons, and Hsiao models on Brain images of $ 128 \times 128 \times 128 $. For this experiment, we employ a multilevel strategy with ${L}=5$ and set $\text{MaxIter}= 200$ for all models. 

\begin{figure}[h]
  \begin{center}
    \includegraphics[width=14cm]{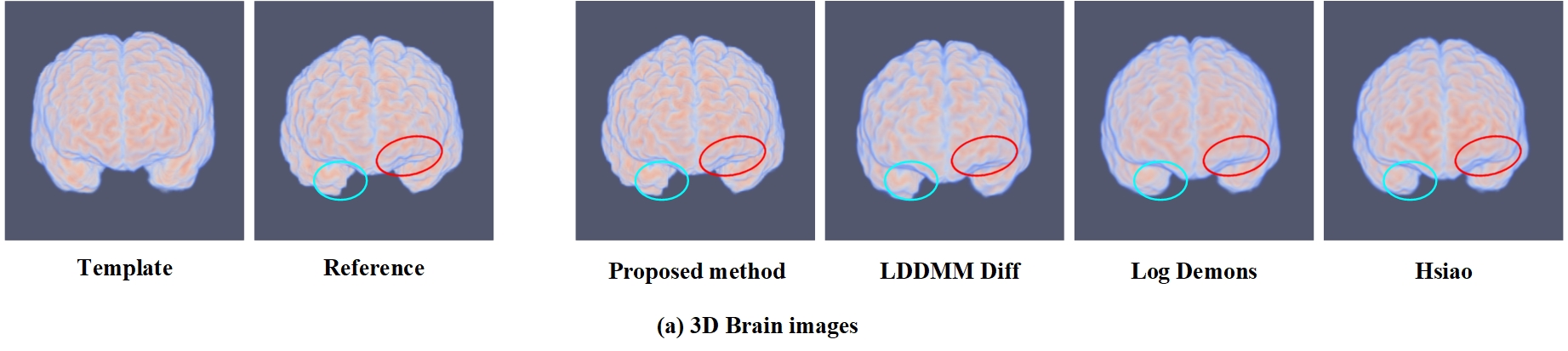}
    \includegraphics[width=14cm]{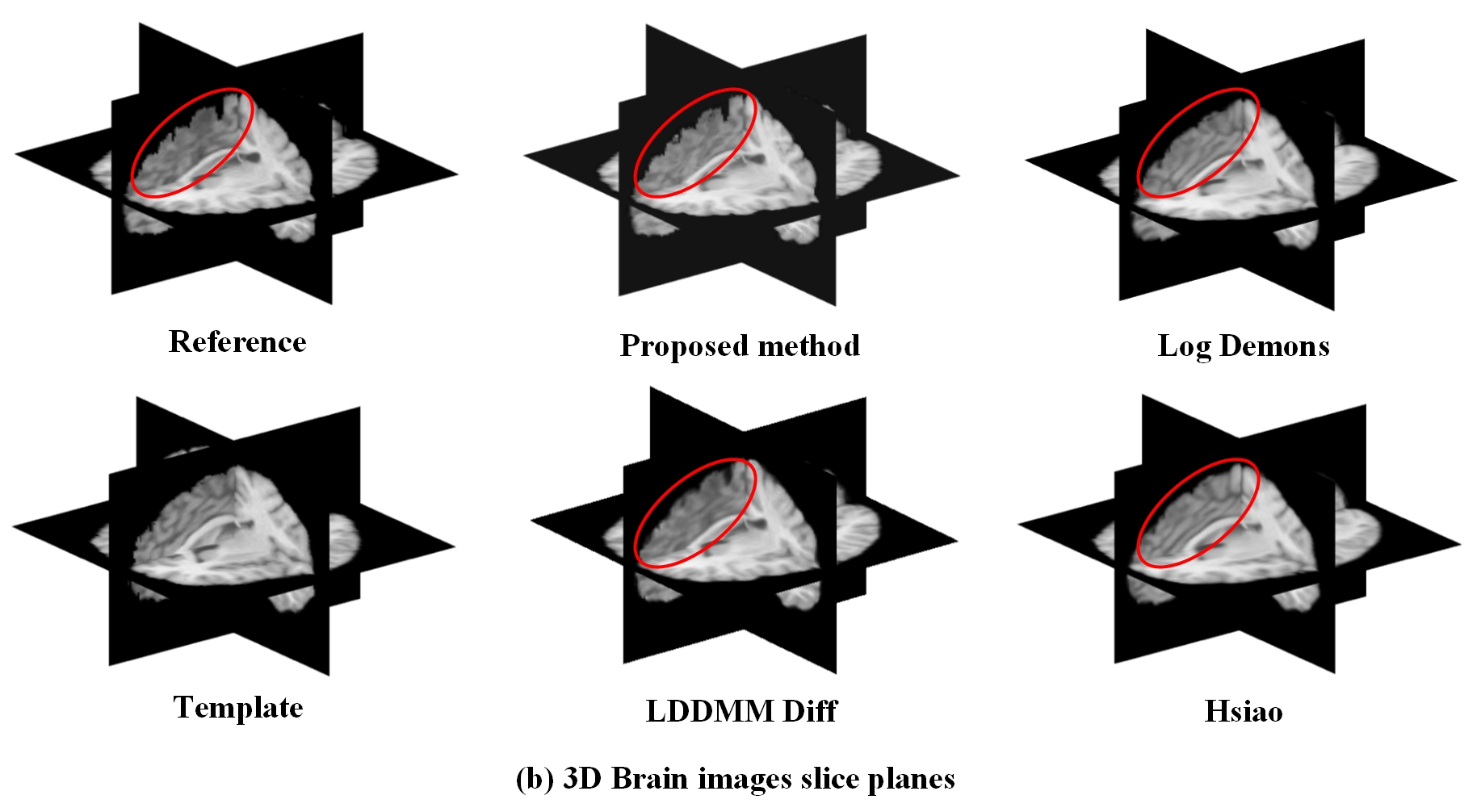}      
  \end{center}
  \caption{3D visualization of the registration problem for the Brain images. (a) shows the reference image, template image and deformed template images generated by these four models. (b) visualize the registered images from slice planes.}\label{fig:3D}
\end{figure}

Observing the surface of the brain in \Cref{fig:3D}, particularly the highlighted elliptical region, it is clear that our registration scheme produces results that are closer to the reference image compared to other methods. This indicates that our method is more effective in dealing with complicated details. The results presented in \Cref{table3D} demonstrate that all methods achieve diffeomorphic transformations, and our approach demonstrates a relatively narrow range of variation in the Jacobian determinant ($\det(\nabla \bar{\bm{\varphi}}) \in [0.22, 2.26]$). As shown in \Cref{table3D}, our proposed model consistently outperforms the other methods in terms of scores $\rm{Re_{-}SSD}$, \emph{ssim}, and \emph{psnr}. Furthermore, the Log Demons method has the lowest computational cost since it does not require solving the system of linear equations, which is typically the bottleneck in large-scale 3D registration problems. Among the remaining methods, our proposed approach also has a lower time cost, completing in only 1091.96 seconds. Therefore, this example demonstrates the applicability of our approach to 3D image registration problems, yielding top-ranking results with diffeomorphic transformations and minimal volume change.

\begin{table}[pos=H]
  \centering
   \setlength{\tabcolsep}{5pt}
  \begin{lrbox}{\mybox} 
    \begin{tabular}{cccccccccc}
      \toprule[1.5pt]
      \textbf{Example} & \textbf{Methods}  & $\overline{\det}(J(\bar{\bm{\varphi}}))$&$ R_{min} $& $\det_{\min}(J(\bar{\bm{\varphi}}))$ & $\det_{\max}(J(\bar{\bm{\varphi}}))$ & \emph{ssim}  & $\rm{Re_{-}SSD}$ &\emph{psnr}& Time(s)\\
      \hline
      \multirow{4}{*}{\textbf{Brain}}  
      & Proposed    &   1.000  & + &  0.2201  &2.26  & \textbf{0.9865} & \textbf{2.25}\%  & \textbf{24.64} & 1091.96\\
      & LDDMM       &  0.959      &+ & 0.0008 & 3.14 & 0.9822 & 3.19\% &22.99  & 1961.95 \\
      & Log Demons       &    1.001    &+ & 0.0315 & 2.80 & 0.8368 & 22.13\% & 14.90 & \textbf{398.43} \\
      & Hsiao        &    1.001    &+ & 0.6881 & 2.08 & 0.8560 & 25.46\% & 14.18 & 4698.91 \\
      \toprule[1.5pt]
    \end{tabular}
  \end{lrbox}
  \caption{The quantitative evaluation and time comparisons of the proposed and diffeomorphic models for 3D experiment. The best metrics values are highlighted
by the \textbf{bold}.}
  \scalebox{0.86}{\usebox{\mybox}}\label{table3D} 
\end{table}

\section{Conclusions}\label{sec:Section6}
In this paper, we propose a new registration model that incorporates a relaxation constraint on the Jacobian equation $\det(\nabla\bm{\varphi}(\bm{x})) = f(\bm{x})>0$. This constraint allows us to obtain smooth and diffeomorphic mappings for image registration problems involving large transformations. We observe that the relaxation function $f(\bm{x})$ can restrict the range of the Jacobian determinant $\det(\nabla\bm{\varphi}(\bm{x}))$. Therefore, instead of directly controlling the Jacobian determinant, we propose controlling $f(\bm{x})$ as an alternative approach that achieves the same goal indirectly. To ensure diffeomorphic and volume-preserving deformations, we introduce a novel penalty term $\int_{\Omega}\phi(f(\bm{x}))d\bm{x}$ that automatically controls the relaxation function such that $f(\bm{x})$ is positive and as close to one as possible. Additionally, we incorporate a regularization term $\int_{\Omega} \|\nabla f(\bm{x})\| ^2 d\bm{x}$ to enhance the smoothness of $f(\bm{x})$, thus indirectly improving the smoothness of the deformation. Furthermore, we analyze the existence of an optimal solution for the proposed variational model and present a numerical algorithm that combines penalty-splitting and multilevel schemes. This algorithm is capable of automatically detecting and correcting grid foldings. Finally, we conduct numerical experiments to demonstrate that the proposed model with the penalty term generates on average volume-preserving and smoother transformations. We also provide convergence comparisons between our algorithm and other existing algorithms. The results show that the proposed model, compared to classical registration models, produces diffeomorphic transformations and achieves better performance for 2D and 3D images with large deformations. In future work, we plan to extend our scheme to a multimodal registration model.

\section*{Acknowledgments}
\addcontentsline{toc}{section}{Acknowledgments}
This work was supported by the National Natural Science Foundation of China (grant numbers 11771369, 12322117, 12226008); the Education Bureau of Hunan Province, China (grant number 22A0119); the Natural Science Foundation of Hunan Province, China (grant numbers 2018JJ2375, 2018XK2304), and the Postgraduate Scientific Research Innovation Project of Xiangtan University (grant number XDCX2021B101), and National Key Research \& Development Program of China (Nos. 2023YFA1009300 and 2022YFC2504300).


\printcredits

\bibliographystyle{cas-model2-names}

\bibliography{ImgReg-Refs}

\begin{thebibliography}{62}
\expandafter\ifx\csname natexlab\endcsname\relax\def\natexlab#1{#1}\fi
\providecommand{\url}[1]{\texttt{#1}}
\providecommand{\href}[2]{#2}
\providecommand{\path}[1]{#1}
\providecommand{\DOIprefix}{doi:}
\providecommand{\ArXivprefix}{arXiv:}
\providecommand{\URLprefix}{URL: }
\providecommand{\Pubmedprefix}{pmid:}
\providecommand{\doi}[1]{\href{http://dx.doi.org/#1}{\path{#1}}}
\providecommand{\Pubmed}[1]{\href{pmid:#1}{\path{#1}}}
\providecommand{\bibinfo}[2]{#2}
\ifx\xfnm\relax \def\xfnm[#1]{\unskip,\space#1}\fi
\bibitem[{Arsigny et~al.(2006)Arsigny, Commowick, Pennec and
  Ayache}]{Berlin2006}
\bibinfo{author}{Arsigny, V.}, \bibinfo{author}{Commowick, O.},
  \bibinfo{author}{Pennec, X.}, \bibinfo{author}{Ayache, N.},
  \bibinfo{year}{2006}.
\newblock \bibinfo{title}{A log-{Euclidean} framework for statistics on
  diffeomorphisms}, in: \bibinfo{booktitle}{Med. Image Comput. Comput. Assist.
  Interv.}, pp. \bibinfo{pages}{924--931}.
\newblock \DOIprefix\doi{https://doi.org/10.1007/11866565_113}.
\bibitem[{Ashburner and Friston(2011)}]{ashburner2011}
\bibinfo{author}{Ashburner, J.}, \bibinfo{author}{Friston, K.J.},
  \bibinfo{year}{2011}.
\newblock \bibinfo{title}{Diffeomorphic registration using geodesic shooting
  and {Gauss-Newton} optimisation}.
\newblock \bibinfo{journal}{NeuroImage} \bibinfo{volume}{55},
  \bibinfo{pages}{954--967}.
\newblock \DOIprefix\doi{https://doi.org/10.1016/j.neuroimage.2010.12.049}.
\bibitem[{Aubert and Kornprobst(2006)}]{GAubert2006}
\bibinfo{author}{Aubert, G.}, \bibinfo{author}{Kornprobst, P.},
  \bibinfo{year}{2006}.
\newblock \bibinfo{title}{Mathematical Problems in Image Processing: Partial
  Differential Equations and the Calculus of Variations}.
\newblock \bibinfo{edition}{2nd} ed., \bibinfo{publisher}{Springer Publishing
  Company, Incorporated}.
\bibitem[{Bauer et~al.(2015)Bauer, Joshi and Modin}]{Bauer_2015}
\bibinfo{author}{Bauer, M.}, \bibinfo{author}{Joshi, S.},
  \bibinfo{author}{Modin, K.}, \bibinfo{year}{2015}.
\newblock \bibinfo{title}{Diffeomorphic density matching by optimal information
  transport}.
\newblock \bibinfo{journal}{SIAM J. Imaging Sci.} \bibinfo{volume}{8},
  \bibinfo{pages}{1718--1751}.
\newblock \DOIprefix\doi{https://doi.org/10.1137/151006238}.
\bibitem[{Beg et~al.(2005)Beg, Miller, Trouv{\'{e}} and Younes}]{Beg_2005}
\bibinfo{author}{Beg, M.F.}, \bibinfo{author}{Miller, M.I.},
  \bibinfo{author}{Trouv{\'{e}}, A.}, \bibinfo{author}{Younes, L.},
  \bibinfo{year}{2005}.
\newblock \bibinfo{title}{Computing large deformation metric mappings via
  geodesic flows of diffeomorphisms}.
\newblock \bibinfo{journal}{Int. J. Comput. Vis.} \bibinfo{volume}{61},
  \bibinfo{pages}{139--157}.
\newblock \DOIprefix\doi{https://doi.org/10.1023/B:VISI.0000043755.93987.aa}.
\bibitem[{Burger et~al.(2013)Burger, Modersitzki and
  Ruthotto}]{Hyperelastic2013}
\bibinfo{author}{Burger, M.}, \bibinfo{author}{Modersitzki, J.},
  \bibinfo{author}{Ruthotto, L.}, \bibinfo{year}{2013}.
\newblock \bibinfo{title}{A hyperelastic regularization energy for image
  registration}.
\newblock \bibinfo{journal}{SIAM J. Sci. Comput.} \bibinfo{volume}{35},
  \bibinfo{pages}{B132--B148}.
\newblock \DOIprefix\doi{https://doi.org/10.1137/110835955}.
\bibitem[{Charon and Trouv{\'{e}}(2013)}]{Charon_2013}
\bibinfo{author}{Charon, N.}, \bibinfo{author}{Trouv{\'{e}}, A.},
  \bibinfo{year}{2013}.
\newblock \bibinfo{title}{The varifold representation of nonoriented shapes for
  diffeomorphic registration}.
\newblock \bibinfo{journal}{SIAM J. Imaging Sci.} \bibinfo{volume}{6},
  \bibinfo{pages}{2547--2580}.
\newblock \DOIprefix\doi{https://doi.org/10.1137/130918885}.
\bibitem[{Chen(2021)}]{Chen_2021}
\bibinfo{author}{Chen, C.}, \bibinfo{year}{2021}.
\newblock \bibinfo{title}{Spatiotemporal imaging with diffeomorphic optimal
  transportation}.
\newblock \bibinfo{journal}{Inverse Probl.} \bibinfo{volume}{37},
  \bibinfo{pages}{115004}.
\newblock \DOIprefix\doi{https://doi.org/10.1088/1361-6420/ac2a91}.
\bibitem[{Chen et~al.(2019)Chen, Gris and Öktem}]{CChen2019}
\bibinfo{author}{Chen, C.}, \bibinfo{author}{Gris, B.},
  \bibinfo{author}{Öktem, O.}, \bibinfo{year}{2019}.
\newblock \bibinfo{title}{A new variational model for joint image
  reconstruction and motion estimation in spatiotemporal imaging}.
\newblock \bibinfo{journal}{SIAM J. Imaging Sci.} \bibinfo{volume}{12},
  \bibinfo{pages}{1686--1719}.
\newblock \DOIprefix\doi{https://doi.org/10.1137/18M1234047}.
\bibitem[{Chen and Öktem(2018)}]{CChen2018}
\bibinfo{author}{Chen, C.}, \bibinfo{author}{Öktem, O.}, \bibinfo{year}{2018}.
\newblock \bibinfo{title}{Indirect image registration with large diffeomorphic
  deformations}.
\newblock \bibinfo{journal}{SIAM J. Imaging Sci.} \bibinfo{volume}{11},
  \bibinfo{pages}{575--617}.
\newblock \DOIprefix\doi{https://doi.org/10.1137/17M1134627}.
\bibitem[{Christensen et~al.(1996)Christensen, Rabbitt and
  Miller}]{article1996}
\bibinfo{author}{Christensen, G.}, \bibinfo{author}{Rabbitt, R.},
  \bibinfo{author}{Miller, M.}, \bibinfo{year}{1996}.
\newblock \bibinfo{title}{Deformable templates using large deformation
  kinematics}.
\newblock \bibinfo{journal}{IEEE Trans. Image Process.} \bibinfo{volume}{5},
  \bibinfo{pages}{1435--1447}.
\newblock \DOIprefix\doi{https://doi.org/10.1109/83.536892}.
\bibitem[{Chumchob and Chen(2010)}]{MTV2010}
\bibinfo{author}{Chumchob, N.}, \bibinfo{author}{Chen, K.},
  \bibinfo{year}{2010}.
\newblock \bibinfo{title}{A variational approach for discontinuity-preserving
  image registration}.
\newblock \bibinfo{journal}{East-West J. Math.} \bibinfo{volume}{2010},
  \bibinfo{pages}{266--282}.
\bibitem[{Chumchob et~al.(2011)Chumchob, Chen and Brito-Loeza}]{Chumchob2011}
\bibinfo{author}{Chumchob, N.}, \bibinfo{author}{Chen, K.},
  \bibinfo{author}{Brito-Loeza, C.}, \bibinfo{year}{2011}.
\newblock \bibinfo{title}{A fourth-order variational image registration model
  and its fast multigrid algorithm}.
\newblock \bibinfo{journal}{Multiscale Model. Simul.} \bibinfo{volume}{9},
  \bibinfo{pages}{89--128}.
\newblock \DOIprefix\doi{https://doi.org/10.1137/100788239}.
\bibitem[{Collignon et~al.(1995)Collignon, Maes, Delaere, Vandermeulen, Suetens
  and Marchal}]{collignon1995}
\bibinfo{author}{Collignon, A.}, \bibinfo{author}{Maes, F.},
  \bibinfo{author}{Delaere, D.}, \bibinfo{author}{Vandermeulen, D.},
  \bibinfo{author}{Suetens, P.}, \bibinfo{author}{Marchal, G.},
  \bibinfo{year}{1995}.
\newblock \bibinfo{title}{Automated multi-modality image registration based on
  information theory}, in: \bibinfo{booktitle}{Inf. Process. Med. Imaging}, pp.
  \bibinfo{pages}{263--274}.
\bibitem[{D'Agostino et~al.(2003)D'Agostino, Maes, Vandermeulen and
  Suetens}]{Emiliano2003}
\bibinfo{author}{D'Agostino, E.}, \bibinfo{author}{Maes, F.},
  \bibinfo{author}{Vandermeulen, D.}, \bibinfo{author}{Suetens, P.},
  \bibinfo{year}{2003}.
\newblock \bibinfo{title}{A viscous fluid model for multimodal non-rigid image
  registration using mutual information}.
\newblock \bibinfo{journal}{Med. Image. Anal.} \bibinfo{volume}{7},
  \bibinfo{pages}{565--575}.
\newblock \DOIprefix\doi{https://doi.org/10.1016/S1361-8415(03)00039-2}.
\bibitem[{Dalca et~al.(2018)Dalca, Balakrishnan, Guttag and
  Sabuncu}]{Dalca_2018}
\bibinfo{author}{Dalca, A.V.}, \bibinfo{author}{Balakrishnan, G.},
  \bibinfo{author}{Guttag, J.}, \bibinfo{author}{Sabuncu, M.R.},
  \bibinfo{year}{2018}.
\newblock \bibinfo{title}{Unsupervised learning for fast probabilistic
  diffeomorphic registration}, in: \bibinfo{booktitle}{Med. Image Comput.
  Comput. Assist. Interv.}, pp. \bibinfo{pages}{729--738}.
\newblock \DOIprefix\doi{https://doi.org/10.1007/978-3-030-00928-1_82}.
\bibitem[{Dong(2013)}]{Dong_2013}
\bibinfo{author}{Dong, Y.}, \bibinfo{year}{2013}.
\newblock \bibinfo{title}{The proximal point algorithm revisited}.
\newblock \bibinfo{journal}{J. Optim. Theory Appl.} \bibinfo{volume}{161},
  \bibinfo{pages}{478--489}.
\newblock \DOIprefix\doi{https://doi.org/10.1007/s10957-013-0351-3}.
\bibitem[{Dupuis et~al.(1998)Dupuis, Grenander and Miller}]{paul1998}
\bibinfo{author}{Dupuis, P.}, \bibinfo{author}{Grenander, U.},
  \bibinfo{author}{Miller, M.I.}, \bibinfo{year}{1998}.
\newblock \bibinfo{title}{Variational problems on flows of diffeomorphisms for
  image matching}.
\newblock \bibinfo{journal}{Q. Appl. Math.} \bibinfo{volume}{56},
  \bibinfo{pages}{587--600}.
\bibitem[{Feydy et~al.(2017)Feydy, Charlier, Vialard and
  Peyr{\'{e}}}]{Feydy_2017}
\bibinfo{author}{Feydy, J.}, \bibinfo{author}{Charlier, B.},
  \bibinfo{author}{Vialard, F.X.}, \bibinfo{author}{Peyr{\'{e}}, G.},
  \bibinfo{year}{2017}.
\newblock \bibinfo{title}{Optimal transport for diffeomorphic registration},
  in: \bibinfo{booktitle}{Med. Image Comput. Comput. Assist. Interv.}, pp.
  \bibinfo{pages}{291--299}.
\newblock \DOIprefix\doi{https://doi.org/10.1007/978-3-319-66182-7_34}.
\bibitem[{Fischer and Modersitzki(2002)}]{fischer2002fast}
\bibinfo{author}{Fischer, B.}, \bibinfo{author}{Modersitzki, J.},
  \bibinfo{year}{2002}.
\newblock \bibinfo{title}{Fast diffusion registration}.
\newblock \bibinfo{journal}{Contemp. Math.} \bibinfo{volume}{313},
  \bibinfo{pages}{117--127}.
\newblock \DOIprefix\doi{https://doi.org/10.1090/conm/313/05372}.
\bibitem[{Fischer and Modersitzki(2003)}]{fischer2003curvature}
\bibinfo{author}{Fischer, B.}, \bibinfo{author}{Modersitzki, J.},
  \bibinfo{year}{2003}.
\newblock \bibinfo{title}{Curvature based image registration}.
\newblock \bibinfo{journal}{J. Math. Imaging Vis.} \bibinfo{volume}{18},
  \bibinfo{pages}{81--85}.
\newblock \DOIprefix\doi{https://doi.org/10.1023/A:1021897212261}.
\bibitem[{Fischer and Modersitzki(2004)}]{FISCHER2004107}
\bibinfo{author}{Fischer, B.}, \bibinfo{author}{Modersitzki, J.},
  \bibinfo{year}{2004}.
\newblock \bibinfo{title}{A unified approach to fast image registration and a
  new curvature based registration technique}.
\newblock \bibinfo{journal}{Linear Algebra Appl.} \bibinfo{volume}{380},
  \bibinfo{pages}{107--124}.
\newblock \DOIprefix\doi{https://doi.org/10.1016/j.laa.2003.10.021}.
\bibitem[{Fischler and Elschlager(1973)}]{Fischler}
\bibinfo{author}{Fischler, M.}, \bibinfo{author}{Elschlager, R.},
  \bibinfo{year}{1973}.
\newblock \bibinfo{title}{The representation and matching of pictorial
  structures}.
\newblock \bibinfo{journal}{IEEE Trans. Comput.} \bibinfo{volume}{C-22},
  \bibinfo{pages}{67--92}.
\newblock \DOIprefix\doi{https://doi.org/10.1109/T-C.1973.223602}.
\bibitem[{Frohn-Schauf et~al.(2007)Frohn-Schauf, Henn and Witsch}]{Frohn}
\bibinfo{author}{Frohn-Schauf, C.}, \bibinfo{author}{Henn, S.},
  \bibinfo{author}{Witsch, K.}, \bibinfo{year}{2007}.
\newblock \bibinfo{title}{Multigrid based total variation image registration}.
\newblock \bibinfo{journal}{Comput. Vis. Sci.} \bibinfo{volume}{11},
  \bibinfo{pages}{101--113}.
\newblock \DOIprefix\doi{https://doi.org/10.1007/s00791-007-0060-2}.
\bibitem[{Guo et~al.(2023)Guo, Chen, Choi and Lui}]{GUO2023}
\bibinfo{author}{Guo, Y.}, \bibinfo{author}{Chen, Q.}, \bibinfo{author}{Choi,
  G.P.}, \bibinfo{author}{Lui, L.M.}, \bibinfo{year}{2023}.
\newblock \bibinfo{title}{Automatic landmark detection and registration of
  brain cortical surfaces via quasi-conformal geometry and convolutional neural
  networks}.
\newblock \bibinfo{journal}{Comput. Biol. Med.} \bibinfo{volume}{163},
  \bibinfo{pages}{107185}.
\newblock \DOIprefix\doi{https://doi.org/10.1016/j.compbiomed.2023.107185}.
\bibitem[{Haber and Modersitzki(2004)}]{HaberNumerical}
\bibinfo{author}{Haber, E.}, \bibinfo{author}{Modersitzki, J.},
  \bibinfo{year}{2004}.
\newblock \bibinfo{title}{Numerical methods for volume preserving image
  registration}.
\newblock \bibinfo{journal}{Inverse Probl.} \bibinfo{volume}{20},
  \bibinfo{pages}{1621--1638}.
\newblock \DOIprefix\doi{https://doi.org/10.1088/0266-5611/20/5/018}.
\bibitem[{Haber and Modersitzki(2006)}]{Haber2007}
\bibinfo{author}{Haber, E.}, \bibinfo{author}{Modersitzki, J.},
  \bibinfo{year}{2006}.
\newblock \bibinfo{title}{Image registration with guaranteed displacement
  regularity}.
\newblock \bibinfo{journal}{Int. J. Comput. Vis.} \bibinfo{volume}{71},
  \bibinfo{pages}{361--372}.
\newblock \DOIprefix\doi{https://doi.org/10.1007/s11263-006-8984-4}.
\bibitem[{Han and Wang(2020)}]{HHan2020a}
\bibinfo{author}{Han, H.}, \bibinfo{author}{Wang, Z.}, \bibinfo{year}{2020}.
\newblock \bibinfo{title}{A diffeomorphic image registration model with
  fractional-order regularization and {Cauchy}--{Riemann} constraint}.
\newblock \bibinfo{journal}{SIAM J. Imaging Sci.} \bibinfo{volume}{13},
  \bibinfo{pages}{1240--1271}.
\newblock \DOIprefix\doi{https://doi.org/10.1137/19M1260621}.
\bibitem[{Han et~al.(2021)Han, Wang and Zhang}]{HHan2021a}
\bibinfo{author}{Han, H.}, \bibinfo{author}{Wang, Z.}, \bibinfo{author}{Zhang,
  Y.}, \bibinfo{year}{2021}.
\newblock \bibinfo{title}{Multiscale approach for two-dimensional diffeomorphic
  image registration}.
\newblock \bibinfo{journal}{Multiscale Model. Simul.} \bibinfo{volume}{19},
  \bibinfo{pages}{1538--1572}.
\newblock \DOIprefix\doi{https://doi.org/10.1137/20M1383987}.
\bibitem[{Henn(2006)}]{henn2006full}
\bibinfo{author}{Henn, S.}, \bibinfo{year}{2006}.
\newblock \bibinfo{title}{A full curvature based algorithm for image
  registration}.
\newblock \bibinfo{journal}{J. Math. Imaging Vis.} \bibinfo{volume}{24},
  \bibinfo{pages}{195--208}.
\newblock \DOIprefix\doi{https://doi.org/10.1007/s10851-005-3621-3}.
\bibitem[{Hering et~al.(2021)Hering, Häger, Moltz, Lessmann, Heldmann and van
  Ginneken}]{Alessa2021}
\bibinfo{author}{Hering, A.}, \bibinfo{author}{Häger, S.},
  \bibinfo{author}{Moltz, J.}, \bibinfo{author}{Lessmann, N.},
  \bibinfo{author}{Heldmann, S.}, \bibinfo{author}{van Ginneken, B.},
  \bibinfo{year}{2021}.
\newblock \bibinfo{title}{{CNN}-based lung {CT} registration with multiple
  anatomical constraints}.
\newblock \bibinfo{journal}{Med. Image. Anal.} \bibinfo{volume}{72},
  \bibinfo{pages}{102139}.
\newblock \DOIprefix\doi{https://doi.org/10.1016/j.media.2021.102139}.
\bibitem[{Hermosillo et~al.(2002)Hermosillo, Chefd'Hotel and
  Faugeras}]{hermosillo2002variational}
\bibinfo{author}{Hermosillo, G.}, \bibinfo{author}{Chefd'Hotel, C.},
  \bibinfo{author}{Faugeras, O.}, \bibinfo{year}{2002}.
\newblock \bibinfo{title}{Variational methods for multimodal image matching}.
\newblock \bibinfo{journal}{Int. J. Comput. Vis.} \bibinfo{volume}{50},
  \bibinfo{pages}{329--343}.
\newblock \DOIprefix\doi{https://doi.org/10.1023/A:1020830525823}.
\bibitem[{Horn and Schunck(1981)}]{horndetermining1981}
\bibinfo{author}{Horn, B.K.}, \bibinfo{author}{Schunck, B.G.},
  \bibinfo{year}{1981}.
\newblock \bibinfo{title}{Determining optical flow}.
\newblock \bibinfo{journal}{Artif. Intell.} \bibinfo{volume}{17},
  \bibinfo{pages}{185--203}.
\newblock \DOIprefix\doi{https://doi.org/10.1016/0004-3702(81)90024-2}.
\bibitem[{Hsiao et~al.(2014)Hsiao, Hsieh, Chen, Gong, Luo and
  Liao}]{hsiao2014new}
\bibinfo{author}{Hsiao, H.Y.}, \bibinfo{author}{Hsieh, C.Y.},
  \bibinfo{author}{Chen, X.}, \bibinfo{author}{Gong, Y.}, \bibinfo{author}{Luo,
  X.}, \bibinfo{author}{Liao, G.}, \bibinfo{year}{2014}.
\newblock \bibinfo{title}{New development of nonrigid registration}.
\newblock \bibinfo{journal}{ANZIAM J.} \bibinfo{volume}{55},
  \bibinfo{pages}{289--297}.
\newblock \DOIprefix\doi{https://doi.org/10.1017/S1446181114000091}.
\bibitem[{Hsieh and Charon(2021)}]{Hsieh_2021}
\bibinfo{author}{Hsieh, H.W.}, \bibinfo{author}{Charon, N.},
  \bibinfo{year}{2021}.
\newblock \bibinfo{title}{Diffeomorphic registration with density changes for
  the analysis of imbalanced shapes}, in: \bibinfo{booktitle}{Lect. Notes
  Comput. Sci.}, pp. \bibinfo{pages}{31--42}.
\newblock \DOIprefix\doi{https://doi.org/10.1007/978-3-030-78191-0_3}.
\bibitem[{Hömke et~al.(2007)Hömke, Frohn-Schauf, Henn and Witsch}]{Lars2007}
\bibinfo{author}{Hömke, L.}, \bibinfo{author}{Frohn-Schauf, C.},
  \bibinfo{author}{Henn, S.}, \bibinfo{author}{Witsch, K.},
  \bibinfo{year}{2007}.
\newblock \bibinfo{title}{Total variation based image registration}, in:
  \bibinfo{booktitle}{Image. Process. Partial. Differ. Equ.}, pp.
  \bibinfo{pages}{343--361}.
\newblock \DOIprefix\doi{https://doi.org/10.1007/978-3-540-33267-1_19}.
\bibitem[{Joshi and Miller(2000)}]{Joshi2000}
\bibinfo{author}{Joshi, S.}, \bibinfo{author}{Miller, M.},
  \bibinfo{year}{2000}.
\newblock \bibinfo{title}{Landmark matching via large deformation
  diffeomorphisms}.
\newblock \bibinfo{journal}{IEEE Trans. Image Process.} \bibinfo{volume}{9},
  \bibinfo{pages}{1357--1370}.
\newblock \DOIprefix\doi{https://doi.org/10.1109/83.855431}.
\bibitem[{Köstler et~al.(2008)Köstler, Ruhnau and Wienands}]{kostler2008}
\bibinfo{author}{Köstler, H.}, \bibinfo{author}{Ruhnau, K.},
  \bibinfo{author}{Wienands, R.}, \bibinfo{year}{2008}.
\newblock \bibinfo{title}{Multigrid solution of the optical flow system using a
  combined diffusion- and curvature-based regularizer}.
\newblock \bibinfo{journal}{Numer. Linear Algebra Appl.} \bibinfo{volume}{15},
  \bibinfo{pages}{201--218}.
\newblock \DOIprefix\doi{https://doi.org/10.1002/nla.576}.
\bibitem[{Lam and Lui(2014)}]{KCLam2014}
\bibinfo{author}{Lam, K.C.}, \bibinfo{author}{Lui, L.M.}, \bibinfo{year}{2014}.
\newblock \bibinfo{title}{Landmark- and intensity-based registration with large
  deformations via quasi-conformal maps}.
\newblock \bibinfo{journal}{SIAM J. Imaging Sci.} \bibinfo{volume}{7},
  \bibinfo{pages}{2364--2392}.
\newblock \DOIprefix\doi{https://doi.org/10.1137/130943406}.
\bibitem[{Lam and Lui(2015)}]{lamquasiconformal2015}
\bibinfo{author}{Lam, K.C.}, \bibinfo{author}{Lui, L.M.}, \bibinfo{year}{2015}.
\newblock \bibinfo{title}{Quasi-conformal hybrid multi-modality image
  registration and its application to medical image fusion}, in:
  \bibinfo{booktitle}{Adv. Vis. Comput.}, \bibinfo{address}{Cham}. pp.
  \bibinfo{pages}{809--818}.
\newblock \DOIprefix\doi{https://doi.org/10.1007/978-3-319-27857-5_72}.
\bibitem[{Lombaert(2023)}]{DiffLogDemons}
\bibinfo{author}{Lombaert, H.}, \bibinfo{year}{2023}.
\newblock \bibinfo{title}{Diffeomorphic log demons image registration}.
\newblock \URLprefix
  \url{https://www.mathworks.com/matlabcentral/fileexchange/39194-diffeomorphic-log-demons-image-registration}.
\bibitem[{Mang and Ruthotto(2017)}]{MangA}
\bibinfo{author}{Mang, A.}, \bibinfo{author}{Ruthotto, L.},
  \bibinfo{year}{2017}.
\newblock \bibinfo{title}{A {Lagrangian Gauss--Newton--Krylov} solver for mass-
  and intensity-preserving diffeomorphic image registration}.
\newblock \bibinfo{journal}{SIAM J. Sci. Comput.} \bibinfo{volume}{39},
  \bibinfo{pages}{B860--B885}.
\newblock \DOIprefix\doi{https://doi.org/10.1137/17M1114132}.
\bibitem[{Modersitzki(2009)}]{Fair}
\bibinfo{author}{Modersitzki, J.}, \bibinfo{year}{2009}.
\newblock \bibinfo{title}{FAIR: Flexible Algorithms for Image Registration}.
\newblock \bibinfo{publisher}{SIAM}.
\newblock \DOIprefix\doi{https://10.1137/1.9780898718843}.
\bibitem[{Mok and Chung(2020)}]{Mok_2020}
\bibinfo{author}{Mok, T.C.}, \bibinfo{author}{Chung, A.C.},
  \bibinfo{year}{2020}.
\newblock \bibinfo{title}{Fast symmetric diffeomorphic image registration with
  convolutional neural networks}, in: \bibinfo{booktitle}{IEEE Conf. Comput.
  Vis. Pattern Recognit.}, \bibinfo{publisher}{{IEEE}}. pp.
  \bibinfo{pages}{4643--4652}.
\newblock \DOIprefix\doi{https://doi.org/10.1109/CVPR42600.2020.00470}.
\bibitem[{Moore et~al.(2004)Moore, Liney and Beavis}]{MooreQuality}
\bibinfo{author}{Moore, C.S.}, \bibinfo{author}{Liney, G.P.},
  \bibinfo{author}{Beavis, A.W.}, \bibinfo{year}{2004}.
\newblock \bibinfo{title}{Quality assurance of registration of {CT} and {MRI}
  data sets for treatment planning of radiotherapy for head and neck cancers}.
\newblock \bibinfo{journal}{J. Appl. Clin. Med. Phys.} \bibinfo{volume}{5},
  \bibinfo{pages}{25--35}.
\newblock \DOIprefix\doi{https://doi.org/10.1120/jacmp.v5i1.1951}.
\bibitem[{Nocedal and Wright(2006)}]{NumericalOptimization2006}
\bibinfo{author}{Nocedal, J.}, \bibinfo{author}{Wright, S.},
  \bibinfo{year}{2006}.
\newblock \bibinfo{title}{Numerical Optimization}.
\newblock \bibinfo{edition}{2nd} ed., \bibinfo{publisher}{Oxford University
  Press}, \bibinfo{address}{New York}.
\bibitem[{Parikh(2014)}]{Parikh_2014}
\bibinfo{author}{Parikh, N.}, \bibinfo{year}{2014}.
\newblock \bibinfo{title}{Proximal algorithms}.
\newblock \bibinfo{journal}{Found. Trends Optim.} \bibinfo{volume}{1},
  \bibinfo{pages}{127--239}.
\newblock \DOIprefix\doi{https://doi.org/10.1561/2400000003}.
\bibitem[{Risser et~al.(2011)Risser, Vialard, Wolz, Murgasova, Holm and
  Rueckert}]{Risser2011}
\bibinfo{author}{Risser, L.}, \bibinfo{author}{Vialard, F.},
  \bibinfo{author}{Wolz, R.}, \bibinfo{author}{Murgasova, M.},
  \bibinfo{author}{Holm, D.D.}, \bibinfo{author}{Rueckert, D.},
  \bibinfo{year}{2011}.
\newblock \bibinfo{title}{Simultaneous multi-scale registration using large
  deformation diffeomorphic metric mapping}.
\newblock \bibinfo{journal}{IEEE T. Med. Imaging} \bibinfo{volume}{30},
  \bibinfo{pages}{1746--1759}.
\newblock \DOIprefix\doi{https://doi.org/10.1109/TMI.2011.2146787}.
\bibitem[{Ruhaak et~al.(2017)Ruhaak, Polzin, Heldmann, Simpson, Handels,
  Modersitzki and Heinrich}]{Ruhaak2017}
\bibinfo{author}{Ruhaak, J.}, \bibinfo{author}{Polzin, T.},
  \bibinfo{author}{Heldmann, S.}, \bibinfo{author}{Simpson, I.J.A.},
  \bibinfo{author}{Handels, H.}, \bibinfo{author}{Modersitzki, J.},
  \bibinfo{author}{Heinrich, M.P.}, \bibinfo{year}{2017}.
\newblock \bibinfo{title}{Estimation of large motion in lung {CT} by
  integrating regularized keypoint correspondences into dense deformable
  registration}.
\newblock \bibinfo{journal}{IEEE T. Med. Imaging} \bibinfo{volume}{36},
  \bibinfo{pages}{1746--1757}.
\newblock \DOIprefix\doi{https://doi.org/10.1109/TMI.2017.2691259}.
\bibitem[{Thirion(1998)}]{thirion1998image}
\bibinfo{author}{Thirion, J.P.}, \bibinfo{year}{1998}.
\newblock \bibinfo{title}{Image matching as a diffusion process: an analogy
  with {Maxwell}{'}s demons}.
\newblock \bibinfo{journal}{Med. Image. Anal.} \bibinfo{volume}{2},
  \bibinfo{pages}{243--260}.
\newblock \DOIprefix\doi{https://doi.org/10.1016/S1361-8415(98)80022-4}.
\bibitem[{Vercauteren et~al.(2008)Vercauteren, Pennec, Perchant and
  Ayache}]{Vercauteren2008}
\bibinfo{author}{Vercauteren, T.}, \bibinfo{author}{Pennec, X.},
  \bibinfo{author}{Perchant, A.}, \bibinfo{author}{Ayache, N.},
  \bibinfo{year}{2008}.
\newblock \bibinfo{title}{Symmetric log-domain diffeomorphic registration: A
  demons-based approach}, in: \bibinfo{booktitle}{Med. Image Comput. Comput.
  Assist. Interv. 2008}, pp. \bibinfo{pages}{754--761}.
\newblock \DOIprefix\doi{https://doi.org/10.1007/978-3-540-85988-8_90}.
\bibitem[{Wei et~al.(2022)Wei, Ahmad, Guo, Chen, Huang, Ma, Wu, Li, Wang, Lin,
  Yap, Shen and Wang}]{Wei_2022}
\bibinfo{author}{Wei, D.}, \bibinfo{author}{Ahmad, S.}, \bibinfo{author}{Guo,
  Y.}, \bibinfo{author}{Chen, L.}, \bibinfo{author}{Huang, Y.},
  \bibinfo{author}{Ma, L.}, \bibinfo{author}{Wu, Z.}, \bibinfo{author}{Li, G.},
  \bibinfo{author}{Wang, L.}, \bibinfo{author}{Lin, W.}, \bibinfo{author}{Yap,
  P.T.}, \bibinfo{author}{Shen, D.}, \bibinfo{author}{Wang, Q.},
  \bibinfo{year}{2022}.
\newblock \bibinfo{title}{Recurrent tissue-aware network for deformable
  registration of infant brain {MR} images}.
\newblock \bibinfo{journal}{IEEE T. Med. Imaging} \bibinfo{volume}{41},
  \bibinfo{pages}{1219--1229}.
\newblock \DOIprefix\doi{https://doi.org/10.1109/TMI.2021.3137280}.
\bibitem[{Yanovsky et~al.(2008)Yanovsky, Guyader, Leow, Toga, Thompson and
  Vese}]{Yanovskyarticle}
\bibinfo{author}{Yanovsky, I.}, \bibinfo{author}{Guyader, C.L.},
  \bibinfo{author}{Leow, A.}, \bibinfo{author}{Toga, A.W.},
  \bibinfo{author}{Thompson, P.M.}, \bibinfo{author}{Vese, L.},
  \bibinfo{year}{2008}.
\newblock \bibinfo{title}{Unbiased volumetric registration via nonlinear
  elastic regularization}, in: \bibinfo{booktitle}{2nd MICCAI workshop Math.
  Found. Comput. Anat.}
\bibitem[{Yanovsky et~al.(2007)Yanovsky, Osher, Thompson and
  Leow}]{Yanovsky2007Log}
\bibinfo{author}{Yanovsky, I.}, \bibinfo{author}{Osher, S.},
  \bibinfo{author}{Thompson, P.M.}, \bibinfo{author}{Leow, A.D.},
  \bibinfo{year}{2007}.
\newblock \bibinfo{title}{Log-unbiased large-deformation image registration},
  in: \bibinfo{booktitle}{2nd Int. Conf. Comput. Vis. Theory Appl}.
\bibitem[{Yeo et~al.(2010)Yeo, Sabuncu, Vercauteren, Ayache, Fischl and
  Golland}]{Yeo2010}
\bibinfo{author}{Yeo, B.}, \bibinfo{author}{Sabuncu, M.},
  \bibinfo{author}{Vercauteren, T.}, \bibinfo{author}{Ayache, N.},
  \bibinfo{author}{Fischl, B.}, \bibinfo{author}{Golland, P.},
  \bibinfo{year}{2010}.
\newblock \bibinfo{title}{Spherical demons: Fast diffeomorphic landmark-free
  surface registration}.
\newblock \bibinfo{journal}{IEEE T. Med. Imaging} \bibinfo{volume}{29},
  \bibinfo{pages}{650--668}.
\newblock \DOIprefix\doi{https://doi.org/10.1109/TMI.2009.2030797}.
\bibitem[{Zeidler(1985)}]{EZeidler1985}
\bibinfo{author}{Zeidler, E.}, \bibinfo{year}{1985}.
\newblock \bibinfo{title}{Nonlinear Functional Analysis and its Applications
  {III}: Variational Methods and Optimization}.
\newblock \bibinfo{publisher}{Springer-Verlag}.
\bibitem[{Zhang and Chen(2018)}]{Daoping2018}
\bibinfo{author}{Zhang, D.}, \bibinfo{author}{Chen, K.}, \bibinfo{year}{2018}.
\newblock \bibinfo{title}{A novel diffeomorphic model for image registration
  and its algorithm}.
\newblock \bibinfo{journal}{J. Math. Imaging Vis.} \bibinfo{volume}{60},
  \bibinfo{pages}{1261--1283}.
\newblock \DOIprefix\doi{https://doi.org/10.1007/s10851-018-0811-3}.
\bibitem[{Zhang et~al.(2022a)Zhang, Choi, Zhang and Lui}]{Zhang2022}
\bibinfo{author}{Zhang, D.}, \bibinfo{author}{Choi, G.P.T.},
  \bibinfo{author}{Zhang, J.}, \bibinfo{author}{Lui, L.M.},
  \bibinfo{year}{2022}a.
\newblock \bibinfo{title}{A unifying framework for n-dimensional
  quasi-conformal mappings}.
\newblock \bibinfo{journal}{SIAM J. Imaging Sci.} \bibinfo{volume}{15},
  \bibinfo{pages}{960--988}.
\newblock \DOIprefix\doi{https://doi.org/10.1137/21M1457497}.
\bibitem[{Zhang and Chen(2015)}]{JPZhang2015}
\bibinfo{author}{Zhang, J.}, \bibinfo{author}{Chen, K.}, \bibinfo{year}{2015}.
\newblock \bibinfo{title}{Variational image registration by a total
  fractional-order variation model}.
\newblock \bibinfo{journal}{J. Comput. Phys.} \bibinfo{volume}{293},
  \bibinfo{pages}{442--461}.
\newblock \DOIprefix\doi{https://doi.org/10.1016/j.jcp.2015.02.021}.
\bibitem[{Zhang et~al.(2016)Zhang, Chen and Yu}]{ZHANG2016}
\bibinfo{author}{Zhang, J.}, \bibinfo{author}{Chen, K.}, \bibinfo{author}{Yu,
  B.}, \bibinfo{year}{2016}.
\newblock \bibinfo{title}{An improved discontinuity-preserving image
  registration model and its fast algorithm}.
\newblock \bibinfo{journal}{Appl. Math. Model.} \bibinfo{volume}{40},
  \bibinfo{pages}{10740--10759}.
\newblock \DOIprefix\doi{https://doi.org/10.1016/j.apm.2016.08.009}.
\bibitem[{Zhang and Li(2021)}]{zhang2021diffeomorphic}
\bibinfo{author}{Zhang, J.}, \bibinfo{author}{Li, Y.}, \bibinfo{year}{2021}.
\newblock \bibinfo{title}{Diffeomorphic image registration with an optimal
  control relaxation and its implementation}.
\newblock \bibinfo{journal}{SIAM J. Imaging Sci.} \bibinfo{volume}{14},
  \bibinfo{pages}{1890--1931}.
\newblock \DOIprefix\doi{https://doi.org/10.1137/21M1391274}.
\bibitem[{Zhang et~al.(2022b)Zhang, Sun, Kong and
  Zhang}]{zhang_VectorialMinimized_2022}
\bibinfo{author}{Zhang, J.}, \bibinfo{author}{Sun, Z.}, \bibinfo{author}{Kong,
  X.}, \bibinfo{author}{Zhang, J.}, \bibinfo{year}{2022}b.
\newblock \bibinfo{title}{A vectorial minimized surface regularizer based image
  registration model and its numerical algorithm}.
\newblock \bibinfo{journal}{Appl. Math. Model.} \bibinfo{volume}{106},
  \bibinfo{pages}{150--176}.
\newblock \DOIprefix\doi{https://doi.org/10.1016/j.apm.2022.01.015}.

\end{thebibliography}



\end{document}